\documentclass[mnsc,non-blindrev]{informs3} 
\OneAndAHalfSpacedXI 

\usepackage{endnotes}
\usepackage[ruled]{algorithm}
\usepackage{algorithmicx,algpseudocode,amsfonts,subfigure}
\usepackage{graphicx,hyperref,comment,appendix}
\usepackage[capitalize,noabbrev]{cleveref}
\usepackage{tikz}
\usepackage{subfigure,caption,microtype,booktabs} 

\usepackage{amsthm,amsmath,amssymb,soul}
\usepackage{comment,enumitem}
\setlist{leftmargin=*}

\newtheorem{theorem}{Theorem}
\newtheorem{corollary}{Corollary}

\newtheorem{proposition}{Proposition}
\newtheorem{remark}{Remark}

\renewcommand{\Re}{\mathbb{R}}
\newcommand{\ie}{\textit{i.e.}}

\newcommand{\eg}{\textit{e.g.}}
\newcommand{\E}{\mathbb{E}}
\renewcommand{\Pr}{\text{Pr}}
\newcommand{\regret}{\texttt{Regret}}
\newcommand{\rev}{\texttt{Revenue}}
\newcommand{\kl}{\texttt{KL}}
\newcommand{\refund}{\texttt{Refund}}
\newcommand{\gdp}{\texttt{LEAP}}
\newcommand{\gdppp}{\texttt{LEAP++}}
\newcommand{\cA}{\mathcal{A}}

\usepackage{natbib}
\bibpunct[, ]{(}{)}{,}{a}{}{,}%
\def\bibfont{\small}%
\usepackage{bm,comment}

\ECRepeatTheorems

\EquationsNumberedThrough    

\allowdisplaybreaks
\begin{document}


\RUNAUTHOR{}

\RUNTITLE{Learning and Earning under Price Protection}

\TITLE{Phase Transitions in Learning and Earning under Price Protection Guarantee}

\ARTICLEAUTHORS{

\AUTHOR{Qing Feng}
\AFF{School of Operations Research and Information Engineering, Cornell University\\
	\EMAIL{qf48@cornell.edu}}

\AUTHOR{Ruihao Zhu}
\AFF{SC Johnson College of Business, Cornell University\\
	 \EMAIL{ruihao.zhu@cornell.edu}} 

\AUTHOR{Stefanus Jasin}
\AFF{Stephen M. Ross School of Business, University of Michigan\\
	\EMAIL{sjasin@umich.edu}} 

} 

\ABSTRACT{%
Motivated by the prevalence of ``price protection guarantee", which allows a customer who purchased a product in the past to receive a refund from the seller during the so-called price protection period (typically defined as a certain time window after the purchase date) in case the seller decides to lower the price, we study the impact of such policy on the design of online learning algorithm for data-driven dynamic pricing with initially unknown customer demand. We consider a setting where a firm sells a product over a horizon of $T$ time steps. For this setting, we characterize how the value of $M$, the length of price protection period, can affect the optimal regret of the learning process. We show that the optimal regret is $\tilde{\Theta}(\sqrt{T}+\min\{M,\,T^{2/3}\})$ by first establishing a fundamental impossible regime with novel regret lower bound instances. Then, we propose \gdp, a phased exploration type algorithm for \underline{L}earning and \underline{EA}rning under \underline{P}rice Protection to match this lower bound up to logarithmic factors or even doubly logarithmic factors (when there are only two prices available to the seller). Our results reveal the surprising phase transitions of the optimal regret with respect to $M$. Specifically, when $M$ is not too large, the optimal regret has no major difference when compared to that of the classic setting with no price protection guarantee. We also show that there exists an upper limit on how much the optimal regret can deteriorate when $M$ grows large. Finally, we conduct extensive numerical experiments to show the benefit of \gdp~over other heuristic methods for this problem. 
}%



\maketitle

%
\section{Introduction}\label{sec:intro}
Rapid development of data science technologies have enabled the incorporation of data-driven decision-making into revenue management. Among others, one of the most successful and prevalent examples is online learning for data-driven dynamic pricing \citep{KleinbergL03,HarrisonKZ2012,KZ14}. In the classic formulation of this problem, there is a firm, hereafter referred to as the seller, who sells a single product to customers over a time horizon without knowing their response to different prices ahead. The seller aims at maximizing her expected total revenue, but she faces the exploration-exploitation (a.k.a. learning and earning) tradeoff: On the one hand, she has to offer different prices to estimate the demand (response) of the customers; on the other hand, she also needs to choose the optimal price to maximize her expected total revenue. For this problem and many of its variants, many works in the literature have developed (nearly) optimal online learning algorithms to strike the right balance between exploration and exploitation (interested readers are referred to \cite{Boer15} for a comprehensive survey). To date, companies from different industries, including e-commerce \citep{MehtaDA18}, retailing \citep{Bond22}, and many others, have deployed data-driven dynamic pricing methods to gain more revenue.

In the traditional data-driven dynamic pricing literature, it is typically assumed that the seller can adjust the price as frequently as she wishes. In practice, however, as has been frequently reported (see, \eg, \citealt{BondiGWBR21,Sivakumar22}), customers can find it frustrating and unfair if the price of a recently purchased product drops to a lower level shortly afterward. This can further cause negative reviews due to customer dissatisfaction and, even worse, customer alienation. Consequently, it becomes clear that directly applying data-driven dynamic pricing in its canonical form may not be beneficial in the long run \citep{Mohammed12}.

\begin{figure}[!ht]
	\centering
	\subfigure[]{\includegraphics[width=6cm,height=8cm]{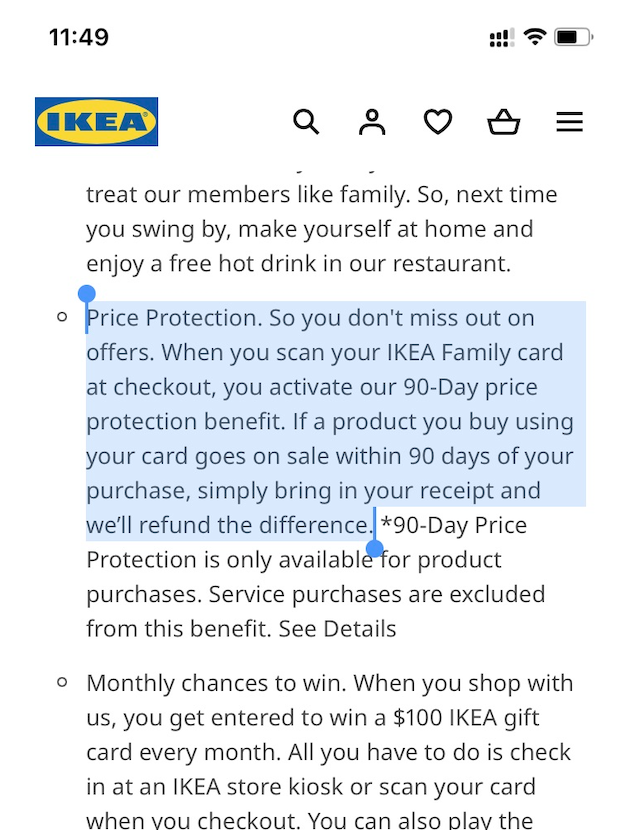}}
	\hspace{2.5cm}
	\subfigure[]{\includegraphics[width=6cm,height=8cm]{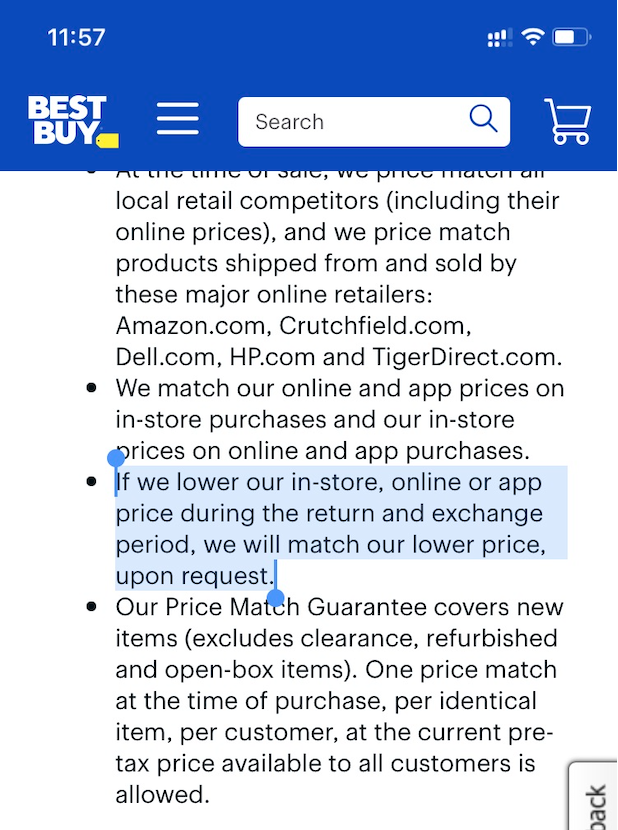}	}
	\caption{Price protection policies of two retailers}
	\label{fig:price_protection}
\end{figure} 

As a (partial) remedy for the inconvenience created by dynamic pricing, many companies have now implemented some forms of price protection guarantees (either explicitly or implicitly) for their products (see \cref{fig:price_protection} for some examples). In a nutshell, such price protection guarantee allows a customer who purchased a product in the past to receive a refund from the seller during the so-called price protection period (typically defined as a certain time window after the purchase date) in case the seller decides to lower the price. Here, the amount of refund is equal to the difference between the price paid and the lowest price found within the price protection period. For example, in \cref{fig:pp_example}, suppose for a certain period of time, the seller uses five different prices $p_5,\,p_4,\, p_3,\, p_2$, and $p_1$ (in descending order), and the customer purchases the product at $p_4$. Then, thanks to price protection, she will receive a refund of amount $p_4-p_2$ so that her actual payment is $p_2$, which is the lowest price found during the price protection period \footnote{We remark that although some companies only refund customers once during the price protection period while others allow multiple refunds if price drops further, we restrict our attention to the latter for brevity. Our forthcoming results, however, shall be mostly applicable to the former as well.} \footnote{If multiple price drops occur during the price protection period, one can assume that the customers receive the respective refunds all at once at the end of the price protection period. Alternatively, we could also assume that the customers receive refunds immediately after each price drop. Since the total amount of refunds are the same in both cases, we do not distinguish one from another throughout.}. 
\begin{figure}[!ht]
	\centering
	\includegraphics[width=10.5cm,height=5.4cm]{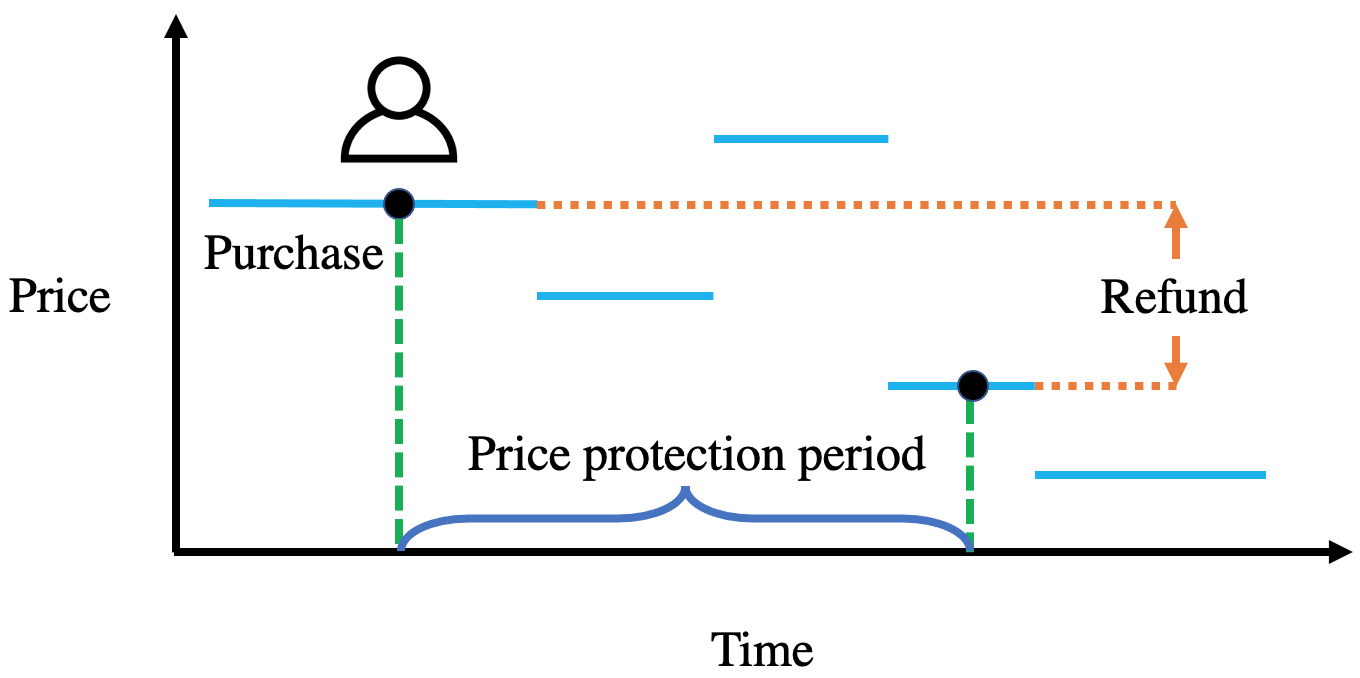}
	\caption{Toy example for price protection}
	\label{fig:pp_example}
\end{figure} 

\subsection{Main Contributions}
While price protection has now become one of the industry standards, its impact on data-driven dynamic pricing, especially on the corresponding online learning process, remains largely unexplored. In this paper, we provide an initial characterization of this effect. 

Let $T$ be the total number of time steps and $M$ be the length of price protection period. Our results and contributions can be summarized as follows:
\vspace{1mm}
\begin{enumerate}
    \item Due to the existence of price protection, directly applying classic approaches that  ignores price protection such as Upper Confidence Bound (UCB) algorithm \cite{ABF02} and Thompson Sampling algorithm \cite{AgrawalG12} would suffer linear in $T$ regret. We use both theoretical and numerical evidences to support this claim.
    
    \vspace{1mm}
    \item As such, we begin with the case where there are only two prices available to the seller. We show that the optimal regret is $\tilde{\Theta}(\sqrt{T}+\min\{M,\,T^{2/3}\})$ by first establishing a theoretical lower bound. Both of our lower bound instances and the proof techniques are different than traditional regret lower bounds (see the forthcoming \cref{remark:two_price_lb} for a detailed discussion). Next, we develop \gdp, a novel phased exploration algorithm for \underline{L}earning and \underline{EA}rning under \underline{P}rice Protection to match this lower bound up to $O(\log\log(T))$ factors. To achieve this, \gdp~creatively implements asynchronous price exploration and price elimination schedules. This mechanism allows us to simultaneously avoid excessive refunds and excessive selection of sub-optimal price (see the forthcoming \cref{remark:async} for a detailed discussion). After we complete the analysis for the case with only two price points, we generalize our results to the case with multiple price points and establish the optimal (up to logarithmic factors) regret. We refer to \cref{table:regret} for a summary.
    \begin{table}[!ht]
        \centering
        \renewcommand{\arraystretch}{1.8}
        \begin{tabular}{|c| c |c|} 
        \hline
        Number of prices & 2 & $K$\\
        \hline
        Regret lower bound &$\Omega\left(\sqrt{T}+\min\left\{M,\,T^{2/3}\right\}\right)$&$\Omega\left(\sqrt{KT}+\min\left\{M,\,K^{1/3}T^{2/3}\right\}\right)$\\
        \hline
        Regret upper bound &$O\left(\sqrt{T}+\min\left\{M\log\log(T)\,,\,T^{2/3}\right\}\right)$&$\tilde{O}\left(\sqrt{KT}+\min\left\{M,\,K^{1/3}T^{2/3}\right\}\right)$\\
        \hline
        \end{tabular}
    \caption{Optimal regret of our problem}
    \label{table:regret}
    \end{table}
   
    \item Importantly, our results reveal the surprising phase transitions of the optimal regret with respect to the length of price protection period $M$. Specifically, when $M\leq\sqrt{T}$, price protection has no major effect on the optimal regret when compared to the case without it. When $\sqrt{T}<M<T^{2/3}$, the optimal regret scales linearly with $M$ while, when $M\geq T^{2/3}$, the optimal regret would remain at the level of $\tilde{\Theta}(T^{2/3})$, which serves as an upper limit on the deterioration caused to the optimal regret by price protection guarantee. As the value of $M$ varies among different companies, our results is useful for informing the implementation of price protection policy in practice.
    
    \vspace{1mm}
    \item Finally, we conduct extensive numerical experiments to show that, by properly coordinating price protection and learning, our proposed algorithms are able to achieve significant improvement compared to other benchmark heuristic methods in our setting. 
\end{enumerate}

\subsection{Related Work}\label{sec:related}
In this section, we provide a detailed review of prior related studies. We start by reviewing existing works in online learning for classic dynamic pricing. Next, we discuss some connections between our setting and that of dynamic pricing in the face of strategic customers. Finally, we discuss existing works in the context of multi-armed bandits with limited adaptivity. Throughout, we will highlight the similarity and key differences between our work and prior works. 

\vspace{2mm}
\noindent
\textbf{Online Learning for Dynamic Pricing: } Online learning for dynamic pricing has been intensively studied in the literature, given its wide applications in retailing and e-commerce \citep{Boer15,FerreiraLS2016}. Many papers have developed near-optimal algorithms for pricing a product over multiple time periods under various settings (see, \eg, \citealt{KleinbergL03,BesbesZ2009,AramanC2009,FariasV2010}, \citealt{HarrisonKZ2012}, \citealt{BroderR2012}, \citealt{denZ2014}, \citealt{KZ14,WangDY2014,KZ16,FerreiraSLW18,CohenMW21,ChenG21,WangCL21,CheungLZ21,ZhuZ21,BastaniSZ2022,SimchiSZ2022,BuSLX22}). In recent years, there has also been increasing interests in dynamic pricing by taking into account exogenous and/or personalized features (see, \eg, \citealt{QiangB2016,javanmardN2019,KeskinLS2022,BanK2021,ChenSW2022,LiZ19}). Different from the above classic works, we mainly focus on online learning for dynamic pricing under price protection. 

\vspace{2mm}
\noindent
\textbf{Dynamic Pricing under Strategic Customers:} Our work is closely related to the literature of dynamic pricing in the presence of strategic and patient customers. In the latter, it generally assumes that customers can observe prices throughout a certain period of time, and then purchase at the lowest price possible (see, \eg, \citealt{FeldmanKL2016,BirgeCK21,ChenGG22,zhang2022online}, and references therein). Prior works show that strategic and patient customers present extra challenge to the seller and may incur larger regret. For example, \cite{BirgeCK21} studies the markdown pricing problem of a firm that sells a product to a set of
myopic and forward-looking customers. \cite{BirgeCK21} shows the seller's revenue under demand model uncertainty and strategic customers depends crucially on how forward-looking the customers are. In particular, when the customers' forward-looking behavior is exogenous and their memory is unbounded, the seller has to suffer a major revenue loss. The authors also characterize  conditions under which the seller can achieve asymptotic optimality. Our setting differs from \cite{BirgeCK21} in several aspects. 
First, we study a frequentist setting where there is no prior belief over the customer demand models while \cite{BirgeCK21} studies a Bayesian setting where there is a prior over potential customer demand models. Also, \cite{BirgeCK21} considers the customers to be forward-looking and they are capable of predicting the future prices of the product. In contrast, our setting assumes the customers would make an instantaneous purchase decision, and then passively receive the refunds from the price protection guarantee. Finally, we do not impose the restriction that the seller can only mark down the price over time.

\cite{FeldmanKL2016} shows in the adversarial setting (\ie, customer demand may not follow a stationary distribution) that if customers can observe the prices for several time steps and wait for the lowest price, then the optimal regret for the seller would become $\Theta(T^{2/3})$ instead of $\Theta(\sqrt{T})$ in classic dynamic pricing. Our setting is similar to \cite{FeldmanKL2016} in that customers effectively also pay the lowest price within certain time duration. But we are critically different in that we consider a less pessimistic scenario where the random demand follows a stochastic distribution, and the oracle policy (when the demand distribution is known ahead) can be easily identified.

\cite{zhang2022online} studies the setting where customers are patient and are willing to wait for a few periods before deciding to leave the system without making any  purchase. Specifically, each customer is endowed with a (potentially different) willingness to pay and willingness to wait, and a customer will buy the product at the first time the price is smaller than or equal to the customer's willingness to pay. \cite{zhang2022online} develops several near-optimal cyclic pricing policy. They also show that standard learning algorithms that ignore patient customer behavior will suffer a significant loss in revenue. Our setting differs from \cite{zhang2022online} in at least one crucial component: Since \cite{zhang2022online} assumes that a customer will buy the product at the first time the price is smaller than or equal to the customer's willingness to pay, this price may not be the smallest price within the time periods in which the customer is willing to wait. Moreover, unlike in their setting, the optimal pricing policy in our setting is not cyclic.


\vspace{2mm}
\noindent
\textbf{Multi-armed Bandits with Limited Adaptivity:} From a technical point of view, our setting is related to multi-armed bandits with limited adaptivity (e.g., \citealt{GaoHRZ19,LeviX21,RouyerSB21}). This is so because, in our setup, whenever there is a price decrease, the seller may have to refund the customers. This would implicitly refrain the seller from changing the price. In \cite{LeviX21} and \cite{GaoHRZ19}, the authors impose a restriction that the decision-maker (which is the seller in our setting) can only change her actions/policies (which is price in our setting) up to a pre-specified number of times. \cite{LeviX21} reveals an interesting phase transition behavior of the optimal regret w.r.t. the number of changes allowed. Moreover, both papers confirm that (doubly) logarithmic amount of adaptivity would be enough to attain the optimal regret. As a matter of fact, part of our optimal learning strategy to plan pricing in batches (\ie, minimizing the number of price changes) is motivated by \cite{LeviX21} and \cite{GaoHRZ19}. However, different than this line of work, there is no explicit constraint on the number of price changes in our model. Instead, this constraint enters the play implicitly as a kind of regularization. It turns out that this difference leads to drastically different results, from algorithm design to the phase transitions of the optimal regret. We refer to \cref{sec:add_disc} for more detailed discussions. 

In \cite{RouyerSB21}, the authors assume that a unit cost would be incurred whenever an action change takes place. Similar to \cite{RouyerSB21}, the cost of price changes in our setting is incorporated into the objective function as a regularization. In some sense, results for stochastic bandits with switching cost developed in \cite{RouyerSB21} can be viewed as a special case of our results with $M=1$.

\section{Problem Formulation}\label{sec:model}
In this section, we introduce the notations that will be be used throughout paper. In addition, we will also introduce the learning protocol for our problem. 

\subsection{Notations}
 We define $[n]$ to be the set $\{1,2,\ldots,n\}$ for any positive integer $n.$ For $k\in [1, \infty]$, we use $\| x\|_k$ to denote the $\ell_k$-norm of a vector $ x\in\Re^d.$ We use $\log_a(\cdot)$ to denote the logarithm with base $a\,.$ When $a$ is left unspecified, this is the natural logarithm. We adopt the asymptotic notations $\,O(\cdot)\,,\,\Omega(\cdot)\,,$ and $\Theta(\cdot)$ as defined in \citep{CLRS09}. When logarithmic factors are omitted, we write them as $ \tilde{O}(\cdot)\,,\,\tilde{\Omega}(\cdot)\,,$ and $\tilde{\Theta}(\cdot)\,.$ With some abuse, these notations are used when we try to avoid the clutter of writing out constants explicitly. The indicator variable is denoted as $\bm{1}[\cdot]\,.$

\subsection{Model}

\textbf{Model Primitives:} We consider a firm, hereafter refereed to as the \emph{seller}, that sells a single product over a horizon of $T$ time steps. For each time step $t,$ the seller must first set a price $p_{i_t}\in P=\{p_1,\ldots,p_K\}\subseteq[0,1]$ and then observes the i.i.d. customer demand $D_t(p_{i_t})\in[0,1].$ We let $\mu_k=\E[D_t(p_k)]$ be the \emph{initially unknown} expected demand w.r.t. price $p_k\in P\,.$ We use $\lambda_k=p_k\mu_k$ to denote the expected reward of offering price $p_k$ for $k\in[K]\,.$ For ease of exposition, we often refer to $p_k$ as price $k\,.$ To simplfy notation, we define $D_t(p_i)\equiv 0$ for all $i\in[K]$ whenever $t\leq 0$ or $t>T\,.$

\vspace{2mm}
\noindent\textbf{Price Protection Guarantee:} Motivated by practice, we assume that the seller provides price protection guarantee for a total of $M$ time steps. That is, even if a customer purchases the product in time step $t$ at an initial price $p_{i_t},$ her final payment would  be the minimum among the prices of this period and the $M$ immediate next time periods, \ie, $\min\{p_{i_t}\,,\,p_{i_{t+1}}\,,\,\ldots\,,\,p_{i_{t+M}}\}\,.$ In other words, she would be refunded the difference between $p_{i_t}$ and the minimum of $p_{i_t}\,,\,\ldots\,,\,p_{i_{t+M}}\,.$

\vspace{2mm}
\noindent\textbf{Objective:} Since the actual payment of a customer who arrives at time step $t$ is $\min\{p_{i_t},\dots,p_{i_{t+M}}\}\,$, the total revenue of the seller throughout the entire time horizon can be written as
\begin{align*}
\texttt{Revenue}(p_{i_1}\,,\,\ldots\,,\,p_{i_T})=\sum_{t=1}^T\min\{p_{i_t}\,,\,\ldots\,,\,p_{i_{t+M}}\}D_t(p_{i_t})\,,
\end{align*}
and the total refund throughout the horizon can be written as
\begin{equation*}
    \texttt{Refund}(p_{i_1}\,,\,\ldots\,,\,p_{i_T})=\sum_{t=1}^T\left(p_{i_t}-\min\{p_{i_t}\,,\,\ldots\,,\,p_{i_{t+M}}\}\right)D_t(p_{i_t}).
\end{equation*}
It is also useful to think of revenue from an instantaneous point of view. At time step $t$, if the seller sets the price to $p_{i_t}$, then she will receive an instantaneous reward (\ie, payment) of $p_{i_t}D_t(p_{i_t})$. At the same time, this decision may also incur refunds for customers arriving between time step $t-M$ to time step $t-1$. Consider any time step $s$ satisfying $t-M\leq s\leq t-1$. By time the end of time step $t-1$, the actual payment of the customer arriving at time period $s$ is $\min\{p_{i_s},\dots,p_{i_{t-1}}\}$. Therefore, if $p_{i_t}$ is selected at time period $t$, the additional refund for the customer of time step $s$ is $\left[\min_{w\in[s,t-1]}p_{i_w}-p_{i_t}\right]^+D_t(p_{i_s})$. This means the \emph{net} revenue at time period $t$ for the seller when selecting $p_{i_t}$ can be written as
\begin{align}\label{eq:inst_rev}
	\rev_t(p_{i_t}|p_{i_{t-1}}\,,\,\ldots\,,\,p_{i_{t-M}})=p_{i_t}D_t(p_{i_t})-\underbrace{\sum_{s=t-M}^{t-1}\left[\min_{w\in[s,t-1]}p_{i_w}-p_{i_t}\right]^+D_t(p_{i_s})}_{\refund_t\left(p_{i_t}|p_{i_{t-1}},\ldots,p_{i_{t-M}}\right)}.
\end{align}
We denote the second term, \ie, the instantaneous refund, as $\refund_t(p_{i_t}|p_{i_{t-1}},\ldots,p_{i_{t-M}})$. 

The objective of the seller is to maximize her expected cumulative revenue by following a possibly randomized \emph{non-anticipatory} pricing policy $\pi=(\pi_1,\ldots,\pi_T).$ Specifically, denoting $H_t=\{p_{i_s},D_s(p_{i_s}),\refund_s\}_{s=1}^{t-1}$ as the history at time $t,$ at each time step $t,$ the seller follows $\pi_t:H_t\to P$ to pick a price $p_{i_t}.$ We use the notion of expected \emph{regret} to measure the performance of the seller. Specifically, regret is defined as the difference between the maximum expected cumulative revenue of the seller could attain should her knows all $\lambda_k$'s ahead (note that, in this case, the seller would follow a fixed price policy, see \cref{remark:oracle} for more details) and the cumulative revenue of $\pi,$ \ie, 
\begin{align}
	\regret(\pi)=T\max_{k\in[K]}\ \lambda_k-\E\left[\texttt{Revenue}(p_{i_1}\,,\,\ldots\,,\,p_{i_T})\right]\,.
\end{align}
We point out that the regret in our setting shall depend on both $M$ and $T\,,$ but we avoid the explicit dependence for notational brevity. 

\vspace{2mm}
\begin{remark}\label{remark:oracle}
If there is no price protection, the optimal policy for a seller with full knowledge of all $\lambda_k$'s would be constantly selecting the price with the highest $\lambda_k$. If the seller follows the same policy in the case with price protection, she would still achieve the maximum expected revenue (i.e., since no refund would be incurred). Hence, under price protection, the optimal policy is still constantly selecting the price with highest $\lambda_k$ when all $\lambda_k$'s are known to the seller.
\end{remark}

\vspace{2mm}
\noindent\textbf{Additional Notations: } To facilitate our discussion, for any time step $t\,,$ we will use $N_k(t)$ to denote the number of times that price $k$ is selected up to the beginning of time step $t\,,$ \ie,
\begin{align*}
	N_k(t)=\sum_{s=1}^{t-1}\bm{1}[i_s=k]\,.
\end{align*} 
The corresponding empirical mean of $\lambda_k$ up to the beginning of time step $t$ is denoted as 
\begin{align*}
\hat{\lambda}_k(t)=\frac{\sum_{s=1}^{t-1}\bm{1}[i_s=k] \, p_{i_s}D_s(p_{i_s})}{N_k(t)}\,.
\end{align*}

\subsection{Blooper Reel}\label{sec:ucb}
In this section, to provide an initial, but concrete understanding of the challenges of data-driven dynamic pricing in the presence of price protection guarantee, we consider two celebrated dynamic pricing algorithms, the Upper Confidence Bound (UCB) algorithm \citep{ABF02,KleinbergL03} and the Thompson Sampling (TS) algorithm \citep{AgrawalG12}. It is well known that, in absence of the price protection mechanism, both the UCB and the TS algorithms are capable of attaining the minimax optimal regret (\ie, $\tilde{O}(\sqrt{T})$) for the classic problem of dynamic pricing. Due to their proximity, it is tempting to apply the UCB and the Thompson Sampling algorithms to our problem. However, 
as we show below, a direct application of the UCB algorithm can result in linear regret even when the price protection period is small. 

We first discuss the popular UCB algorithm introduced in~\cite{ABF02}: For each time step $t\,,$ UCB selects the price $p_k$ that maximizes the index $\hat{\lambda}_k(t)+\sqrt{\log(T)/N_k(t)}\,,$ \ie,
\begin{align}\label{eq:ucb}
    i_t=\argmax_{k\in[K]}~\hat{\lambda}_k(t)+\sqrt{\frac{\log(T)}{N_k(t)}}\,.
\end{align}
The following proposition tells us that the above UCB algorithm could incur $\Omega(T)$ regret even when the price protection period is small.

\vspace{2mm}
\begin{proposition}\label{prop:ucb_regret1}
    Suppose that $T\geq 10\,,$ $M\geq 2\,,$ and there are only two prices $p_1,\,p_2$. Then, there exists a demand distribution such that UCB algorithm will incur $\Omega(T)$ regret.
\end{proposition}
\begin{proof}[Proof Sketch]
The complete proof of this proposition is given in \cref{sec:prop:ucb_regret1} of the Appendix. In a nutshell, we provide an instance of demand distributions such that if $N_1(t)>N_2(t)$, then UCB would pick $p_2$ as the { difference between $\sqrt{\log(T)/N_2(t)}$ and $\sqrt{\log(T)/N_1(t)}$ (\ie, the second terms in \eqref{eq:ucb}) would dominate the difference between $\hat{\lambda}_1(t)$ and $\hat{\lambda}_2(t)$ (\ie, the first terms in \eqref{eq:ucb}) and vice versa.} Consequently, UCB would alternate between $p_1$ and $p_2$, which would incur $\Omega(T)$ regret (due to refund under price protection). 
\end{proof}

\vspace{2mm}
\begin{remark}[\textbf{Further Discussions on UCB}] In \cref{sec:ucb_test} of the Appendix, we provide numerical examples to support the claim in \cref{prop:ucb_regret1}. It is evident that the results in \cref{prop:ucb_regret1} can be easily generalized to more than two prices, \ie, $K\geq2\,.$ Moreover, it turns out that the above findings can go beyond the canonical UCB. In \cref{sec:ucb_discussion} of the Appendix, we provide further discussions on how even a general family of UCB algorithm can fail to provide meaningful performance guarantee under the price protection mechanism. 
\end{remark}

\vspace{2mm}
We also use numerical examples to further demonstrate that both the UCB algorithm \citep{ABF02} and the Thompson Sampling algorithm \citep{AgrawalG12} incur linear regret due to the price protection mechanism. Specifically, we use the following instance. $K=2\,,$ (\ie, there are only two prices), with $p_1=1/4\,,$ $p_2=1\,.$ The demand under the two prices are $D_t(p_1)\sim \text{Bernoulli}(2/3)\,,$ $D_t(p_2)\sim\text{Bernoulli}(1/2)\,.$ In this case, the expected reward of selecting $p_1$ and $p_2$ are $1/6$ and $1/2\,,$ respectively. The time horizon $T$ varies from $1000$ to $20000$ with a step size of $1000\,,$ and the length of price protection period is set to be $M=T/5\,.$ For each $T$, we compute the expected regret by averaging over $10^4$ iterations. The results are shown in \cref{fig:ucb_ts_regret_2}, and the portion of (expected) refund in (expected) regret under different $T$ is shown in \cref{fig:ucb_ts_refund_2}. { Typical sample paths of UCB and Thompson Sampling on this instance when $T=2000$ are shown in \cref{fig:sample_ucb} and \cref{fig:sample_ts}.}

\begin{figure}[!ht]
    \centering
    \subfigure[Expected regret]{\includegraphics[height=6cm, width = 8cm]{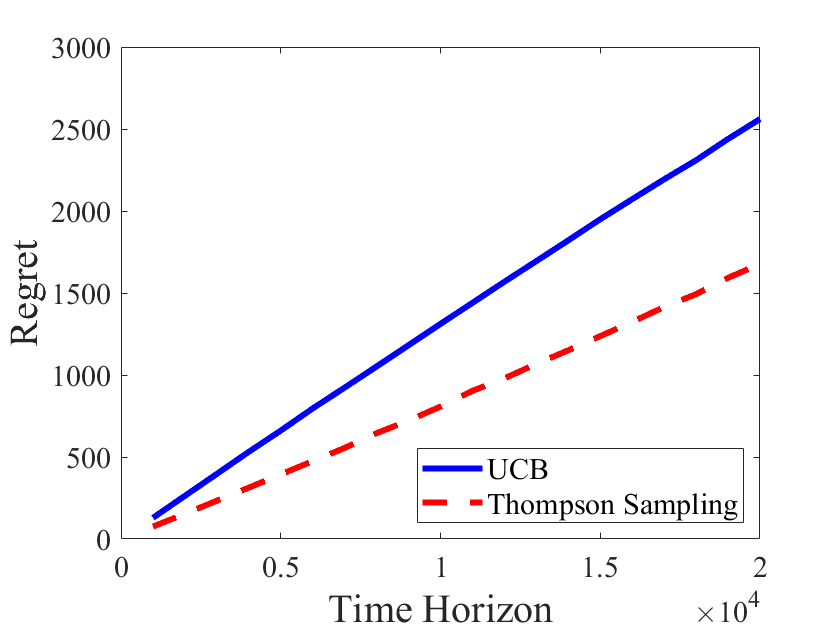}
    \label{fig:ucb_ts_regret_2}}
    \subfigure[Portion of refund]{\includegraphics[height=6cm, width = 8cm]{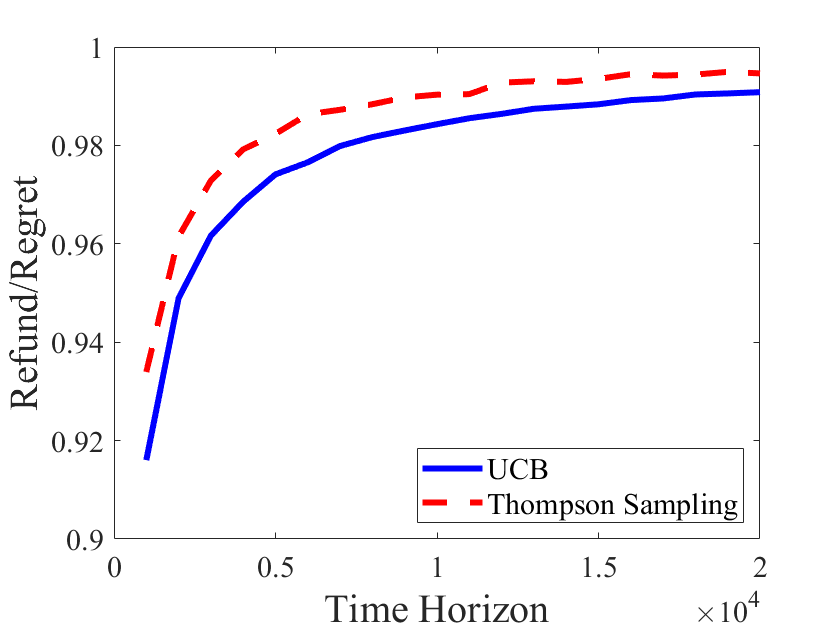}
    \label{fig:ucb_ts_refund_2}}
    \subfigure[Sample path of UCB]{\includegraphics[height=6cm, width = 8cm]{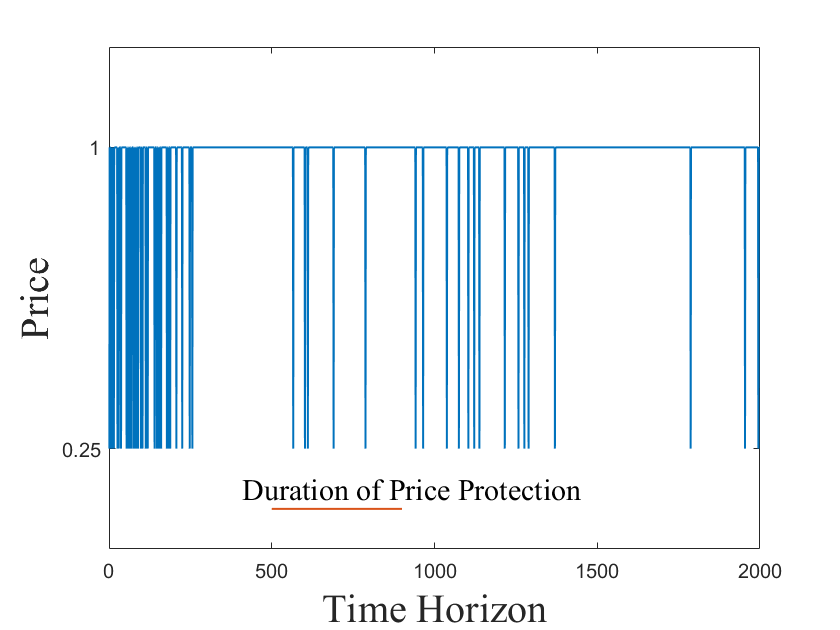}
    \label{fig:sample_ucb}}
    \subfigure[Sample path of Thompson sampling]{\includegraphics[height=6cm, width = 8cm]{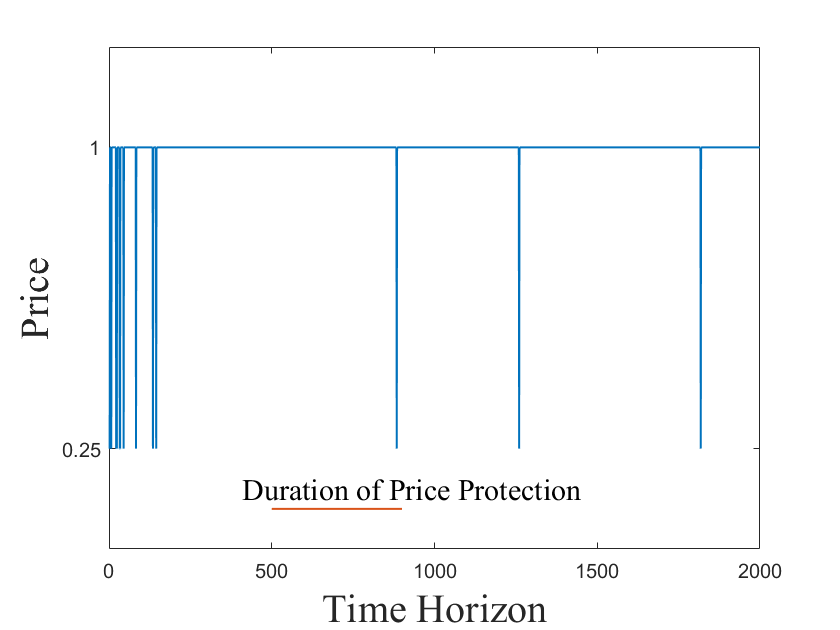}
    \label{fig:sample_ts}}
    \caption{Numerical test of the UCB and the Thompson sampling under different $T$'s}
\end{figure}

From Figure \ref{fig:ucb_ts_regret_2}, one can see that the expected regrets of the UCB and the Thompson Sampling algorithms scale linearly with $T\,.$ 
 Furthermore, one can see from Figure \ref{fig:ucb_ts_refund_2} that refund takes up of more than $90\%$ of regret throughout, which indicates that the majority of the regret comes from the refund under the price protection mechanism. 

{ One reason that both UCB and Thompson Sampling incur linear regrets is as follows: Although both algorithms select $p_2$ more frequently, they do not completely eliminate the selection of $p_1$. Whenever $p_1$ is selected, a major refund would be incurred. As can be seen from \cref{fig:sample_ucb} and \cref{fig:sample_ts}, selecting $p_1$, especially when $t$ gets large, would cause a refund for almost all $M$ time steps before that. Since $M=\Theta(T)$, both UCB or Thompson sampling thus incur linear in $T$ regret.}

To this end, one might be curious if a simple modification of the UCB and the Thompson Sampling, which take the price protection mechanism into account, could resolve the issue. We conduct further numerical experiments in the forthcoming \cref{sec:numerical_two_price} to show that even a natural modification of either UCB or Thompson Sampling algorithm that that takes into account refund would still incur linear regret. 

\section{Price Protection with Two Prices: A First Look}\label{sec:two_price}

The above discussions show that the presence of price protection guarantee can dramatically change the landscape of the dynamic pricing problem from its classical form. Most importantly, directly applying existing algorithms designed for the setting without price protection guarantee to our setting can generate highly unsatisfactory performance.

In particular, the price protection guarantee creates an extra layer of difficulty to the dynamic pricing problem as any (unnecessary) decrease of price may result in undesirable loss of revenue. At the same time, 
in order to identify the optimal price, experimenting among different prices is indispensable \cite{GaoHRZ19,LeviX21} (we refer the readers to the forthcoming \cref{remark:rationale} and \cref{sec:add_disc} for a comparison between the two settings). 

To gain some preliminary insights on how to strike the right balance between exloration and exploitation, we begin with the case where there are only two candidate prices, \ie, $K=2$ and $P=\{p_1\,,\,p_2\}.$ Without loss of generality, we assume that $p_1<p_2.$ In \cref{sec:two_price_lb}, we characterize a theoretical lower bound, which establishes a fundamental impossibility result for the special case with two prices. Next, we turn our attention to regret minimization. In \cref{sec:two_price_ub}, we propose \gdp, an algorithm for \underline{L}earning and \underline{EA}rning under \underline{P}rice protection. \gdp~is a phased exploration type algorithm that judiciously balances exploration, exploitation, and price adjustment (see the forthcoming \cref{remark:rationale} for intuition on this design). In each phase, it explores both prices according to a pre-specified price exploration schedule, and it also has a $M$-dependent price elimination schedule. Specifically, when $M$ is not too large (\ie, $M\leq T^{2/3}$), these two schedules are asynchronous, \ie, price elimination can happen in the middle of a phase of price exploration. This turns out to be beneficial in further reducing expected regret when compared to the synchronous design, and may be of independent interest (see the forthcoming \cref{remark:async} for a detailed discussion on this design). When $M$ is large (\ie, $M\geq T^{2/3}$), these two schedules are synchronous and \gdp~becomes similar to the explore-then-commit algorithm (see, \eg, chapter 6 of \citealt{LS18}), which could help to avoid excessive refund. In \cref{sec:two_price_regret}, we analyze the expected regret upper bound of \gdp, and compare its design against classic phased exploration algorithms for learning and earning (see, \eg, \citealt{AO10}). Combining the above, our results indicate that the upper and lower bound match each other up to $O(\log\log(T))$ factor. Most surprisingly, under the two price case, the optimal regret for our problem exhibits different growing patterns w.r.t. different ranges of price protection period $M\,,$ which we refer to as the phase transition phenomenon. We provide an in-depth discussion for this finding in \cref{sec:two_price_phase}.

\subsection{Lower Bound}\label{sec:two_price_lb}

Below, we establish a theoretical lower bound, which presents a fundamental impossibility regime for our problem.

\vspace{2mm}
\begin{theorem}\label{thm:two_price_lb}
	For any $T\geq M\geq 2$ and any non-anticipatory policy $\pi$, there exists $p_1,p_2\in [0,1]$ and the corresponding distributions of demand $D_t(p_1),D_t(p_2)$ such that the regret of $\pi$ is $\Omega(\sqrt{T}+\min\{M,\,T^{2/3}\})\,.$ 
\end{theorem}
\begin{proof}[Proof Sketch.]
The complete proof is provided in Section \ref{sec:thm:two_price_lb} of the Appendix. Depending on the value of $M,$ we distinguish three different cases.

\vspace{1mm}
\noindent\textbf{Case 1. $M\leq \sqrt{T}:$} For this case, one can use the standard lower bound argument for multi-armed bandit (see, \eg, Theorem 5.1 of \citealt{ABFS02}) to show that the lower bound is $\Omega(\sqrt{T})$.

\vspace{1mm}
\noindent\textbf{Case 2. $\sqrt{T}< M< T^{2/3}:$} For this case, we assume the demand of $p_1$ is deterministic and known in advance, and we come up with two different choices of demand for $p_2.$ In instance 1, selecting $p_2$ would bring $\Theta(M^{-1/2})$ lower expected payment than selecting $p_1;$ while in instance 2, selecting $p_2$ would bring $\Theta(M^{-1/2})$ higher expected payment than selecting $p_1.$
    
As such, if a policy $\pi$ selects $p_2$ for $o(M)$ times before some time period $T_1~(=T-M^{3/2}/2+1)$ such that, it would not be able to distinguish which instance is presented and suffers $\Omega(M)$ expected regret if instance 2 is presented; while if it selects $p_2$ for $\Omega(M)$ times before time period $T_1,$ then either the price protection guarantee or the { avoidance of refund (\ie, by not selecting $p_1$)} during the last $T-T_1=\Theta(M^{3/2})$ time steps would also incur $\Omega(M)$ expected regret if instance 1 is presented.

\vspace{1mm}
\noindent\textbf{Case 3. $M\geq T^{2/3}:$} The proof for this case is similar to that of Case 2 except that in instance 1, selecting $p_2$ would bring $\Theta(T^{-1/3})$ lower expected payment than selecting $p_1;$ while in instance 2, selecting $p_2$ would bring $\Theta(T^{-1/3})$ higher expected payment than selecting $p_1.$

As such, if a policy $\pi$ selects $p_2$ for $o(T^{2/3})$ times before some time period $T_2~=(T/2+1),$ it would not be able to distinguish which instance is presented and suffers $\Omega(T^{2/3})$ expected regret if instance 2 is presented; while if it selects $p_2$ for $\Omega(T^{2/3})$ before time period $T_2,$ then either the price protection mechanism or the {avoidance of refund (\ie, by not selecting $p_1$)} during the last $T-T_2=T/2$ time steps would also incur $\Omega(T^{2/3})$ expected regret if instance 1 is presented. 
\end{proof}

\vspace{2mm}
\begin{remark}[\textbf{Novelty of Our Regret Lower Bound}]\label{remark:two_price_lb}
Despite the connections between our setting and the standard multi-armed bandits, our lower bound proof is different from the classic lower bound proof for multi-armed bandits in a subtle, yet critical way. To see this more clearly, in both Case 2 and Case 3 above, we distinguish two different sub-cases based on the \underline{exact number} of times $p_2$ is selected (\ie, $N_2(t)$) instead of the \underline{expected number} of times $p_2$ is selected (\ie, $\E[N_2(t)]$), as in Theorem 5.1 of \cite{ABFS02} or Theorem 15.2 of \cite{LS18}. It turns out that this subtle difference is extremely important in making sure that the price protection guarantee can incur adequate refund. Otherwise, in Case 2 above, for instance, suppose a policy $\pi$ only has $\E[N_2(T_1)]=\Theta(M)$. Even if it keeps selecting $p_1$ after time step $T_1\,,$ the expected regret incurred by the price protection guarantee could be as low as $M^2/T_1~(=o(M))$ (this happens if $\pi$ is a policy such that $N_2(T_1+1)=T_1$ with probability $M/T_1\,;$ while $N_2(T_1+1)=0$ with probability $1-M/T_1$). Associated with these are a set of carefully designed information-theoretic tools different from those in \cite{ABFS02} and \cite{LS18}. We refer interested readers to \cref{sec:lb_remark} of the Appendix for a more involved discussion. 
\end{remark}

\subsection{Algorithm for Learning and Earning under Price Protection (\gdp)}\label{sec:two_price_ub}
In the previous discussions in \cref{sec:ucb}, we can observe that ordinary learning and earning strategies cannot achieve meaningful regret guarantee because they do not take the price protection guarantee into account. To this end, we develop \gdp, which explores the two prices in phases (of possibly increasing length), and eliminates the under-performing price once its detected. The formal description of \gdp~is provided in \cref{alg:gdp}. It follows different steps w.r.t. different values of $M$. We discuss the two cases $M < T^{2/3}$ and $M \ge T^{2/3}$ separately below.

\begin{algorithm}[htbp]
\SingleSpacedXI
\caption{Algorithm for Learning and Earning under Price Protection (\gdp)} \label{alg:gdp}
\begin{algorithmic}[]
    \State \textbf{Input:} Price set $P=\{p_1\,,\,p_2\}\,,$ time horizon $T\,,$ price protection period $M\,.$ 
    \State \textbf{Initialization:} $a\leftarrow e\sqrt{T}\,,\,B\leftarrow\lceil\log_2\log(T)\rceil\,,\,L\leftarrow\lfloor\log_2(T/e)/2\rfloor\,,\,u_b\leftarrow a^{2-2^{1-b}}~\forall b\in[B]\,,\,t_b\leftarrow\min\{T\,,\,\lceil u_b\rceil\}~\forall b\in[B]\,,\,\tilde{\Delta}_l\leftarrow2^{-l}\,,\,n_l\leftarrow\lceil2\log(T\tilde{\Delta}^2_l)/\tilde{\Delta}^2_l\rceil\,,\, N\leftarrow \lceil T^{2/3}\rceil$
    \If{$M<T^{2/3}$}
    \For {Phase $b=1\,,\, \ldots\,,\, B$}
    \If{neither price is eliminated}
    \State Select each of them for $(t_b-t_{b-1})/2$ times, starting with price $\argmax_k~\hat{\lambda}_k(t_{b-1}+1)$
    \If{$N_1(t)\,,\,N_2(t)\geq n_l$ for the first time}
    \If{$\hat{\lambda}_{\max}(t) - \hat{\lambda}_{k}(t)>\sqrt{{2\log(T\tilde{\Delta}^2_l)}/{n_l}}$ for some price $k$}
    \State Eliminate price $k$, and select the remaining price for the rest of this phase
    \EndIf
    \EndIf
    \Else
    \State Select the remaining price
    \EndIf
    \EndFor
    \Else
    \State Select $p_1$ and $p_2$ for $N$ times, compute $\hat{\lambda}_1(2N+1)$ and $\hat{\lambda}_2(2N+1)\,,$ and select price $\argmax_k~\hat{\lambda}_k(2N+1)$ for the rest of the time
    \EndIf
\end{algorithmic}
\end{algorithm}

\vspace{2mm}
\noindent\underline{\textbf{Case 1. $M< T^{2/3}$}} 

\noindent In this case, \gdp~works in multiple phases. In each phase, if there is only one plausible price left, \gdp~would simply selects that price; otherwise, if neither of the two prices is eliminated, \gdp~would select $p_1$ and $p_2$ for equal number of times. The price with higher empirical mean reward would be selected first. Then, whenever a certain amount of demand data associated with each price is collected, it performs a statistical test to see if the averaged reward of one price significantly exceeds that of the other price. If that is the case, \gdp~would eliminate the under-performing price, and stay with the other for the rest of the time. We remark that different than existing phased exploration type algorithms (see, \eg, \citealt{AO10,GaoHRZ19}), the statistical test of \gdp~can be done throughout each phase instead of only at the end of a phase, and there could possibly be multiple tests in each phase (see the forthcoming \cref{remark:async} for a discussion on this).

Specifically, \gdp~maintains two sequences of time steps, the first one keeps track of the different phases of \gdp. The second keeps track of the time steps when statistical tests would be performed. We begin with the sequence that determines the phases of \gdp. Let $a=e\sqrt{T}$ and $B=\lceil\log_2\log(T)\rceil\,,$ we define the following quantities $\{u_b\,,\,t_b\}_{b=0}^B$:
\begin{align}\label{eq:def_u_t}
  u_0 = 1\,,\ t_0=0\,,\  u_b=a\sqrt{u_{b-1}}~\left(=a^{2-2^{1-b}}\right)\,,\ t_b = \min\left\{T\,,\ \left\lceil u_b\right\rceil\right\}\,\quad  \forall~b=1\,,\ \ldots\,,\ B\,. 
\end{align}

\vspace{1mm}
\noindent
We remark that the definitions in \eqref{eq:def_u_t} follows from \cite{GaoHRZ19}. With this, each phase begins at time step $t_{b-1} +1$ and ends at time step $t_b\,.$ The algorithm begins by viewing both $p_1$ and $p_2$ as plausible prices. In each phase $b\,,$ if none of $p_1$ and $p_2$ is removed, \gdp~would first select the price with higher empirical mean reward, computed with data from all previous phases, \ie, $$\argmax_{k\in\{1\,,\,2\}}~\hat{\lambda}_k(t_{b-1}+1)\,,$$ 
for half of the phase, and then select the other price. 

For the sequence of time steps when statistical tests are done, we use $L=\lfloor\log_2(T/e)/2\rfloor$ and
\begin{align}\label{eq:def_delta_n}
    \tilde{\Delta}_l=2^{-l}\,,\ n_l = \left\lceil\frac{2\log(T\tilde{\Delta}^2_l)}{\tilde{\Delta}^2_l}\right\rceil\,\quad\forall l=1\,,\ 2\,,\ \ldots\,, L\,.
\end{align} 

\vspace{1mm}
\noindent
We remark that the definitions in \eqref{eq:def_delta_n} follows from \cite{AO10}. For $l=1\,,\,2\,,\,\ldots\,,$ whenever both $N_1(t)$ and $N_2(t)$ arrives at $n_l\,,$ \gdp~would first compute $\hat{\lambda}_1(t)$ and $\hat{\lambda}_2(t)\,,$ respectively. Then, the statistical test states that if 
\begin{align}\label{eq:test}
    \hat{\lambda}_{\max}(t) - \hat{\lambda}_{k}(t)>\sqrt{\frac{2\log(T\tilde{\Delta}^2_l)}{n_l}}\,
\end{align}
holds for either $k=1$ or $k=2\,,$ the corresponding price would be removed and \gdp~would stay with the other price for the rest of the time horizon.

\vspace{2mm}

\noindent \underline{\textbf{Case 2. $M\geq T^{2/3}$}} 

\noindent
In this case, if we continue to follow the design for the case when $M<T^{2/3}\,,$ the expected regret incurred by the refund would be overwhelming due to the relatively large price protection period. Therefore, \gdp~would simply select each price for $N=\lceil T^{2/3}\rceil$ number of times (first $p_1$ and then $p_2$) and then selects the price with larger empirical mean reward for the rest of the time, \ie,
\begin{align*}
    i_t=\begin{cases}
    1&\text{ when }t\leq N;\\
    2&\text{ when } N<t\leq 2N;\\
    \argmax_{k}\hat{\lambda}_k(2N+1)&\text{ when } 2N<t.
    \end{cases}
\end{align*}

\vspace{1mm}

\subsection{Regret Analysis and Design Rationale}\label{sec:two_price_regret}
In this section, we analyze the performance of \gdp, and provide the design rationale as well as a detailed comparison between \gdp~and existing phased exploration algorithms. We begin with the case where $M<T^{2/3}\,.$

\vspace{2mm}
\begin{proposition}\label{prop:two_price_ub1}
    When $M< T^{2/3}\,,$ the expected regret of \gdp~is $O(\sqrt{T}+M\log\log(T))\,.$
\end{proposition}
\begin{proof}[Proof Sketch]
The complete proof of this Proposition is provided in \cref{sec:prop:two_price_ub1}.

We first note that by virtue of \gdp, there are at most $B=\log_2\log(T)$ many times of selecting $p_1$ after $p_2\,.$ Therefore, the price protection guarantee would incur at most $O(M\log\log(T))$ regret.

To proceed, we assume w.l.o.g. that $p_1$ is the optimal price, \ie, $\lambda_1\geq\lambda_2\,,$ and denote $\Delta = \lambda_1 - \lambda_2\,,$
as the difference in expected reward between the two prices. We consider the case where $\Delta\geq\gamma$ for some $\gamma\geq \sqrt{e/T}$ whose value is to be specified (see Theorem 3.1 of \citealt{AO10}) as when $\Delta<\gamma,$ the expected regret is at {most $T\Delta\leq T\gamma$ (excluding regret from price protection).} 

Let $l^*$ be the smallest $l$ such that $\tilde{\Delta}_l< \Delta/2\,$  and $t^*$ be the first time step where both $N_1(t)$ and $N_2(t)$ exceeds $n_{l^*}\,.$ We remark that $t^*$ can be much larger than $2n_{l^*}$ due to the exploration schedule is implemented based on $t_b$'s. Moving forward, we consider two different cases, depending on whether either price is removed before the beginning of time step $t^*\,.$ 

Similar to the proof of Theorem 3.1 in \cite{AO10}, defining $t^{(l)}$ as the first time step when both $N_1(t)$ and $N_2(t)$ are at least $n_l\,,$ by standard concentration of measure arguments (\eg, Hoeffding's inequality), $\hat{\lambda}_k(t^{(l)})$ and $\lambda_k$ would be at most $\sqrt{{\log(T\tilde{\Delta}_l^2)}/{2n_l}}$ apart from each other for both $k\in\{1\,,\,2\}$ with probability at least $1-O(T/\Delta^2)$. One can thus establish that $p_1$ is unlikely to be eliminated while $p_2$ is very likely to be eliminated after time step $t^*\,$ (as $\sqrt{{\log(T\tilde{\Delta}_{l^*}^2)}/{2n_{l^*}}}<\tilde{\Delta}_{l^*}/4\leq\Delta/4$ by definition of $l^*$). Therefore, the expected regret incurred by mistakenly eliminating $p_1$ or not eliminating $p_2$ after time step $t^*$ is at most $O(1/\Delta)\,$ (note that the regret is at most $T\Delta$ excluding refund).

Different than the proof of Theorem 3.1 in \cite{AO10}, however, the analysis for expected regret incurred by properly eliminating $p_2$ before time $t^*$ can be dramatically different. This is because, $t^*\,,$ the first time step when $N_k(t)\geq n_{l^*}$ holds for both $k\in\{1\,,\,2\}\,,$ can be much larger than $2n_{l^*}$ due to the asynchronous price exploration and price elimination schedules (see the forthcoming \cref{remark:async} for a discussion on the benefit of this design).

To this end, we let $b^*$ be the first phase such that $t_b\geq 2n_{l^*}\,,$ which means $t^*\in[t_{b^*-1}+1\,,\,t_{b^*}]\,.$ Suppose $p_2$ is eliminated before phase $b^*,$ we must have that until the end of phase $b^*-1,$ $p_2$ is selected by at most $2n_{l^*}$ times, and would incur an expected regret of order $O(\log(T\Delta^2)/\Delta)\,.$ Suppose $p_2$ is eliminated during phase $b^*\,,$ we note that in phase $b^*,$ if $p_2$ is selected first in phase $b^*,$ it would be selected for at most $t_{b^*}/2$ times; Otherwise, it would be selected for at most $n_{l^*}$ times. By our design of \gdp, it can be shown that the probability that $p_2$ would be selected first during phase $b^*$ is at most $O(1/(\Delta\sqrt{t_{b^*-1}}))\,.$
Putting these together, we could see that the expected regret under this case is at most 
$O(\sqrt{T}+ {\log(T\Delta^2)}/{\Delta})\,,$ where we utilize the definition of $t_b$'s in \eqref{eq:def_u_t} and hence $t_b/\sqrt{t_{b-1}}=O(\sqrt{T})$.

Altogether, the expected regret (excluding refund from price protection) of \gdp~is of order
\vspace{1mm}
\begin{eqnarray*}
O\left(T\Delta\bm{1}[\Delta\leq\gamma]+\left(\sqrt{T}+\frac{1+\log(T\Delta^2)}{\Delta}\right)\bm{1}[\Delta\geq\gamma]\right)\,.
\end{eqnarray*}

\noindent
By setting $\gamma=\Theta(1/\sqrt{T})\,,$ the expected regret (excluding refund from price protection) is at most $O(\sqrt{T})\,.$ The statement follows by further incorporating the expected regret from the refund.
\end{proof}

We now turn our attention to the case when $M\geq T^{2/3}\,.$

\vspace{2mm}
\begin{proposition}\label{prop:two_price_ub2}
    When $M\geq T^{2/3}\,,$ the expected regret of \gdp~is $O(T^{2/3})\,.$
\end{proposition}

\vspace{2mm}
The complete proof of this proposition is provided in \cref{sec:prop:two_price_ub2}. We point out that in this proposition, we save a $O(\log_e(T))$ factor when compared to the conventional results (see, \eg, section 1.2 of \citealt{Slivkins19}) by avoiding a direct application of the Hoeffding's inequality. Combining the above, we have the following theorem regarding the regret upper bound of \gdp.

\vspace{2mm}
\begin{theorem}\label{thm:two_price_ub}
    The expected regret of \gdp~is $O(\sqrt{T}+\min\{M\log\log(T)\,,\,T^{2/3}\})\,.$
\end{theorem}

\vspace{2mm}
This theorem follows immediately from \cref{prop:two_price_ub1} and \cref{prop:two_price_ub2}. 

\vspace{2mm}
\begin{remark}[\textbf{On Phased Exploration Design}]\label{remark:rationale}
The phased exploration design of \gdp~is indeed quite natural. As one can clearly see, whenever there is a price drop from $p_2$ to $p_1\,,$ refund can be incurred due to the price protection guarantee. This prompts us to consider limiting the price changes. As shown in \cite{GaoHRZ19,LeviX21}, phased exploration algorithms are able to provide nearly optimal performance when the number of policy updates and/or action changes are limited, we thus follow a similar design for \gdp.
\end{remark}

\vspace{2mm}
\begin{remark}[\textbf{The Benefit of Asynchronous Design ($M<T^{2/3}$)}]\label{remark:async}
Phased exploration algorithms have been widely used in the multi-armed bandits literature. Nevertheless, almost all existing phased exploration algorithms consider exclusively synchronous exploration and elimination schedules, \ie, elimination could only take place at the end of a phase. Among others, \cite{AO10} considers an exponentially growing exploration schedule, \ie, $t_b$'s would be defined similar to $n_l$ in \eqref{eq:def_delta_n} and there are $\Theta(\log(T))$ phases. When $M<T^{2/3}\,,$ directly applying their algorithm to our setting would result in an expected regret of order $O(\sqrt{T}+M\log(T))\,.$ On the other hand, \cite{GaoHRZ19} considers a double exponentially growing exploration schedule, \ie, $t_b$'s would be defined as in \eqref{eq:def_u_t} and there are $\Theta(\log\log(T))$ phases. Hence, directly applying their algorithm to our setting would result in an expected regret of order $O(\sqrt{T}\log(T)+M\log\log(T))\,.$ The degradation from $\sqrt{T}$ to $\sqrt{T}\log(T)$ in the first term is due to the fact that, \cite{GaoHRZ19} aims at minimizing policy adaptivity (\ie, number of policy updates), and hence, a sub-optimal price may be retained for a long period of time even when it is detected (because elimination only takes place at the end of a phase). In return, with exponentially fewer number of phases, \cite{GaoHRZ19} reduces the amount of refund from $O(M\log(T))$ to $O(M\log\log(T))$.

\begin{figure}[h]
	\centering
	\subfigure[\cite{GaoHRZ19}has low refund, but does not eliminate $p_2$ promptly]{\includegraphics[width=10.6cm,height=5.3cm]{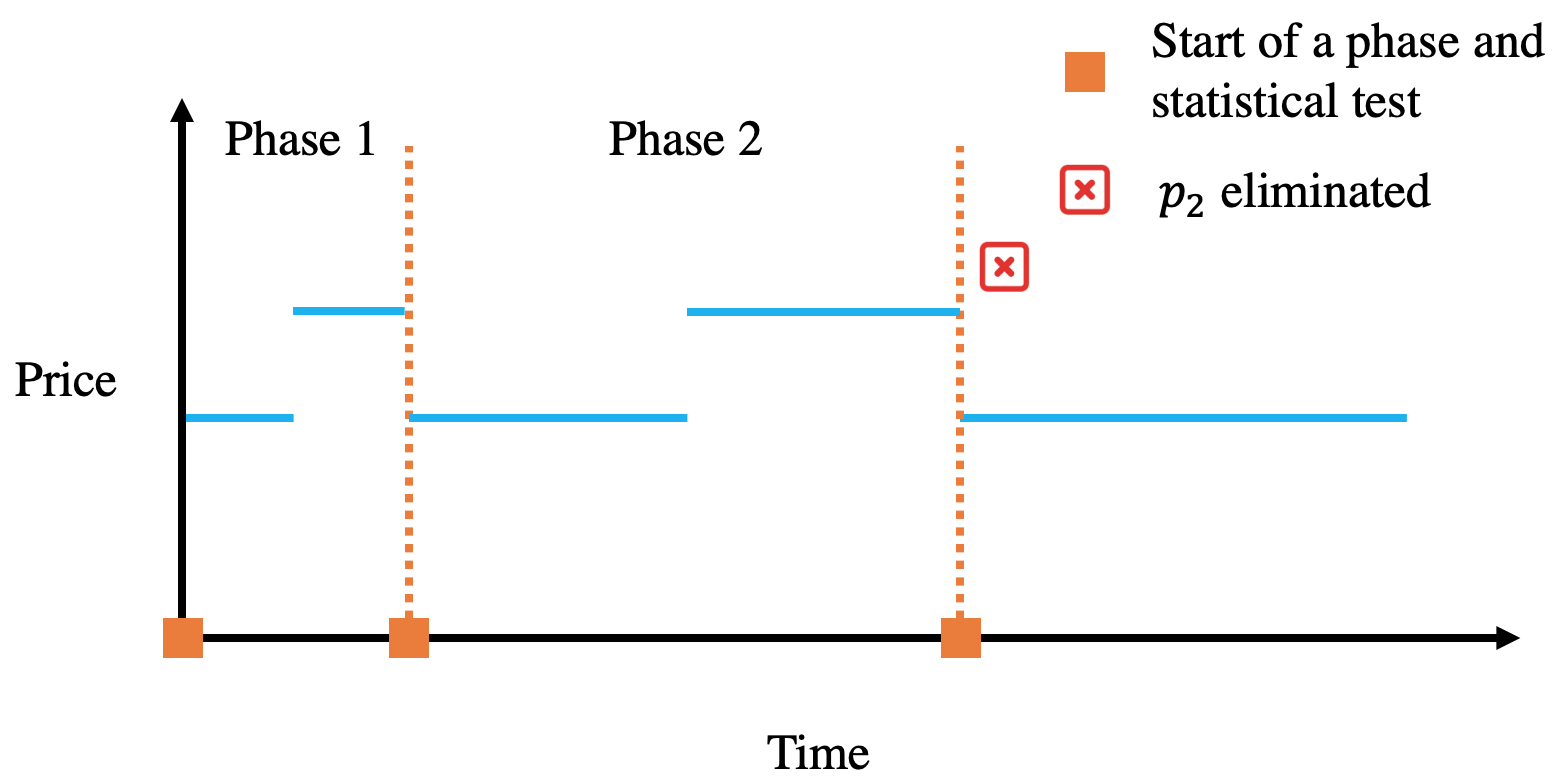}}\vspace{-2mm}
	\subfigure[\cite{AO10} eliminates $p_2$ promptly but with larger refund]{\includegraphics[width=10.6cm,height=5cm]{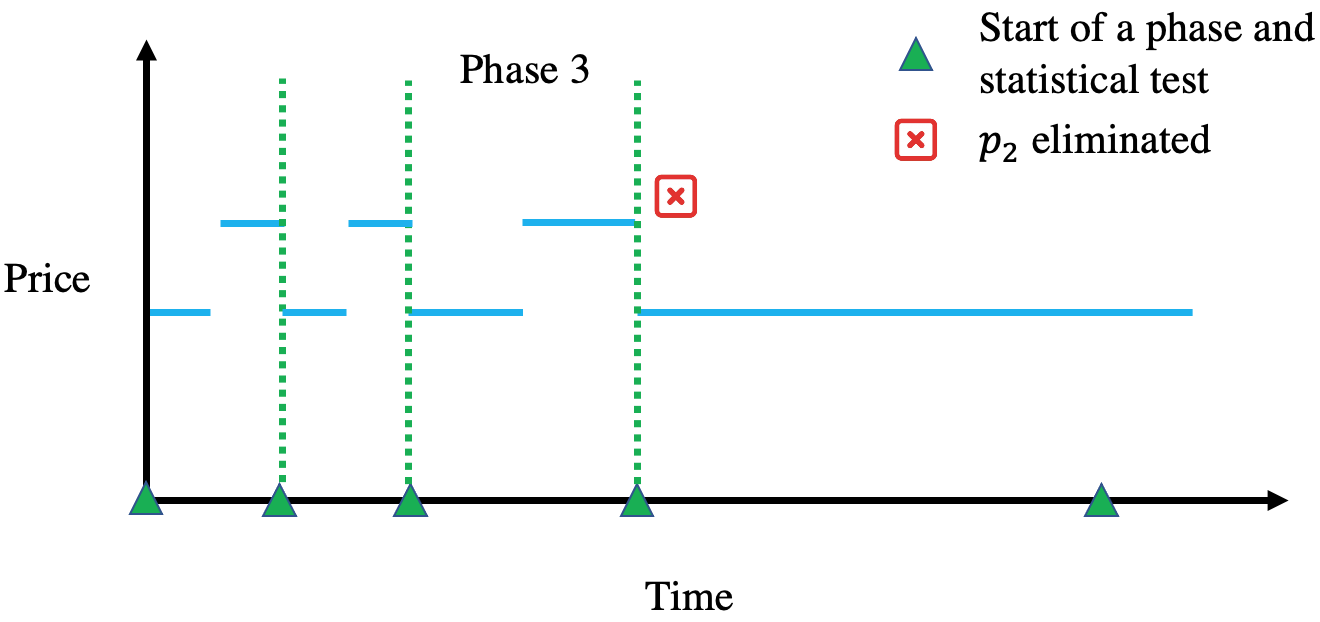}}\vspace{-2mm}
	\subfigure[\gdp~ensures $p_2$ is eliminated promptly while enjoying low refund]{\includegraphics[width=10.6cm,height=5.2cm]{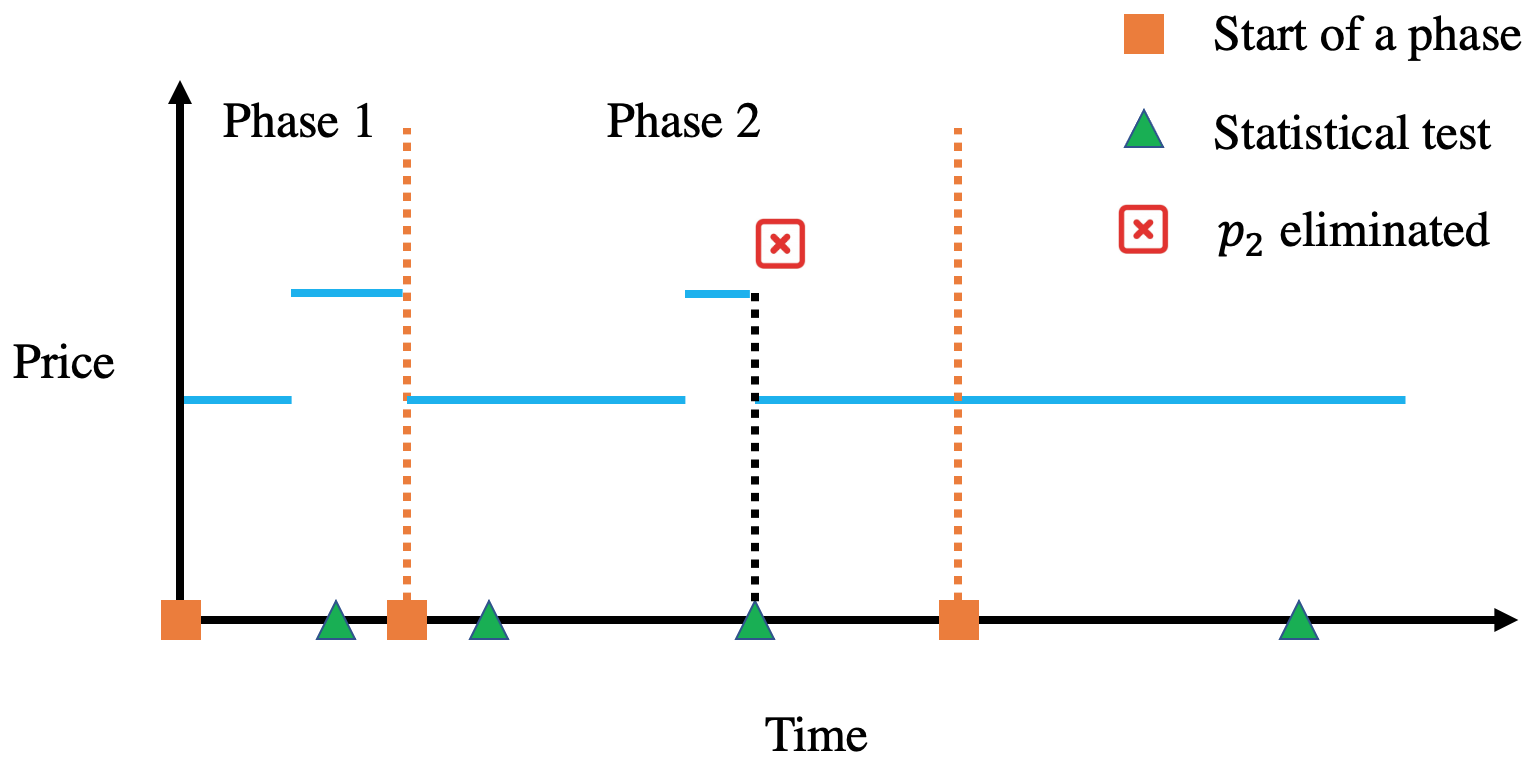}}
	\caption{Different phased exploration algorithms}
	\label{fig:gdp}
\end{figure} 

As such, one would prefer \cite{AO10} when $M<\sqrt{T}$ while \cite{GaoHRZ19} when $\sqrt{T}\leq M<T^{2/3}\,,$ and the expected regret of this mixture is
\begin{align*}
	 \begin{cases}
	    O(\sqrt{T})&\text{ when } M<\sqrt{T}\,;\\
	    O(\sqrt{T}\log(T))&\text{ when }  M=\sqrt{T}\,;\\
	    O(M\log\log(T))&\text{ when } \sqrt{T}<M< T^{2/3}\,.
	    \end{cases}
\end{align*}
It is thus interesting to see if one can acquire the same (or even improved) guarantee with an one-size-fit-all algorithm. In the above, \gdp~provides an affirmative answer, and it enjoies an improved regret bound when $M=\sqrt{T}\,.$ As pointed out in the sketched proof of \cref{prop:two_price_ub1}, this is critically made possible by following the exploration schedule of \cite{GaoHRZ19} while elimination schedule of \cite{AO10}. With these, the asynchronous design of \gdp~can simultaneously ensure that it only has $\Theta(\log\log(T))$ phases while the sub-optimal price is eliminated in time.

We refer to \cref{fig:gdp} for a conceptual comparison of the three synchronous versus asynchronous designs. There, we assume that $p_1$ is the optimal price. 

\end{remark}

\subsection{Phase Transitions and Discussions}\label{sec:two_price_phase}

In this section, we discuss an important implication of \cref{thm:two_price_lb} and \cref{thm:two_price_ub}. By comparing the results of these two theorems, one can easily verify that the regret upper bound $O((\sqrt{T}+\min\{M\log\log(T)\,,\ T^{2/3}\})$ achieved by \gdp, after ignoring the $O(\log\log(T))$ factor, is unimprovable within all non-anticipatory pricing strategies. Motivated by this observation, we define the optimal (worst case) regret $R^*(T,M)$ as
\begin{equation}
    R^*(T,M)=\inf_{\pi}\sup_{\lambda_1\,,\,\lambda_2\in[0\,,\,1]}  ~\regret(\pi).\label{eq:def_opt_regret1}
\end{equation}
Hence, $R^*(T,M)$ can be interpreted as the statistical complexity of the pricing problem under price protection guarantee when there are only two prices, in a sense that no admissible pricing strategy can perform uniformly better than this rate over all possible instances. We state a 
corollary.

\vspace{2mm}
\begin{corollary}\label{cor:opt_regret_two_price}
    The optimal regret defined in \eqref{eq:def_opt_regret1} for the two price case is
    \begin{equation*}
        R^*(T,M)=\tilde{\Theta}(\sqrt{T}+\min\{M\,,\, T^{2/3}\})\,.
    \end{equation*}
\end{corollary}

\vspace{2mm}
From \cref{cor:opt_regret_two_price}, one can clearly see that the increasing patterns of the optimal regret rate are different when the price protection period belongs to different ranges. As the price protection period $M$ increases, the optimal regret first remains at the level of ${\Theta}(\sqrt{T})$ when $M<\sqrt{T}$, and then gradually increases to $\tilde{\Theta}(M)$ when $\sqrt{T}\leq M<T^{2/3}\,,$ and finally reaches the level of  $\tilde{\Theta}(T^{2/3})$ when $M\geq T^{2/3}\,.$ This is depicted in Figure \ref{fig:phase-transition}. One can see that there are three distinct ranges of $M$, \ie, $M<\sqrt{T}\,,$ $\sqrt{T}\leq M<T^{2/3}\,,$ and $M\geq T^{2/3}\,,$ referred to as three \textit{phases}, and the optimal regret shows different property in each phase. We refer to the transitions between the regret-increasing transitions in different phases as \textit{phase transitions}. 
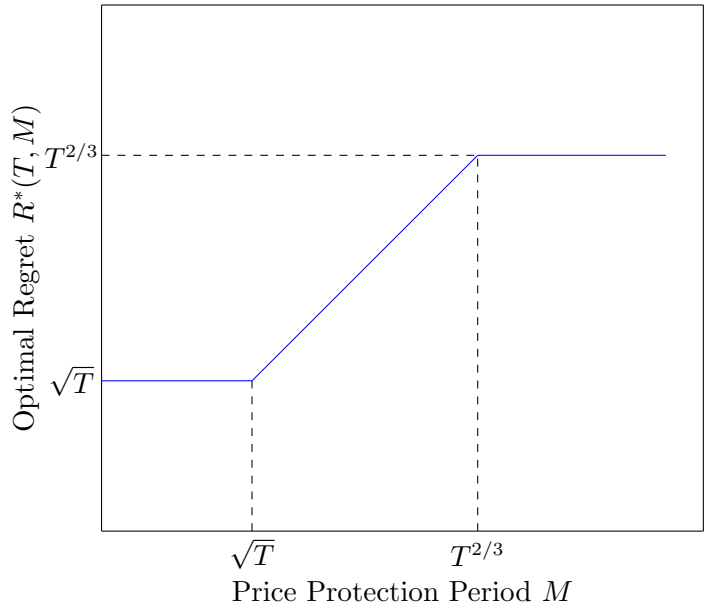
\begin{figure}[!ht]
    \centering
    \begin{tikzpicture}
    \draw (0,0)--(0,7);
    \draw (0,0)--(8,0);
    \draw (0,7)--(8,7);
    \draw (8,0)--(8,7);
    \node at (4,-0.8) {Price Protection Period $M$};
    \node [rotate=90] at (-1,3.5) {Optimal Regret $R^*(T,M)$};
    \node at (2,-0.3) {$\sqrt{T}$};
    \node at (5,-0.3) {$T^{2/3}$};
    \node at (-0.4,2) {$\sqrt{T}$};
    \node at (-0.4,5) {$T^{2/3}$};
    \draw[color=blue] (0,2)--(2,2);
    \draw[color=blue] (2,2)--(5,5);
    \draw[color=blue] (5,5)--(7.5,5);
    \draw[dashed] (2,0)--(2,2);
    \draw[dashed] (5,0)--(5,5);
    \draw[dashed] (0,5)--(5,5);
    \end{tikzpicture}
    \caption{Phase Transitions with Different Price Protection Period $M$}
    \label{fig:phase-transition}
\end{figure}

We also remark that, with some slight modifications to the proofs, all the results presented in Theorems \ref{thm:two_price_lb} and \ref{thm:two_price_ub} continue to 
hold in the asymptotic regime where the relationship between $M$ and $T$ are depicted by asymptotic notations rather than precise inequality signs, \ie letting $A\lesssim B$ be $A=O(B)\,$, we have: 
\begin{align*}
    R^*(T,M)=\begin{cases}
           \tilde{\Theta}(\sqrt{T})&\text{ when }M\lesssim\sqrt{T}\\
           \tilde{\Theta}(M)&\text{ when }\sqrt{T}\lesssim M\lesssim T^{2/3}\\
           \Theta(T^{2/3})&\text{ when }T^{2/3}\lesssim M\,,
        \end{cases}
\end{align*}

\vspace{1mm}
\noindent
Interestingly, the above findings indicate that when $M\lesssim\sqrt{T}\,,$ the presence of price protection guarantee has virtually no significant impact to the statistical complexity of data-driven dynamic pricing problem; while when $\sqrt{T}\lesssim M\,,$ the statistical complexity of the problem starts to increase due to the price protection guarantee. Critically, the complexity of the problem will never exceed $\Theta(T^{2/3})$ even if $M$ becomes very large. As pointed out in \cref{sec:intro}, different companies may set different length of price protection period, our results thus provide a thorough and complete characterization of how the presence of price protection guarantee can impact the statistical complexity of the dynamic pricing problem.

\section{General Learning and Earning under Price Protection} \label{sec:general}
In this section, we extend our findings to the general case where there could be $K~(\geq 2)$ prices. We first establish the lower bound for this general case in \cref{sec:lb}. Then, in \cref{sec:ub}, we present a new algorithm \gdppp, which achieves this lower bound up to logarithmic factors. Finally, in \cref{sec:add_disc}, we provide further discussions of the results.

\subsection{Lower Bound}\label{sec:lb}
In this section, we present a lower bound for the general case. The lower bound shares a similar spirit as that in \cref{thm:two_price_lb}. 

\vspace{2mm}
\begin{theorem}\label{thm:k_price_lb}
	For any $T, \, K, \, M$ satisfying $2\leq K<M\leq T$ and any non-anticipatory policy $\pi$, there exists $P=\{p_1,p_2,\dots,p_K\}\subset[0,1]$ and the corresponding distributions of demand $D_t(p_1),$ $D_t(p_2),\ldots,D_t(p_K)$ such that the regret of $\pi$ is $\Omega(\sqrt{KT}+\min\{M\,,\,K^{1/3}T^{2/3}\})\,.$
\end{theorem}

\vspace{2mm}
The complete proof of this result follows the same spirit as that of \cref{thm:two_price_lb}, and it is provided in Section \ref{sec:thm:k_price_lb} of the Appendix.

\subsection{Generalized Algorithm for Learning and Earning under Price Protection (\gdppp)}\label{sec:ub}
In this section, we propose \gdppp, an algorithm for the general problem of learning and earning in the presence of price protection guarantee. Below, we first present a generic template of \gdppp. The precise algorithm would be instantiated by different sub-algorithm(s) $\cA\,,$ depending on the relationship between $M$, $K$ and $T$. 

Similar to our algorithm for the two price setting in \cref{sec:two_price_ub}, all sub-algorithms $\cA$ considered here would enjoy a phased exploration nature, which would dynamically maintain the set of plausible prices. During each phase, $\cA$ would uniformly explore each plausible prices (starting form the lowest price to the highest, to avoid excessive refund). At the end of a phase, prices that are deemed to be sub-optimal with high probability (via a statistical test) would be eliminated. This procedure iterates until the end of the time horizon. 

Different sub-algorithms $\cA$ might have different price exploration and price elimination schedules. Specifically, each sub-algorithm $\cA$ would choose a sequence of time steps $t_0~(=0)\,,\ t_1, \ldots, t_{B_{\cA}}~(=T)\,,$ where phase $b~(=1,\ldots,B_{\cA})$ would begin at time step $t_{b-1}+1\,,$ and ends at time step $t_b\,.$ The set of plausible prices at the beginning of each phase $b$ is denoted as $P_b$ with $P_1=P$, representing that every price is plausible at the beginning. During phase $b\,,$ all the prices in $P_b$ would be selected for equal number of times in the ascending value order. At the end of phase $b\,,$ a statistical test $\texttt{Test}_{\cA\,,\, b}\,:P\to P$ that depends on the algorithm and all available historical information (the dependence on which is omitted) will be performed. All the prices that are sub-optimal with high probability would be eliminated, and the remaining would constitute $P_{b+1}\,.$ A formal description of this framework is presented in \cref{alg:gdppp}.

\vspace{2mm}
\begin{remark}[\textbf{Exploration Schedule}]
We point out that in some cases, $t_b$ may be chosen adaptively, \ie, depending on $P_b$ (see, \eg, \cite{AO10}). In this case, \gdppp~would not pre-specify the sequence $t_b$'s. Instead, it would determine this on the fly. As such, in \cref{alg:gdppp}, $t_b$'s are not defined in the initialization step, but are computed at the beginning of each phase.
\end{remark}
\begin{algorithm}[htbp]
\SingleSpacedXI
\caption{Generalized Algorithm for Learning and Learning under Price Protection (\gdppp)} \label{alg:gdppp}
\begin{algorithmic}[]
    \State \textbf{Input:} Sub-algorithm $\cA\,,$ price set $P\,,$ time horizon $T\,,$ price protection period $M\,.$ 
    \State \textbf{Initialization:} $P_1\leftarrow P$
    \For {Phase $b=1\,,\ \ldots\,,\ B_{\cA}$}
    \State Determine $t_b,$ the last time step of phase $b$
    \State Select each price in $P_b$ for equally many times in ascending value order during time steps $t_{b-1}+1\,,\ \ldots\,, \ t_b\,.$ Provide potential refund to the customers simultaneously
    \State Compute the set of plausible price for the next phase: $P_{b+1}\leftarrow \texttt{Test}_{\cA\,,\,b}(P_b)$
    \EndFor
\end{algorithmic}
\end{algorithm}

To see an intuition of our design, we present a result that could convert the regret of $\cA$ when no price protection guarantee is offered to the regret of $\cA$ when there is price protection guarantee. Defining $\regret'(\cA)= T\max_{k\in[K]}\lambda_k-\sum_{t=1}^T\lambda_{i_t}$ as the expected regret of $\cA$ when no price protection guarantee is offered and $p_{i_t}$ is selected according to $\cA\,,$ we have the following proposition.

\vspace{2mm}
\begin{proposition}\label{prop:conversion}
    Suppose the price protection period is $M\,,$ for any phased exploration algorithm $\cA$ that follows the template in \cref{alg:gdppp}, we have
    \begin{align*}
        \regret(\cA)= \regret'(\cA) + O(MB_{\cA})\,.
    \end{align*}
\end{proposition}

\vspace{2mm}
The proof of this proposition follows immediately from the design of $\cA$, i.e., 
since refund could only happen when we switch from phase $b-1$ and enter the new phase $b$, and since there are at most $B_{\cA}-1$ switches with each such switch incurring at most $O(M)$ refund.

In what follows, we instantiate \gdppp~with different sub-algorithms $\cA$ for different value of $M\,,$ and establish the corresponding upper bound for expected regret.

\vspace{2mm}
\noindent \underline{\textbf{Case 1. $M\leq\sqrt{KT}$}} 

\noindent
In this case, \gdppp~adopts the phased exploration version of the UCB algorithm discussed in \cite{AO10}. We set the length of phase $b$ as $$n'_b=\left(\left\lceil 2^{2b+1}\log(2^{-2b}T)\right\rceil-\left\lceil2^{2b-1}\log(2^{-2b+2}T)\right\rceil\right)|P_b|\,,$$ \ie, all remaining prices are selected $\lceil 2^{2b+1}\log(2^{-2b}T)\rceil$ times at the end of phase $b\,.$ The time step $t_b$ is thus defined as $$t_b=\min\left\{\sum_{m=1}^b n'_m\,,\,T\right\}\,,$$ and with some abuse of notations, the statistical test at the end of each phase $b\,,$ which also outputs the remaining price for next phase, is defined as
\begin{equation*}
    P_{b+1}=\texttt{Test}_{\cA\,,\,b}(P_b)=\left\{i\in P_b\,:\,\hat{\lambda}_i(t_b+1)+\sqrt{\dfrac{\log(2^{-2b}T)}{2N_i(t_b+1)}}\geq\max_{j\in P_b}\hat{\lambda}_j(t_b+1)-\sqrt{\dfrac{\log(2^{-2b}T)}{2N_j(t_b+1)}}\right\}.
\end{equation*}

\vspace{2mm}
\noindent \underline{\textbf{Case 2. $\sqrt{KT}< M<K^{1/3}T^{2/3}$}} 

\noindent 
In this case, \gdppp~adopts the batched multi-armed bandit algorithm discussed in \cite{GaoHRZ19}. We set $t_b=\min\{\lceil (\sqrt{eT})^{2-2^{-b}}\rceil\,,\,T\}$, and the statistical test at the end of each phase $b$ is defined as
\begin{equation*}
    P_{b+1}=\texttt{Test}_{\cA\,,\,b}(P_b)=\left\{i\in P\,:\,\hat{\lambda}_i(t_b+1)+\sqrt{\dfrac{\log (KT)}{48N_i(t_b+1)}}\geq\max_{j\in P_b}\hat{\lambda}_j(t_b+1)-\sqrt{\dfrac{\log (KT)}{48N_j(t_b+1)}} \right\}.
\end{equation*}


\vspace{2mm}
\noindent \underline{\textbf{Case 3. $M\geq K^{1/3}T^{2/3}$}} 

\noindent
In this case, {\gdppp~adopts the explore-then-commit algorithm discussed in section 6 of \cite{LS18}.} We set $t_1=K\lceil K^{-2/3}T^{2/3}\rceil\,,$ $t_2=T\,,$ and
\begin{equation*}
    \texttt{Test}_{\cA\,,\,1}=\left\{\argmax_{i\in [K]}\hat{\lambda}_i(t_1+1)\right\}.
\end{equation*}
In other words, \gdppp~first selects each price $\lceil K^{-2/3}T^{2/3}\rceil$ times. Then, it chooses the price with the highest empirical mean reward throughout the rest of the time periods.

\vspace{2mm}
The following theorem gives us the regret upper bound of \gdppp.

\vspace{2mm}
\begin{theorem}\label{thm:k_price_ub2}
    For any $T,\,K,\, M,$ the upper bound of \gdppp~is $\tilde{O}(\sqrt{KT}+\min\{M\,,\,K^{1/3}T^{2/3}\})\,.$
\end{theorem}

\vspace{2mm}
This theorem makes use of the results from \cite{AO10,GaoHRZ19} and \cref{prop:conversion}, and its complete proof is provided in \cref{sec:thm:k_price_ub2}.

\vspace{2mm}
\begin{remark}[\textbf{On the Sub-Algorithms for Case 1 and Case 2}]
Our decision to apply different algorithms for Case 1 and Case 2 is purely for the purpose of optimizing the logarithmic factors in the regret upper bound (as previously discussed in \cref{remark:async}). One can definitely combine these two cases by applying the same algorithm (\eg, either using the phased exploration version of the UCB algorithm discussed in \cite{AO10} or the batched multi-armed bandit algorithm discussed in \cite{GaoHRZ19}), and the regret upper bound stated in \cref{thm:k_price_ub2} still holds.
\end{remark}

\subsection{Additional Discussions}\label{sec:add_disc}
Combining \cref{thm:k_price_lb} and \cref{thm:k_price_ub2}, one can verify that the regret upper bound $\tilde{O}(\sqrt{KT}+\min\{M\,,\, K^{1/3}T^{2/3}\})$ achieved by \gdppp, after ignoring the logarithmic factors, is unimprovable within all non-anticipatory pricing strategies. Similar to \cref{sec:two_price_phase}, we define the optimal (worst case) regret $R^*(T\,,\,M\,,\,K)$ as
\begin{align}\label{eq:def_opt_regret2}
    R^*(T,\,M,\,K)=\inf_{\pi}\sup_{\{\lambda_k\}_{k=1}^K\subset [0,1]}~\texttt{Regret}(\pi)\,.
\end{align}
Here, $R^*(T,M)$ measures the statistical complexity of the general learning and earning problem under price protection. We also have the following corollary

\vspace{2mm}
\begin{corollary}\label{cor:opt_regret}
    The optimal regret defined in \eqref{eq:def_opt_regret2} is
    \begin{equation*}
        R^*(T,\,M,\,K)=\tilde{\Theta}(\sqrt{KT}+\min\{M\,,\, K^{1/3}T^{2/3}\})\,.
    \end{equation*}
\end{corollary}

\vspace{2mm}
From \cref{cor:opt_regret}, one can see that the optimal regret rate also exhibits three phases. As the price protection period $M$ increases, the optimal regret first remains at the level of $\tilde{\Theta}(\sqrt{KT})$ when $M<\sqrt{KT}$, then gradually increases according to $\tilde{\Theta}(M)$ when $\sqrt{KT}\leq M<K^{1/3}T^{2/3}\,,$ and finally reaches the level of  $\tilde{\Theta}(K^{1/3}T^{2/3})$ when $M\geq K^{1/3}T^{2/3}\,.$ All these results could also be extended to the asymptotic regimes (in terms of $M$ and $T$).

\vspace{2mm}
\noindent\textbf{Further Comparisons with Bandits with Limited Adaptivity:} With all the above discussions, it is clear that our setting is related to the setting of bandits with limited adaptivity (\eg, \citealt{GaoHRZ19,LeviX21}). In what follows, we provide a detailed comparisons between the two settings in terms of modeling, phase transitions of the optimal regret, and algorithm design.
\vspace{1mm}
\begin{itemize}
    \item \textbf{Modeling:} From the modeling perspective, bandits with limited adaptivity imposes a hard constraint on the amount of adaptivity (\eg, number of policy updates or action changes) for the decision-maker, and she is not allowed to violate this constraint. Different than theirs, a price drop in our setting would result in a $M$-dependent cost, which means price protection is closer to a soft constraint. In other words, the seller is allowed to change the price arbitrarily, but with the cost of a potential refund. We also emphasize that, in our setting, only price decrease may incur refund and price increase would not. In contrast, in bandits with limited adaptivity, any change of action or policy would contribute to the adaptivity constraint.
    
    \vspace{1mm}
    \item \textbf{Phase Transitions of Optimal Regret:} The above difference turns out to significantly change the landscape of the problem as we have already shown in the regret upper and lower bounds. In particular, our setting only exhibits three phases in the optimal regret as $M$ grows. In contrast, in bandits with limited adaptivity, more phases could be possible. For instance, in \cite{LeviX21}, phase transition takes place whenever the amount of adaptivity (referred to as switching budget in \cite{LeviX21}) is increased by $K-1$.
    
    \vspace{1mm}
    \item \textbf{Algorithm Design:} Last but not the least, our setting also requires a more proactive algorithmic framework, which means that the algorithms have to take different values of $M$ into account. For instance, \gdp~and \gdppp~would proactively decrease the number of phases as $M$ increases. When $M$ is not too large, they are similar to the classic phased exploration algorithm; when $M$ grows larger, they would eventually degenerates to a explore-then-commit algorithm. In contrast, algorithms in \cite{LeviX21,GaoHRZ19} can usually treat the adaptivity constraint as an input parameter and then apply a one-size-fit-all algorithm to solve the problem. 
\end{itemize}

\section{Numerical Experiments}\label{sec:numerical}
In this section, we conduct extensive numerical experiments to further evaluate the performance of the proposed algorithms.
\subsection{Numerical Experiments for Two Prices: \gdp~ vs. Modified Bandits Algorithms}\label{sec:numerical_two_price}
We first consider the two price setting. We compare the performance of \gdp~with the modified versions of UCB and Thompson Sampling. In \cref{sec:ucb}, we have demonstrated that both UCB and Thompson Sampling suffer linear regrets due to ignoring the impact of price protection guarantee. One natural idea is to modify these algorithms as follows: in each time step, the seller selects the price that maximizes her (estimated) instantaneous revenue (computed based on \eqref{eq:inst_rev} and historical data) plus information gain. Based on this idea, we present the following two heuristics: 

\vspace{2mm}
\noindent\textbf{UCB with price protection (UCB-PP): } In each time step, the seller selects the price that maximizes $\hat{\lambda}_k(t)-\texttt{Refund}_t(p_k|p_{t-1},\dots,p_{t-M})+\sqrt{\log(T)/N_k(t)}$, \ie,
\begin{equation*}
    i_t=\argmax_{k\in[K]}\hat{\lambda}_k(t)-\texttt{Refund}_t(p_k|p_{t-1},\dots,p_{t-M})+\sqrt{\dfrac{\log(T)}{N_k(t)}}.
\end{equation*}

\vspace{2mm}
\noindent\textbf{Thompson Sampling with price protection (TS-PP):} It maintains a posterior distribution over the reward parameters $\lambda_k$'s. At each time step, it first samples $\theta_k$ from the posterior distribution, and then adjusted them to $\tilde{\theta}_k(t)$, which can be viewed as the expected instantaneous revenue of selecting price $k$ at time step $t$. Finally, the algorithm selects the price that maximizes $\tilde{\theta}_k(t)$. The details of TS-PP is presented in Algorithm \ref{alg:tspp}.

\begin{algorithm}[htbp]
\SingleSpacedXI
\caption{Thompson Sampling with Price Protection (TS-PP) Algorithm} \label{alg:tspp}
\begin{algorithmic}[]
    \State \textbf{Input:} Price set $P\,,$ time horizon $T\,,$ price protection period $M\,.$ 
    \State \textbf{Initialization:} $S_k\leftarrow 0\,,\ F_k\leftarrow 0$ for all $k\in[K]$.
    \For {$t=1\,,\ \ldots\,,\ T$}
    \State For each $k\in [K]\,,$ sample $\theta_k(t)$ from the $\text{Beta}(S_k+1,F_k+1)$ distribution.
    \State Compute
    \begin{equation*}
        \tilde{\theta}_k(t)\leftarrow \theta_k(t)-\texttt{Refund}_t(p_k|p_{i_{t-1}},\dots,p_{i_{t-M}}).
    \end{equation*}
    \State Select price $i_t=\argmax \tilde{\theta}_k(t)$, and observe demand $D_t$.
    \State Perform a Bernoulli trial with success probability $p_{i_t}D_t$ and observe the output $r_t$.
    \State If $r_t=1$, then $S_{i_t}\leftarrow S_{i_t}+1$, else $F_{i_t}\leftarrow F_{i_t}+1$.
    \EndFor
\end{algorithmic}
\end{algorithm}

\noindent\textbf{Setup:} We test the algorithms on the following instance: the prices are set to be $p_1=1/3\,,$ $p_2=1\,,$ $D_t(p_1)\equiv 1\,,$ and $D_t(p_2)\sim\text{Bernoulli}(1/6)$. In this instance, the optimal price is $p_1$. $T$ varies from $1000$ to $20000$ with a step size of $1000$. For each $T$, we average the results over $10^4$ independent iterations. We consider both of the small $M$ and large $M$ regimes. For the small $M$ case, we choose $M=\lceil\sqrt{T}\rceil$ and, for large $M$, we choose $M=\lceil T^{3/4}\rceil$. For both sets of experiments, we report the regret of different algorithms as well as the portion of refunds in the corresponding regrets.

\vspace{2mm}
\noindent\textbf{Results:} The results for the case $M=\lceil\sqrt{T}\rceil$ are shown in \cref{fig:2p_small_M} whereas the results for the case $M=\lceil T^{3/4}\rceil$ are shown in \cref{fig:2p_large_M}. 
\begin{figure}[h]
    \centering
    \subfigure[Log-log plot of regret]{\includegraphics[height=6cm, width=8cm]{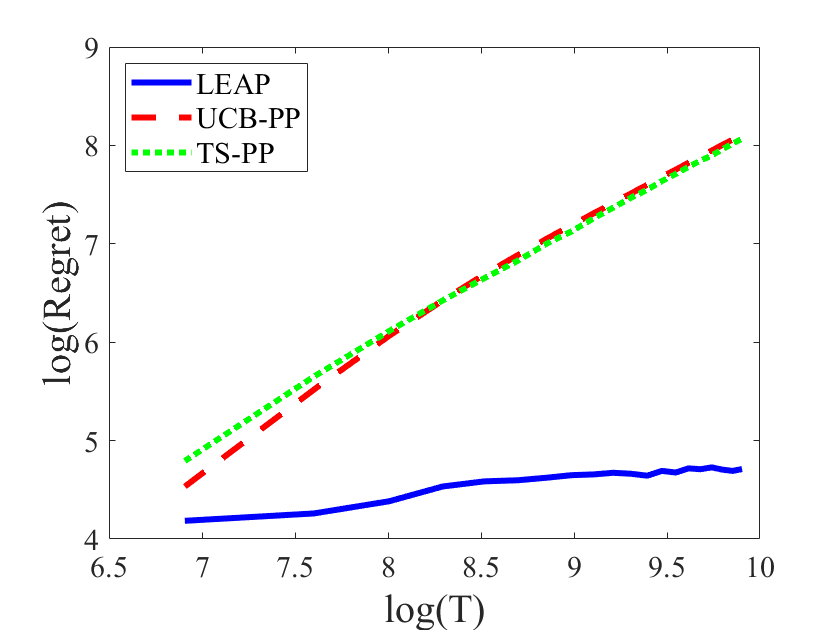}
    \label{fig:gdp-regret-loglog1}}
    \subfigure[Portion of refund in regret ]{\includegraphics[height=6cm, width=8cm]{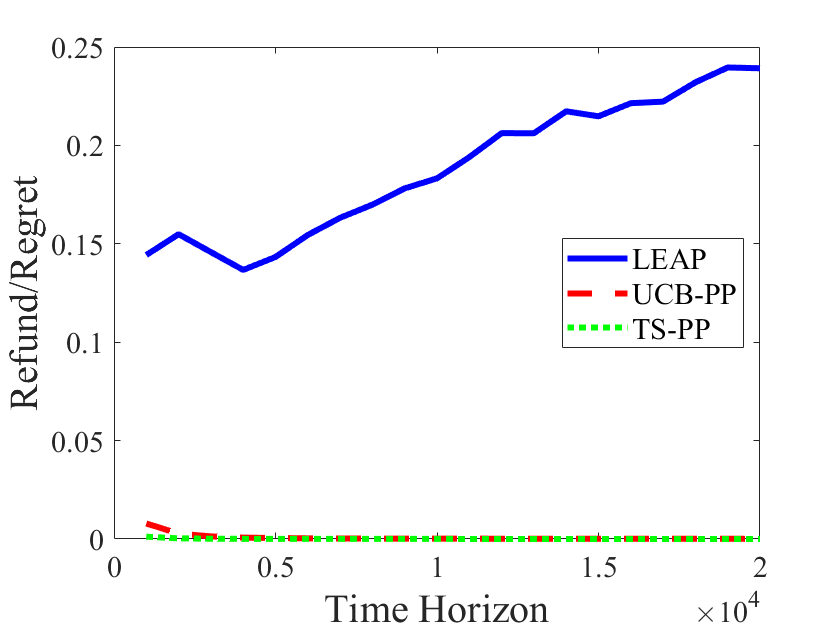}\label{fig:gdp-refund1}}
    \subfigure[Sample path of UCB-PP]{\includegraphics[height=6cm, width = 8cm]{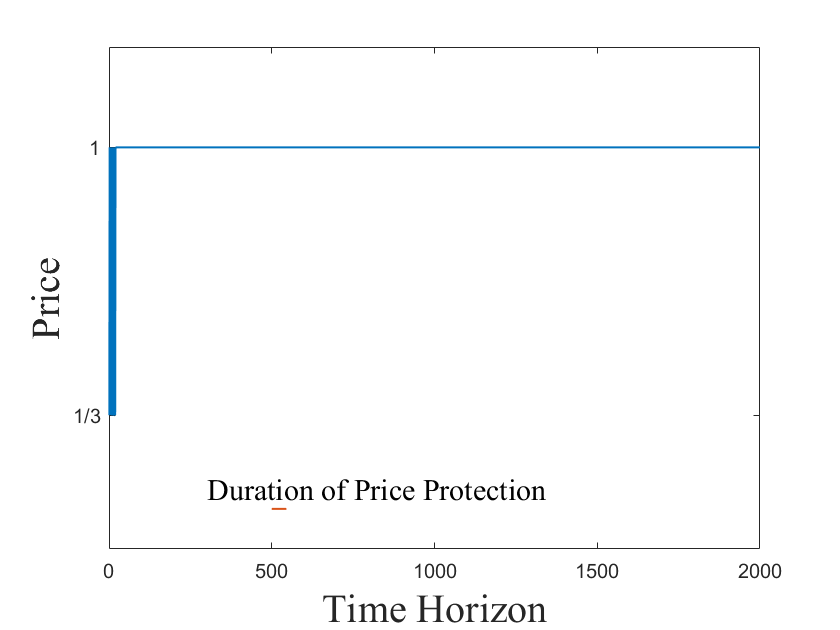}
    \label{fig:sample_ucbpp1}}
    \subfigure[Sample path of TS-PP]{\includegraphics[height=6cm, width = 8cm]{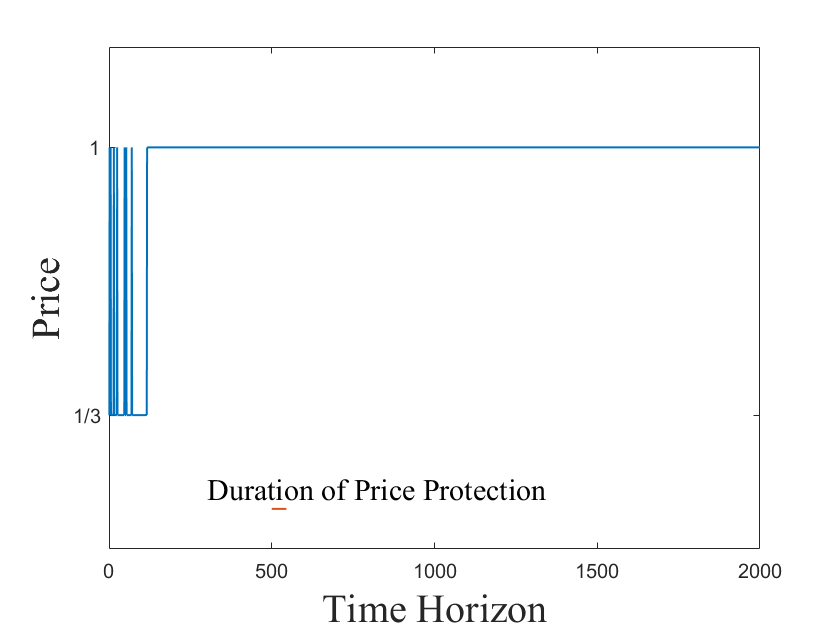}
    \label{fig:sample_tspp1}}
    \caption{Numerical experiments for two prices when $M=\lceil\sqrt{T}\rceil$}
    \label{fig:2p_small_M}
\end{figure}
\begin{figure}[h]
    \centering
    \subfigure[Log-log plot of expected regret]{\includegraphics[height=6cm, width=8cm]{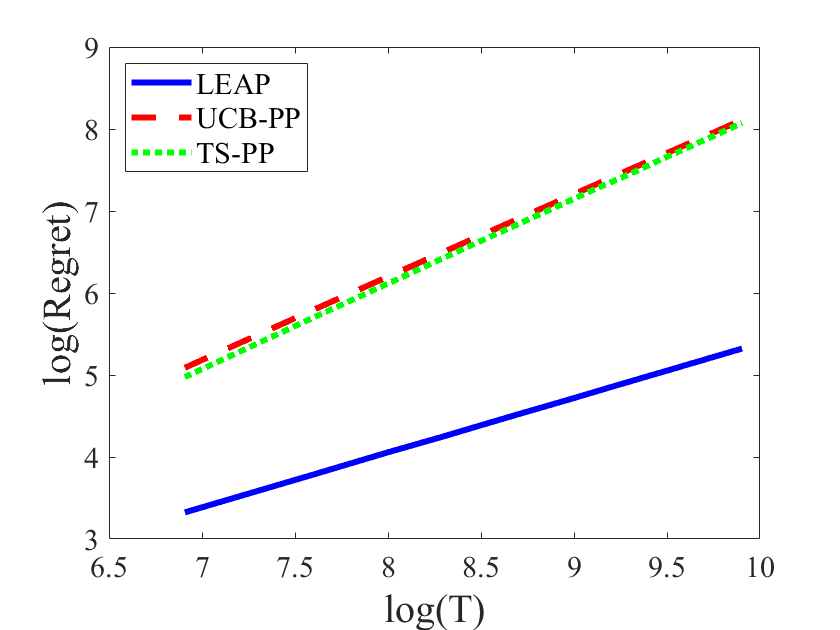}
    \label{fig:gdp-regret-loglog2}}
    \subfigure[Portion of refund in regret]{\includegraphics[height=6cm, width=8cm]{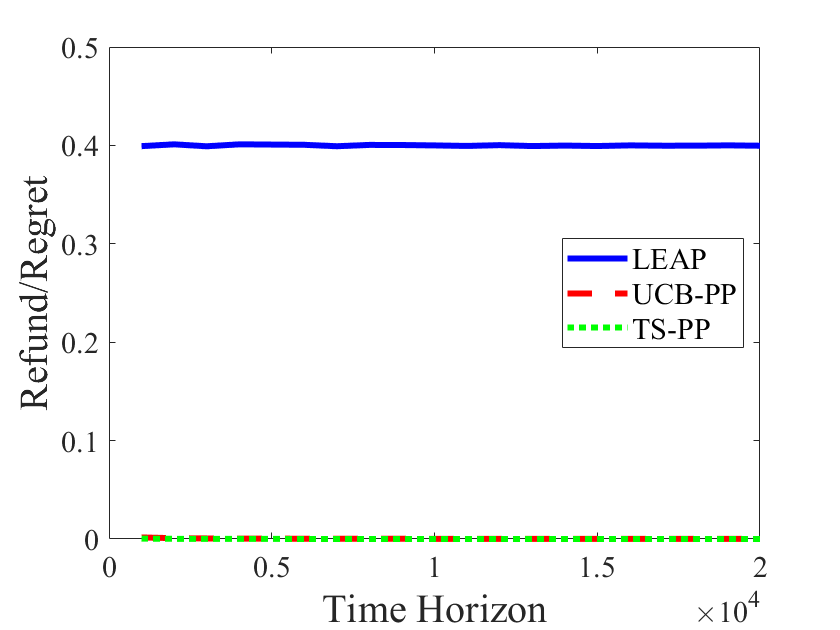}\label{fig:gdp-refund2}}
    \subfigure[Sample path of UCB-PP]{\includegraphics[height=6cm, width = 8cm]{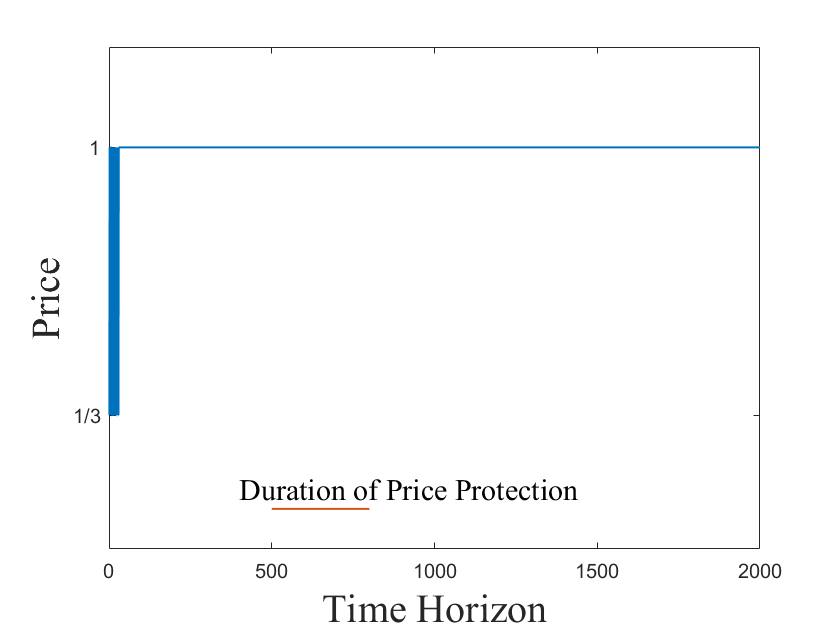}
    \label{fig:sample_ucbpp2}}
    \subfigure[Sample path of TS-PP]{\includegraphics[height=6cm, width = 8cm]{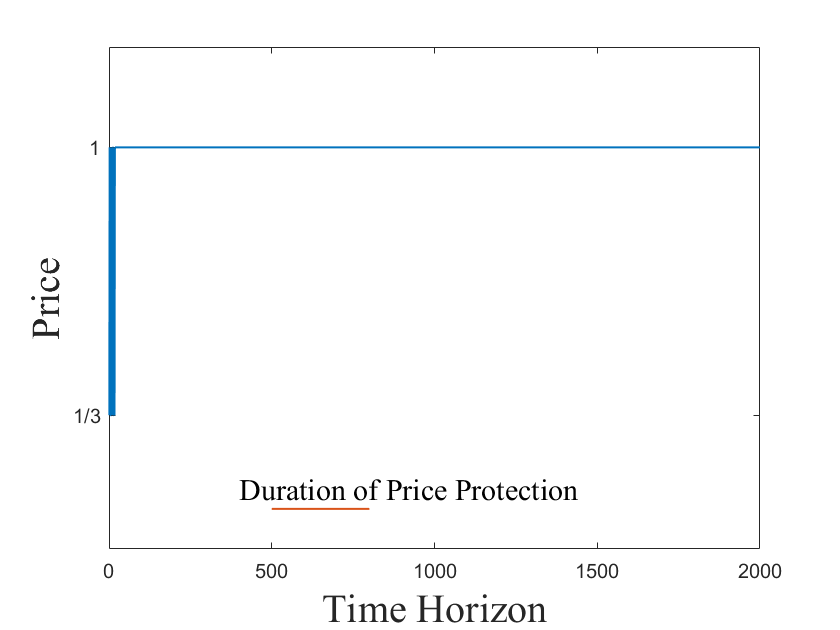}
    \label{fig:sample_tspp2}}
    \caption{Numerical experiments for two prices when$M=\lceil T^{3/4}\rceil$}
    \label{fig:2p_large_M}
\end{figure}

In \cref{fig:gdp-regret-loglog1} and \cref{fig:gdp-regret-loglog2}, we present the log-log plot for the averaged regret of \gdp, UCB-PP, and TS-PP. From the figures, one can clearly see that \gdp~is able to achieve significantly smaller regret compared to competing methods. One can also verify that \gdp's regret growth is roughly $O(T^{1/2})$ when $M=\lceil\sqrt{T}\rceil$ and $O(T^{2/3})$ when $M=\lceil T^{3/4}\rceil,$ which is consistent with our previous theoretical results (see, \eg, \cref{thm:two_price_ub}). In both cases, we can easily verify that neither UCB-PP nor TS-PP achieves a sub-linear regret \footnote{We remark that in log-log plot, if the regret growth line of an algorithm has a slope of $\alpha$, then the regret of this algorithm is $\Theta(T^{\alpha}).$ }.

In \cref{fig:gdp-refund1} and  \cref{fig:gdp-refund2}, we take a deeper look at the source of regret. From the results, we can see that refunds account for about 15\% to 25\% of \gdp's regret when $M=\lceil\sqrt{T}\rceil$ and it accounts for about $40\%$ of \gdp's regret when $M=\lceil T^{2/3}\rceil$. For UCB-PP and TS-PP, by incorporating refunds into decision-making, they do reduce the level of refund to an extremely low level (\ie, $<5\%$ in regret). However, such modifications come at the cost of high regret. As can be seen in \cref{fig:sample_ucbpp1}, \cref{fig:sample_tspp1}, \cref{fig:sample_ucbpp2} and \cref{fig:sample_tspp2}, this is so because incorporating refund into price selection would deter the algorithms from lowering the price and, hence, they cannot correctly identify the optimal price, which is exactly $p_1$ in our setup.


\subsection{Numerical Experiments for Multiple Prices: \gdp~vs. \gdppp}

Recall that, in \cref{sec:general}, we developed \gdppp to deal with the setting with multiple prices. We now use numerical experiment to show that simply (naively) adopting \gdp~to this case may not yield a good performance. Thus, it is necessary to use \gdppp. 

\vspace{2mm}
\noindent\textbf{Adopting \gdp~to Multiple Prices:} We first present how to apply \gdp~to the multiple prices case. The quantities $\{u_b\,,\,t_b\}_{b=1}^B$ and $\{n_l\}_{l=1}^L$ are defined  as in \eqref{eq:def_u_t} and \eqref{eq:def_delta_n}. The algorithm begins by viewing all prices $\{p_k\}_{k=1}^K$ as plausible prices. Suppose there are $k_b$ plausible prices remaining in phase $b$. Then, in phase $b$, each price is selected equally $(t_b-t_{b-1})/k_b$ times, and the prices are selected according to their empirical mean rewards in all previous phases. For all $l=1,\,\dots\,,L\,,$ if $N_k(t)\geq n_l$ for all $k\in[K]$, \gdp~would compute all $\hat{\lambda}_k(t)$. If 
\begin{equation*}
    \max_{j\in[K]} \,\, \hat{\lambda}_j-\sqrt{\dfrac{\log(T\tilde{\Delta}_l^2)}{N_j(t)}} \,\, > \,\, \hat{\lambda}_k+\sqrt{\dfrac{\log(T\tilde{\Delta}_l^2)}{N_k(t)}}
\end{equation*}
holds for some $k\in[K]$, then the corresponding price would be removed. If some price is removed using the above test, the current phase will terminate early; otherwise, the current phase resumes.

\vspace{2mm}
\noindent\textbf{Setup:} We mainly compare the sensitivity of \gdp~and \gdppp~towards the number of prices. Here, we fix the time horizon $T=20000\,,$ and choose the number of prices as $K=2n+1\,,$ where $n=2,3,\dots,10\,.$ For each $k\in[K]\,,$ the prices are set to be $p_k=1/3+2(k-1)/(3K-3)\,.$ The demand is set to be $D_k\sim\text{Bernoulli}(1/(3p_k))$ when $k$ is odd, and $D_k\sim\text{Bernoulli}(1/(4p_k))$ when $k$ is even. That is, the expected reward of $p_k$ is $1/3$ when $k$ is odd, or $1/4$ when $k$ is even. The price protection period $M$ is set to be $T^{7/12}K^{5/12}\,.$ 

\vspace{2mm}
\noindent\textbf{Results:} The results are shown in \cref{fig:gdppp-regret-loglog3} and \cref{fig:gdppp-refund3}.

\begin{figure}[!ht]
    \centering
    
    
    \subfigure[Log-log plot of expected regret]{\includegraphics[height=6cm, width=8cm]{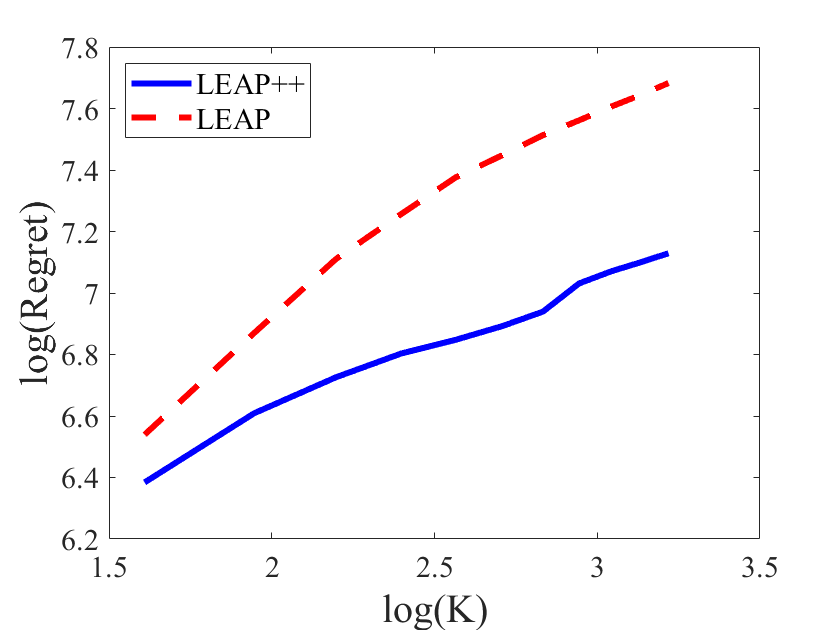}\label{fig:gdppp-regret-loglog3}}
    \subfigure[Expected refund]{\includegraphics[height=6cm, width=8cm]{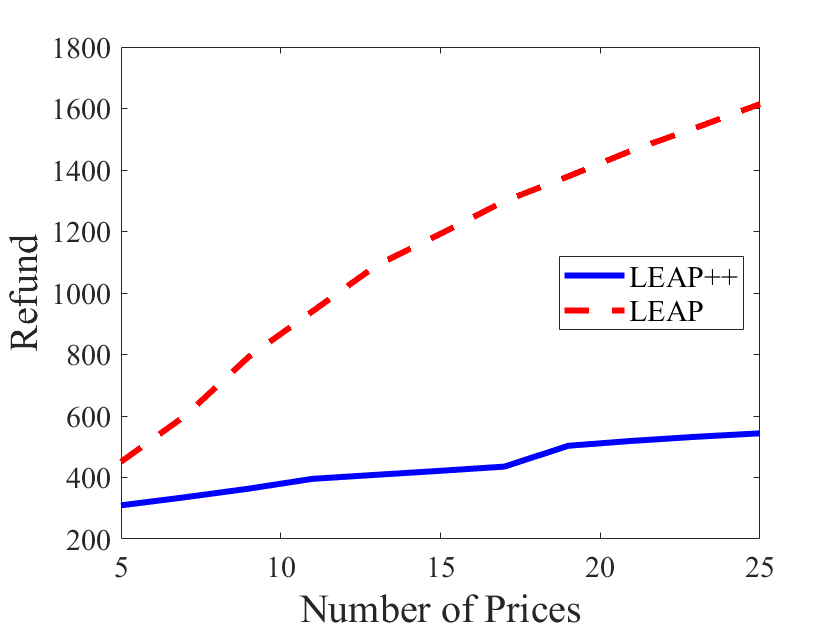}\label{fig:gdppp-refund3}}
    \caption{Numerical Experiments for $K$ prices}
\end{figure}

From \cref{fig:gdppp-regret-loglog3} one can see that the regret of \gdp~is significantly larger than the regret of \gdppp. The slope of \gdp's regret is approximately $3/4$ and the slope of \gdppp's regret is approximately $1/3$. The main source of a higher regret for \gdp~is due to the larger refund, as can shown in \cref{fig:gdppp-refund3}. One can see that, as $K$ increases, the refund under \gdp~ grows almost linearly in $K$. The reason for this is that, in each phase, \gdppp~selects the prices in ascending value order, which ensures that refunds only occur at the end of a phase. In contrast, \gdp~selects prices in the order of empirical mean reward, which could cause alternating selection for high and low prices (\eg, $p_K$ first, then followed by $p_1$, $p_{K-1}$, and $p_2$, $\ldots$), and incurs unnecessary refunds within each phase.

\subsection{Cost of Price Protection}
In this section, we show how 
price protection guarantee affects the seller's regret and revenue. 

\vspace{2mm}
\noindent\textbf{Setup:} We compare the regret and revenue when there is price protection guarantee with those without price protection guarantee. For the case with price protection we use \gdppp~while, for the case without price protection, we use UCB and Thompson Sampling. We set the prices to be $P=\{1/3\,,\,2/3\,,\,1\}\,,$ and the demand is $D_t(p_1)\equiv 1\,,$ $D_t(p_2)\sim\text{Bernoulli}(1/3)\,,$ $D_t(p_3)\sim\text{Bernoulli}(1/4)\,.$ The time horizon $T$ is chosen to range from $1000$ to $20000$ with step size $1000$. When there is a price protection, the price protection period is set to be $M=\sqrt{3T}$ and $M=T$. The results are shown in \cref{fig:price-fairness-regret} and \cref{fig:price-fairness-revenue}.
\begin{figure}[!ht]
    \centering
    \subfigure[Expected Regret]{\includegraphics[height=6cm, width=8cm]{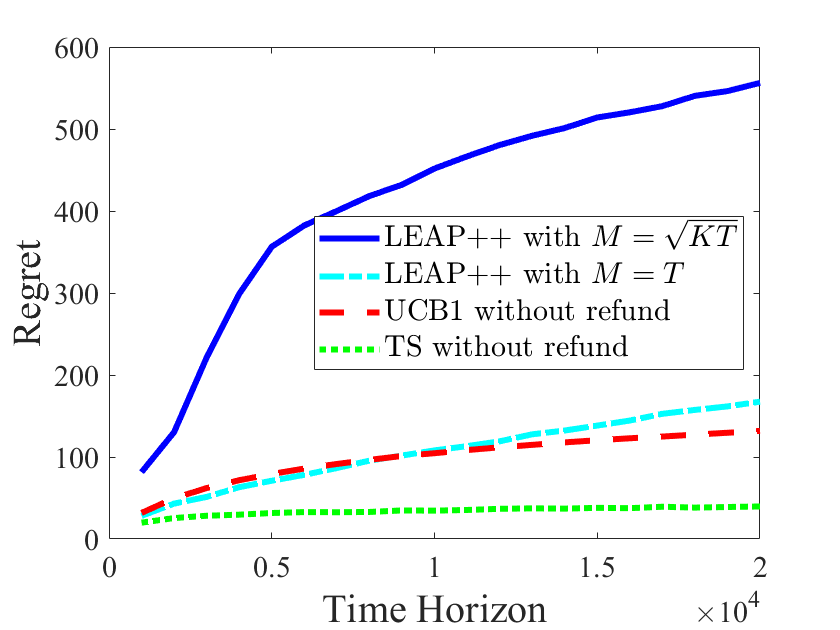}
    \label{fig:price-fairness-regret}}
    \subfigure[Expected Total Revenue]{\includegraphics[height=6cm, width=8cm]{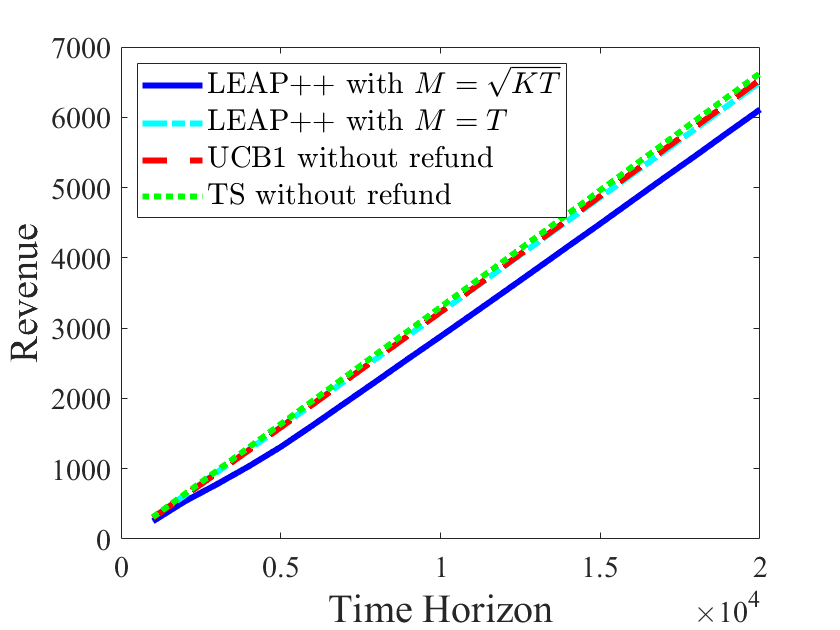}
    \label{fig:price-fairness-revenue}}
    
    \caption{\gdppp~under price protection compared to UCB and TS without price protection}
\end{figure}

\vspace{2mm}
\noindent\textbf{Results:} One can see from Figure \ref{fig:price-fairness-regret} that, with price protection guarantee and $M=\sqrt{3T}\,,$ the regret of \gdppp~is approximately 5 times of that of UCB without price protection, and approximately 10 times of that of Thompson Sampling without price protection. For the case with $M=T$, the regret \gdppp~is slightly more than that of UCB and approximately 4 times of that of Thompson Sampling. However, the four cases do not have significant difference in terms of total revenues. One can see from Figure \ref{fig:price-fairness-revenue} the total revenue under the four cases are very close to each other, which implies that the presence of price protection guarantee does not cause significant decrease in total revenue when $T$ is large.

One can also notice that the regret of \gdppp~is smaller when $M=T$ then when $T=\sqrt{KT}$. The reason of this can be explained as follows: although \gdppp~has a better upper bound when $M=\sqrt{KT}$, having at most $O(\sqrt{KT}\log\log T)$ compared to the $O(T^{2/3})$ when $M=T$, the constant term for for the upper bound with $M=\sqrt{KT}$ can be much larger than that with $M=T$; thus, empirically \gdppp~could have a lower regret when $M=T$ compared to $M=\sqrt{KT}$. This also implies that it is reasonable for the seller to have a long price protection period without having too much additional regret.

\section{Conclusion}
In this paper, we study the impact of price protection on online learning for dynamic pricing with initially unknown customer demand. We consider a setting in which a firm sells a product over a horizon of $T$ time periods under price protection guarantee. That is, any customer purchasing the product can receive a refund, the amount of which equals to the difference between the price paid and the lowest price found within the price protection period, defined as the immediate $M$ time steps after the purchase. We develop algorithms that are rate-optimal (up to logarithmic factors), and we also observe an interesting phase transition phenomena on the optimal regret as a length of protection period increases.
Finally, we also conduct numerical experiments to demonstrate the empirical benefit of our algorithms 
compared to the standard UCB and Thompson Sampling algorithms as well as their adapted variants. Our result shows the importance of taking into account the impact of price protection guarantee when designing a data-driven dynamic pricing policy.  

\ACKNOWLEDGMENT{The authors would love to sincerely thank Yining Wang for tremendously helpful discussions for the work.}


\bibliographystyle{ormsv080}
\bibliography{ref}
\newpage
\begin{appendices}{\Large \noindent\textbf{Supplementary and Proofs}}
\section{Proof of Proposition \ref{prop:ucb_regret1} and Additional Discussion}
\subsection{Proof of Proposition \ref{prop:ucb_regret1}}\label{sec:prop:ucb_regret1}
In what follows, we assume the distributions of the reward $p_kD_t(p_k)$ for both prices are supported over the interval $[a-T^{-3/2}/4,a+T^{-3/2}/4]$ for some $a\in(0,1)$. We claim that the UCB algorithm will alternately choose $p_1$ and $p_2$, \ie, for all $t=1,2,\dots,\lfloor T/2\rfloor$, either $i_{2t-1}=1\,,$ $i_{2t}=2\,,$ or $i_{2t-1}=2\,,$ $i_{2t}=1\,.$ We prove our claim by induction. By virtue of the UCB, it would select both $p_1$ and $p_2$ once at the first two time steps. Suppose for all $t=1,2,\dots,t_0$, either $i_{2t-1}=1\,,$ $i_{2t}=2\,,$ or $i_{2t-1}=2\,,$ $i_{2t}=1\,.$ Now consider $t=t_0+1\,,$ we assume w.l.o.g. that $i_{2t_0+1}=1\,.$ Then, $N_1(2t_0+2)=t_0+1$ and $N_2(2t_0+2)=t_0$. Now consider time step $2t_0+2\,:$
\begin{align*}
    &\left(\hat{\lambda}_2(2t_0+2)+\sqrt{\dfrac{\log(T)}{N_2(2t_0+2)}}\right)-\left(\hat{\lambda}_1(2t_0+2)+\sqrt{\dfrac{\log(T)}{N_1(2t_0+2)}}\right)\\
    \geq& -|\hat{\lambda}_2(2t_0+2)-\hat{\lambda}_1(2t_0+2)|+\sqrt{\log(T)}\left(\dfrac{1}{\sqrt{t_0}}-\dfrac{1}{\sqrt{t_0+1}}\right)\\
    \geq&-T^{-3/2}+\dfrac{\sqrt{\log(T)}}{\sqrt{t_0(t_0+1)}(\sqrt{t_0}+\sqrt{t_0+1})}\\
    =&-T^{-3/2}+\sqrt{\log(T)}T^{-3/2}>0.
\end{align*}
Here, the second inequality follows from the fact that $p_kD_t(p_k)$ and each $\hat{\lambda}_k(\cdot)$'s are supported on $[a-T^{-3/2}/2,a+T^{-3/2}/2]$. Therefore at time step $2t_0+2$, UCB would select $p_2$, and the induction claim thus holds for $t=t_0+1$. 

By induction we have that for all $t=1,\dots,\lfloor T/2\rfloor\,,$ either $i_{2t-1}=1\,,$ $i_{2t}=2\,,$ or $i_{2t-1}=2\,,$ $i_{t2}=1\,.$ Therefore, $p_2$ is selected at least $\lfloor T/2\rfloor$ times, and each time $p_2$ is selected, $p_1$ will be selected within 2 time steps. In other words, refund is incurred whenever $p_2$ is selected, and UCB would result $\Omega(T)$ refund/regret.

\subsection{Numerical Example for \cref{prop:ucb_regret1}}\label{sec:ucb_test}
In this test, we set the expected reward of the two prices to be the same, and then consider both the UCB \cite{ABF02} and the Thompson Sampling \cite{AgrawalG12} algorithms. Specifically, we let $p_1=1/2\,,$ $p_2=2/3\,,$ the demand under the two prices are $D_t(p_1)\sim \text{Bernoulli}(2/3)\,,$ $D_t(p_2)\sim\text{Bernoulli}(1/2)$. The time horizon $T$ varies from 500 to $10000$ with a step size $500\,.$ The length of price protection period is set to $M=\lceil \sqrt{T}\rceil\,.$ It is easy to verify that $\lambda_1=\lambda_2=1/3$, which means all the regret incurred (if there is any) would be due to refund under the price protection mechanism. 
\begin{figure}[!ht]
    \centering
    \includegraphics[height=6cm, width = 8cm]{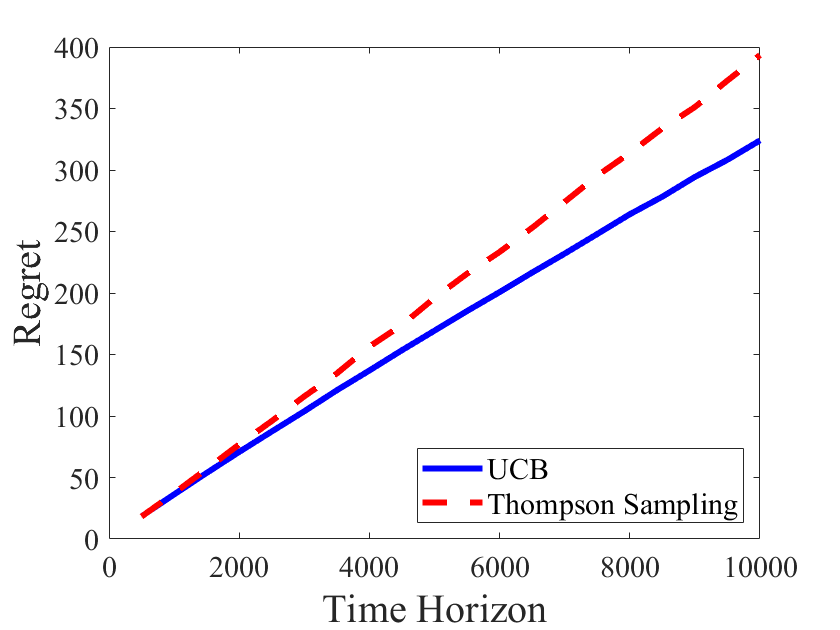}
    \caption{Expected regret of UCB and Thompson Sampling under different $T$}
    \label{fig:ucb_ts_regret_1}
\end{figure}

For each $T$, we present the averaged results over $10^4$ iterations, and report the expected regret. The results under different $T$ are shown in \cref{fig:ucb_ts_regret_1}. From \cref{fig:ucb_ts_regret_1}, one can see that the expected regret of both UCB and Thompson Sampling scale linearly as $T\,.$ This coincides with the claim in \cref{prop:ucb_regret1} that if the two prices have similar expected reward, then under the price protection mechanism, both the canonical UCB and Thompson Sampling could incur linear regret, even with relatively short price protection period.

\subsection{Additional Discussion}\label{sec:ucb_discussion}
In this section, we show that a larger class of UCB algorithms could also incur linear regret under the price protection mechanism. We first propose a general framework of the UCB algorithms.
\begin{algorithm}[htbp]
\SingleSpacedXI
\caption{General Framework of the UCB Algorithms} \label{alg:general_ucb}
\begin{algorithmic}[]
    \State \textbf{Input:} price set $P\,,$ time horizon $T\,.$
    \For {Time step $t=1,\dots,T$}
    \State Selects the price $p_i$ that maximizes $F(\hat{\lambda}_i(t),N_i(t),T,K,t)\,.$
    \EndFor
\end{algorithmic}
\end{algorithm}

Here $F(\hat{\lambda},N,T,K,t)$ is a function that is strictly increasing w.r.t. $\hat{\lambda}$ (since the upper confidence bound is generally higher with a higher empirical mean) and strictly decreasing w.r.t. $N$ (since the upper confidence bound is generally lower when a price is more explored). Note that the framework is quite general, and $F=+\infty$ whenever $N=0$. Besides the UCB in \cite{ABF02}, many other popular UCB algorithms also satisfy the framework in \cref{alg:general_ucb}. Examples include the MOSS introduced in \cite{AudibertB2009}, which selects the price in each time step according to
\begin{equation*}
i_t=\argmax_{i}\hat{\lambda}_i(t)+\sqrt{\dfrac{4}{N_i(t)}\log^+\left(\dfrac{T}{KN_i(t)}\right)},
\end{equation*}
where $\log^+(x)=\log(\max\{1,x\})\,;$ and the KL-UCB introduced in \cite{GarivierC2011} and \cite{MaillardMS2011}, which selects the price in each time step according to
\begin{equation*}
    i_t=\argmax_i\max\left\{\tilde{\lambda}\in[0,1]:d(\hat{\lambda}_i(t)\,,\,\tilde{\lambda})\leq\dfrac{f(t)}{N_i(t)}\right\},
\end{equation*}
where $d(\lambda_1\,,\,\lambda_2)=\texttt{KL}(\text{Bernoulli}(\lambda_1)||\text{Bernoulli}(\lambda_2))\,,$ and $f(t)=1+t\log^2(t)\,.$

\begin{proposition}
    Suppose there are two prices $\{p_1\,,\,p_2\}$ and the demand under both prices are deterministic and $p_1D_t(p_1)=p_2D_t(p_2)=\lambda$, and $M\geq 2$. Under this instance, any UCB algorithm satisfying the framework in Algorithm \ref{alg:general_ucb} will incur $\Omega(T)$ regret.
\end{proposition}
\begin{proof}
We claim that in this case the UCB algorithm will alternately choose $p_1$ and $p_2$, \ie, for all $t=1,2,\dots,\lfloor T/2\rfloor$, either $i_{2t-1}=1\,,$ $i_{2t}=2\,,$ or $i_{2t-1}=2\,,$ $i_{2t}=1\,.$ We prove our claim by induction. Since the UCB requires that each price is selected once in the first $K$ prices, we have that either $i_1=1\,,$ $i_2=2\,,$ or $i_1=2\,,$ $i_2=1\,.$ Suppose for all $t=1,2,\dots,t_0$, either $i_{2t-1}=1\,,$ $i_{2t}=2\,,$ or $i_{2t-1}=2\,,$ $i_{2t}=1\,.$ We assume w.l.o.g. that $i_{2t_0+1}=1$, then $N_1(2t_0+2)=t_0+1$ and $N_2(2t_0+2)=t_0$. Therefore
\begin{equation*}
F(\hat{\lambda}_2(t),N_2(t),T,2,t)-F(\hat{\lambda}_1(t),N_1(t),T,2,t)=F(\lambda,t_0,T,2,t)-F(\lambda,t_0+1,T,2,t)>0,
\end{equation*}
where the last inequality is derived by the assumption that $F$ is strictly decreasing w.r.t. $N$. Therefore $i_{2t_0+2}=2$ if $i_{2t_0+1}=1$, and the claim holds for $t=t_0+1$. By induction we have that for all $t=1,\dots,\lfloor T/2\rfloor\,,$ either $i_{2t-1}=1\,,$ $i_{2t}=2\,,$ or $i_{2t-1}=2\,,$ $i_{t2}=1\,.$ Therefore, $p_2$ is selected at least $\lfloor T/2\rfloor$ times, and each time $p_2$ is selected, $p_1$ will be selected within 2 time steps, thus refund will be incurred whenever $p_2$ is selected, and the UCB algorithm will incur $TD(p_2)(p_2-p_1)/2=\Omega(T)$ refund.
\end{proof}

\section{Proof of Theorem \ref{thm:two_price_lb}}\label{sec:thm:two_price_lb}
We first present a result which shall be used extensively throughout this proof.
\begin{proposition}\label{prop:refund}
Suppose there exists $\mu>0$ such that $D_t(p_i)>\mu$ holds almost surely for all $i=1,2$. Given a price protection period $M,$ for any time step $T_0<T,$ if price $p_2$ is selected at least $M_0$ times up to time $T_0,$ \ie, $N_2(T_0)\geq M_0,$ then if $p_1$ is selected during the time interval $[T_0+1,T],$ the refund incurred is $\mu(p_2-p_1)\min\{M_0,M\}.$
\end{proposition}
The proof of this result is provided in \cref{sec:prop:refund}. Now, depending on the value of $M,$ we distinguish three different cases.

\noindent\textbf{Case 1. $M\leq \sqrt{T}\,:$} For this case, one can use standard lower bound argument for multi-armed bandit (see, \eg, theorem 5.1 of \cite{ABFS02}) to show that the lower bound is $\Omega(\sqrt{T}).$

\noindent\textbf{Case 2. $\sqrt{T}<M< T^{2/3}\,:$} For this case, we proceed according to the following steps.
\begin{itemize}
\item \textbf{Step 1. Indistinguishable Instances:} We consider two different instances of demand parameters $(\mu_1,\mu^{(1)}_2)$ and $(\mu_1,\mu^{(2)}_2).$ We set $p_1=1/2$ and $p_2=3/4$. Here, setting the price to $p_1$ would deterministically bring in the earning $p_1\mu_1\,.$ For $p_2,$ selecting it would lead to a reward follows a Bernoulli distribution with support on $1/3$ and $2/3$ and mean $p_2\mu^{(j)}_2$ in instance $j$. Moreover, 
\begin{align}
	\nonumber&\textbf{Instance 1: } p_2\mu^{(1)}_2=p_1\mu_1-\epsilon M^{-1/2},\qquad 
	\textbf{Instance 2: } p_2\mu^{(2)}_2=p_1\mu_1+\epsilon M^{-1/2}.
\end{align}
In both instances, $\epsilon=1/8$. We also enforce that $p_2\mu^{(1)}_2=1/2$ and there are two absolute constants $c_1,c_2~(\geq1/10)$ such that $$p_2-p_1>c_1\,,\qquad D_t(p_i)>c_2\,,\quad\text{a.s.}\,,\quad\forall i\in\{1,2\,.$$
In other words, the only difference between the two instances is the choice of $\mu_2$ and the only randomness would come from setting price to $p_2.$ 

\item \textbf{Step 2. Key Event and Sample Space:} For any policy $\pi,$ to determine if $\pi$ has selects $p_2$ for sufficiently many times, we define $T_1=T-M^{3/2}/2+1$ and the event $E_1$ on price $p_2$ is selected more than $M/3$ times up to time $T_1,$ \ie,
\begin{align}
    E_1=\left\{N_2(T_1)\geq \frac{M}{3}\right\}.
\end{align}
Associated with these is the sample space/history (for $\pi$)
\begin{align}
S=\prod_{u=1}^{M/3}([T]\cup\{\infty\}\,,\ \{0,1\})\,,
\end{align}
which records the time steps when $p_2$ is offered and the realized demand. We remark that since $\pi$ is non-anticipatory and there is no randomness when $p_1$ is selected, whether $E_1$ holds or not is completely determined by its first $M/3$ times when $p_2$ is offered, \ie, $E_1$ can be interpreted as a subset of $S.$ Throughout the rest of the proof, any probability is defined over $S$ and we use $\Pr_S^{(j)}(\cdot)=\Pr_S(\cdot | \mu_2=\mu^{(j)}_2)$ to denote the probability w.r.t. instance $j.$ 

\item \textbf{Step 3. Complete the Proof:} To this end, we make the following observation:
\begin{enumerate}
    \item Suppose $E_1$ holds, then for instance 1, if $p_1$ is never chosen in the rest $T-T_1+1=M^{3/2}/2$ time steps, then the regret incurred by falsely selecting $p_2$ is at least 
    \begin{align*}
        \frac{M^{3/2}}{2}\epsilon M^{-1/2}=\frac{M}{16}\,;
    \end{align*}
    while if $p_1$ is ever selected in the rest $M^{3/2}/2$ time steps, by \cref{prop:refund}, the refund would incur $\Omega(M)$ expected regret;
    \item Suppose $E_1$ does not hold, then for instance 2, the expected regret incurred by falsely selecting $p_1$ is at least 
    \begin{align}
        \left(T_1-\frac{M}{3}\right)\epsilon M^{-1/2}\geq \frac{T}{6}\frac{1}{8M^{1/2}}\geq \frac{M}{48}=\Omega(M)\,,
    \end{align}
    where we use the fact that $M\leq T^{2/3}$ in both inequalities.
\end{enumerate}
Consequently, the expected regret of any policy $\pi$ on the two instances is at least
\begin{align}\label{eq:lb00}
    \Omega(M(\Pr_S^{(1)}(E_1)+\Pr_S^{(2)}(\neg E_1))).
\end{align}
By the Bretagnolle–Huber inequality (see, \eg, theorem 14.2 of \cite{LS18}), we have
\begin{align}
    \nonumber\Pr_S^{(1)}(E_1)+\Pr_S^{(2)}(\neg E_1)\geq&\exp(-\kl(\Pr_S^{(1)}||\Pr_S^{(2)}))\\
    \nonumber=&\exp\left(-\sum_{t=1}^{M/3}\kl(\text{Bernoulli}(p_2\mu^{(1)}_2)||\text{Bernoulli}(p_2\mu^{(2)}_2))\right)\\
    \geq&\exp\left(\frac{M}{6}\log_e\left(1-\frac{4\epsilon^2}{M}\right)\right)\geq\exp\left(-1\right)\,,\label{eq:lb01}
\end{align}
where the last step uses the fact that for $x\in[0,1/4],$ $\log_e(1-x)\geq-4\log_e(4/3)x\geq-4x\,.$

Combining \eqref{eq:lb00} and \eqref{eq:lb01}, the expected regret of $\pi$ on the two instances sum up to $\Omega(M),$ which indicates that $\pi$ has to incur $\Omega(M)$ expected regret in at least one of them.
\end{itemize}
Here, we remark that the choice of sample space in Step 2 is critical to our proof and we refer interested readers to \cref{sec:lb_remark} for a more detailed discussion.

\noindent\textbf{Case 3. $M\geq T^{2/3}\,$:} For this case, we follow most of the reasoning in Case 2, but with a different except that we consider a different pair of $(\mu_1,\mu_2).$

\begin{itemize}
\item \textbf{Step 1. Indistinguishable Instances:} We consider the following mean demand parameters
\begin{align}
	\nonumber&\textbf{Instance 1: } p_2\mu^{(3)}_2=p_1\mu_1-\epsilon T^{-1/3}, 
	\qquad\textbf{Instance 2: } p_2\mu^{(4)}_2=p_1\mu_1+\epsilon T^{-1/3}.
\end{align}
In both instances, $\epsilon=1/8$. We also enforce that $p_2\mu^{(3)}_2=1/2$ and there are two absolute constants $c_3,c_4~(\geq1/10)$ such that $$p_2-p_1>c_3\,,\qquad \mu_1\,,\ \mu^{(3)}_2\,,\ \mu^{(4)}_2\geq c_4\,.$$ 
We shall consider an environment where the underlying $(\mu_1,\mu_2)$ is equal to either $(\mu^{(j)}_1,\mu^{(j)}_2)$ with probability $1/2,$ \ie,
\begin{align}
    \Pr((\mu_1,\mu_2)=(\mu^{(3)}_1,\mu^{(3)}_2)) = \Pr((\mu_1,\mu_2)=(\mu^{(4)}_1,\mu^{(4)}_2))= \frac{1}{2}.
\end{align}

\item \textbf{Step 2. Key Event and Sample Space:} For any policy $\pi,$ to determine if $\pi$ has selects $p_2$ for sufficiently many times, we define $T_2=T/2+1$ and the event $E_2$ on price $p_2$ is selected more than $T^{2/3}$ times up to time $T_2,$ \ie,
\begin{align}
    E_2=\left\{N_2(T_2)\geq T^{2/3}\right\}.
\end{align}
Associated with these is the sample space/history (for $\pi$)
\begin{align}
S'=\prod_{u=1}^{T^{2/3}}([T]\cup\{\infty\}\,,\ \{0,1\}),
\end{align}
which records the time steps when $p_2$ is offered and the realized demand. Same as before, $E_2$ can be interpreted as a subset of $S'.$ Throughout the rest of the proof, any probability is defined over $S'$ and we use $\Pr_{S'}^{(j)}(\cdot)=\Pr_{S'}(\cdot | \mu_2=\mu^{(j)}_2)$ to denote the probability w.r.t. instance $j.$

\item \textbf{Step 3. Complete the Proof:} To this end, we make the following observation:
\begin{enumerate}
    \item Suppose $E_2$ holds, then for instance 1, if $p_1$ is never selected in the rest of the $T/2$ time steps, then the regret incurred by falsely selecting $p_2$ is at least $T^{2/3}/16;$ while if $p_1$ is ever selected in the rest of the $T/2$ time steps, then by \cref{prop:refund}, the refund would incur $\Omega(T^{2/3})$ expected regret;
    \item Suppose $E_2$ does not hold, then for instance 2, the expected regret incurred by falsely selecting $p_1$ is at least 
    \begin{align}
        \left(T_2-{T^{2/3}}\right)\epsilon T^{-1/3}\geq \frac{T}{16}\frac{1}{8T^{1/3}}=\frac{T^{2/3}}{128}
    \end{align}
    where we use the fact that $T\geq 12$ and hence, $T/2-T^{2/3}\geq T/16$ in the first step.
\end{enumerate}
Consequently, the expected regret of any policy $\pi$ on the two instances is at least
\begin{align}\label{eq:lb02}
    \Omega(T^{2/3}(\Pr_{S'}^{(3)}(E_2)+\Pr_{S'}^{(4)}(\neg E_2))).
\end{align}
By the Bretagnolle–Huber inequality (see, \eg, theorem 14.2 of \cite{LS18}), we have
\begin{align}
    \nonumber\Pr_{S'}^{(3)}(E_2)+\Pr_{S'}^{(4)}(\neg E_2)\geq&\exp(-\kl(\Pr_{S'}^{(3)}||\Pr_{S'}^{(4)}))\\
    \nonumber=&\exp\left(-\sum_{t=1}^{T^{2/3}}\kl(\text{Bernoulli}(p_2\mu^{(3)}_2)||\text{Bernoulli}(p_2\mu^{(4)}_2))\right)\\
    \geq&\exp\left(\frac{T^{2/3}}{2}\log_e\left(1-\frac{1}{16T^{2/3}}\right)\right)\geq\exp\left(-1\right)\,,\label{eq:lb03}
\end{align}
where the last step again uses the fact that for $x\in[0,1/4],$ $\log_e(1-x)\geq-4\log_e(4/3)x\geq-4x.$

Combining \eqref{eq:lb02} and \eqref{eq:lb03}, the expected regret of $\pi$ on the two instances sum up to $\Omega(T^{2/3}),$ which indicates that $\pi$ has to incur $\Omega(T^{2/3})$ expected regret in at least one of them.
\end{itemize}

\subsection{Proof of \cref{prop:refund}}\label{sec:prop:refund}
We denote all the time intervals when $p_2$ is selected (\ie, $i_t=2$) as $I_1,\ldots,I_L.$ That is, if $t<T_0$ and $i_t=2,$ then $t\in I_l$ for some $l\in[L];$ otherwise, $t\notin I_l$ for all $l\in[L].$ 

Among others, we use the partition with the smallest $L,$ which means $\max_{t\in I_l} t+1<\min_{t\in I_{l+1}} t$ (as otherwise, we could combine $I_l$ and $I_{l+1}$). W.l.o.g., we assume the cardinality of each $I_l$ is at most $\min\{M_0,M\},$ \ie, $|I_l|\leq M/3$ for all $l\in[L]$ (or we only consider the last $\min\{M_0,M\}$ time steps for each $I_l$). 

Now, we can easily see that all the offered price $p_2$ will be followed by a $p_1$ within $M$ time steps and hence incurs refund of order $\Omega((p_2-p_1)\min\{M_0,M\}).$    
  
\subsection{Comparisons to Classic Lower Bound Proof Techniques}\label{sec:lb_remark}
Compared to classic proofs for regret lower bounds (see, \eg, theorem 15.2 of \cite{LS18} or theorem 5.1 of \cite{ABFS02}), our proof critically leverages the fact that, when the demand under $p_1$ is completely deterministic, whether $p_2$ would be selected by more than $X$ number of times is completely determined by the first $X$ samples of $D_t(p_2)\,.$ Hence, we can work with a reduced sample space instead of the complete sample space (as theorem 15.2 of \cite{LS18} or theorem 5.1 of \cite{ABFS02}).

One immediate benefit of this is that when we apply the Bretagnolle–Huber inequality in \eqref{eq:lb01} (and \eqref{eq:lb03}), we are able to obtain lower bounds of $\Theta(\exp(M\log_e(1-M^{-1})))$ (and $\Theta(\exp(T^{2/3}\log_e(1-T^{-2/3})))$) instead of the much looser 
$\Theta(\exp(T\log_e(1-M^{-1})))$ (and $\Theta(\exp(T\log_e(1-T^{-2/3})))$) were we working with the complete sample space. The latter cannot hopefully give us the desired lower bound. To the best of our knowledge, this reduced sample space technique was first presented in section 2.4 of \cite{Slivkins19}.

\section{Proof of \cref{prop:two_price_ub1}}\label{sec:prop:two_price_ub1}
We first note that
\begin{align*}
    u_B=a^{2-2^{1-B}}=(e\sqrt{T})^{2-2^{1-\log_2\log(T)}}=\frac{e^2T}{(e\sqrt{T})^{2/\log(T)}}=\frac{e^2T}{e^{2\log(e\sqrt{T})/\log(T)}}\,,
\end{align*}
which indicates that $u_B\geq\frac{e^2T}{e^2}=T\,.$ Therefore, 
$t_B=\min\{T,\lceil u_B\rceil\}\geq u_B\geq T\,.$ That is to say, there are at most $B=\log_2\log(T)$ many times of selecting $p_1$ after $p_2$ (upon the beginning or the half point of each phase). Therefore, the price protection mechanism would incur at most $O(M\log\log(T))$ regret.

To proceed, we assume w.l.o.g. that $p_1$ is the optimal price, \ie, $\lambda_1\geq\lambda_2\,,$ and denote $$\Delta = \lambda_1 - \lambda_2\,,$$
as the difference in expected reward between the two prices. We consider the case where $\Delta\geq\gamma$ for some $\gamma\geq \sqrt{e/T}$ to be specified (see theorem 3.1 of \cite{AO10}) as when $\Delta<\gamma,$ the expected regret is at most $T\gamma\,.$

Let $l^*$ be the smallest $l$ such that $\tilde{\Delta}_l< \Delta/2\,,$ \ie,
\begin{align*}
    l=\argmin_{l}\left(\tilde{\Delta}_l<\frac{\Delta}{2}\right)\,.
\end{align*}
We would have $\tilde{\Delta}_{l^*}\geq \Delta/4$ (as $2\tilde{\Delta}_{l^*}=\tilde{\Delta}_{l^*-1}\geq \Delta/2$) and
\begin{align}\label{eq:property_l_star}
    \sqrt{\frac{\log(T\tilde{\Delta}^2_{l^*})}{2n_{l^*}}}=\sqrt{\frac{\log(T\tilde{\Delta}^2_{l^*})}{2\lceil2\log(T\tilde{\Delta}^2_{l^*})/\tilde{\Delta}^2_{l^*}\rceil}}\leq \sqrt{\frac{\log(T\tilde{\Delta}^2_{l^*})}{4\log(T\tilde{\Delta}^2_{l^*})/\tilde{\Delta}^2_{l^*}}}< \frac{\tilde{\Delta}_{l^*}}{2}=\frac{\Delta}{4}\,.
\end{align}
We let $t^*$ be the first time step where both $N_1(t)$ and $N_2(t)$ exceeds $n_{l^*}\,,$ \ie,
\begin{align}
    t^*=\argmin_{t} (N_1(t)\geq n_{l^*}\text{ and }N_2(t)\geq n_{l^*})\,.
\end{align}
We remark that $t^*$ can be much larger than $2n_{l^*}$ due to the exploration schedule is implemented based on $t_b$'s. To proceed, we consider two different cases, depending on whether either price is removed before the beginning of time step $t^*\,:$

\noindent\textbf{Case A. Both $p_1$ and $p_2$ are not eliminated until time step $t^*$:} The proof of this case is similar to that of \cite{AO10}, but we include it here for completeness. In this case, both $p_1$ and $p_2$ would have been selected for $n_{l^*}$ times. To capture this, we define the event 
\begin{align*}
    E = \left\{\left|\hat{\lambda}_k(t^*)-\lambda_k\right|\leq\sqrt{\frac{\log(T\tilde{\Delta}^2_{l^*})}{2N_k(t^*)}}\leq\sqrt{\frac{\log(T\tilde{\Delta}^2_{l^*})}{2n_{l^*}}}~\forall~k\in\{1\,,\ 2\}\right\}\,.
\end{align*}
By Hoeffding's inequality (see, \eg, (5.6) and (5.7) in \cite{LS18}) and the fact that $p_kD_t(p_k)\in[0\,,\,1]$ is 1/4-subGaussian, we would have
\begin{align}\label{eq:event_e_prob}
    \Pr(E)\geq 1-\frac{2}{T\tilde{\Delta}^2_{l^*}}\,.
\end{align}
Conditioned on $E\,,$ we have
\begin{align*}
    \hat{\lambda}_2(t^*)+\sqrt{\frac{\log(T\tilde{\Delta}^2_{l^*})}{2n_{l^*}}}\leq&\lambda_2+\sqrt{\frac{2\log(T\tilde{\Delta}^2_{l^*})}{n_l^*}}\\
    <&\lambda_2+\Delta-\sqrt{\frac{2\log(T\tilde{\Delta}^2_{l^*})}{n_l^*}}\\
    =&\lambda_1-\sqrt{\frac{2\log(T\tilde{\Delta}^2_{l^*})}{n_l^*}}\leq\hat{\lambda}_1(t^*)-\sqrt{\frac{\log(T\tilde{\Delta}^2_{l^*})}{2n_l^*}}\,,
\end{align*}
where the second inequality follows from \eqref{eq:property_l_star} and all the others are by $E\,.$ In other words, event $E$ would indicate that $p_2$ would be removed after phase $b^*\,;$ Otherwise, it must be that $E$ does not hold, which 
happens with probability at most $2/(T\tilde{\Delta}^2_{l^*})$ according to \eqref{eq:event_e_prob}. Consequently, the contribution to expected regret of this case (we pessimistically upper bound the per round regret as $\Delta$) is at most 
\begin{align}\label{eq:two_price_ub01}
    \nonumber&\regret(\gdp\mid p_1\,,\,p_2\text{ remain until  }t^*)\Pr(p_1\,,\,p_2\text{ remain until }t^*)\\
    \nonumber\leq&\regret(\gdp\mid p_1\,,\,p_2\text{ remain until  }t^*)\Pr(p_1\,,\,p_2\text{ not eliminated at  }(l^*)^{\text{th}}\text{ test})\\
    \leq& T\Delta\Pr(\neg E)\leq T\Delta\frac{32}{\Delta^2}\leq\frac{32}{\Delta}\,.
\end{align}
Here, we slightly overload the notation by using $\regret(\pi\mid Z)$ to denote the conditional expected regret of $\pi$ under event $Z\,.$ 

\noindent\textbf{Case B. Either $p_1$ or $p_2$ is eliminated before time step $t^*$:} In this case, we further distinguish two different scenarios (note that \gdp~cannot remove both $p_1$ and $p_2$ when there are only two prices):
\begin{itemize}
    \item \textbf{Case B1. $p_1$ is eliminated before time step $t^*$: } The proof of this case is again similar to that of \cite{AO10}. Suppose $p_1$ is eliminated before time step $t^*\,,$ then there must exists $l<l^*\,,$ such that 
    \begin{align}\label{eq:l_prop1}
        \hat{\lambda}_2(t^{(l)})-\hat{\lambda}_1(t^{(l)})>\sqrt{\frac{2\log(T\tilde{\Delta}^2_l)}{n_l}}\quad\Rightarrow\quad\hat{\lambda}_2(t^{(l)})-\sqrt{\frac{\log(T\tilde{\Delta}^2_{l})}{2n_{l}}}>\hat{\lambda}_1(t^{(l)})+\sqrt{\frac{\log(T\tilde{\Delta}^2_{l})}{2n_{l}}}\,,
    \end{align}
    where $t^{(l)}$ is the first time step when both $N_1(t)$ and $N_2(t)$ are at least $n_l\,.$ To this end, we define the event,
    \begin{align*}
    E_l = \left\{\left|\hat{\lambda}_k(t^{(l)})-\lambda_k\right|\leq\sqrt{\frac{\log(T\tilde{\Delta}^2_{l})}{2N_k(t^{(l)})}}\leq\sqrt{\frac{\log(T\tilde{\Delta}^2_{l})}{2n_{l}}}~\forall~k\in\{1\,,\ 2\}\right\}\,.
    \end{align*}
    Again, from the Hoeffding's inequality, $\Pr(E_l)\geq1-2/(T\tilde{\Delta}^2_l)\,.$ Conditioned on $E_l\,,$
    \begin{align}\label{eq:l_prop2}
        \hat{\lambda}_1(t^{(l)})+\sqrt{\frac{\log(T\tilde{\Delta}^2_{l})}{2n_{l}}}\geq\lambda_1\,,\qquad \lambda_2\geq \hat{\lambda}_2(t^{(l)})-\sqrt{\frac{\log(T\tilde{\Delta}^2_{l})}{2n_{l}}}\,.
    \end{align}
    Combining \eqref{eq:l_prop1} and \eqref{eq:l_prop2}, this would lead to the contradiction that $\lambda_2>\lambda_1\,.$ In other words, if $p_1$ is eliminated after the $t^{(l)}\,,$ it must be that $E_l$ does not hold. 
    
    Summing this over all possible $l<l^*\,,$ the expected regret is at most (again, we pessimistically upper bound the per round regret as $\Delta$)
    \begin{align}\label{eq:two_price_ub02}
      \nonumber&\regret(\gdp\mid p_1\text{ eliminated before }t^*)\Pr(p_1\text{ eliminated before }t^*) \\
      \leq& T\Delta\sum_{l=1}^{l^*-1}\frac{2}{T\tilde{\Delta}^2_l}=\sum_{l=1}^{l^*-1}\frac{2\Delta}{\tilde{\Delta}^2_l}\leq\sum_{l=1}^{l^*-1}\frac{\tilde{\Delta}_{l}}{\tilde{\Delta}^2_l}=\sum_{l=1}^{l^*-1}\frac{1}{\tilde{\Delta}_l}\leq\frac{1}{\tilde{\Delta}_{l^*}}\leq\frac{2}{\Delta}\,,
    \end{align}
    where the second and last inequalities utilizes the definition of $l^*$ and the associated facts that $\tilde{\Delta}_1>\ldots>\tilde{\Delta}_{l^*-1}\geq \Delta/2$ and $\tilde{\Delta}_{l^*}<\Delta/2\leq 4\tilde{\Delta}_{l^*}\,.$

    \item \textbf{Case B2. $p_2$ is eliminated before time step $t^*$: } This case can be drastically different than \cite{AO10} since $t^*$ can be much larger than $2n_{l^*}\,.$ To deal with this, we are going to leverage our design that the price with high average reward in the previous phase would be selected first. 
    
    To formalize this, we let $b^*$ be the first phase such that $t_b\geq 2n_{l^*}$ \ie, $b^*=\argmin_{b}(t_b\geq 2n_{l^*})\,.$ Then, we would have $t^*\in[t_{b^*-1}+1\,,\,t_{b^*}]\,.$ Note that 
    \begin{align}\label{eq:two_price_ub03}
        \nonumber&\regret(\gdp\mid p_2\text{ eliminated before }t^*)\Pr(p_2\text{ eliminated before }t^*)\\
        \nonumber\leq&\Delta\E[N_2(t_{b^*-1}+1)]\\
        &+\Delta\E\left[N_2(t_{b^*-1}+1:t_{b^*}+1)\mid p_2\text{ eliminated before }t^*\right]\Pr(p_2\text{ eliminated before }t^*)\,,
    \end{align}
    where we overload the notation a bit by defining $N_2(t_1:t_2)=\sum_{t=t_1}^{t_2-1}\bm{1}[i_t=2]\,.$
    
    For the first term in \eqref{eq:two_price_ub03}, suppose $p_2$ is eliminated before phase $b^*\,,$ we must have that until the end of phase $b^*-1\,,$ \ie, $t=t_{b^*-1}\,,$ $p_2$ is selected by at most $2n_{l^*}$ times, and would incur an expected regret of 
    \begin{align}\label{eq:two_price_ub04}
        \Delta 2n_{l^*}=\Delta\left\lceil\frac{2\log(T\tilde{\Delta}^2_{l^*})}{\tilde{\Delta}^2_{l^*}}\right\rceil\leq2\Delta\frac{8\log(T\tilde{\Delta}^2_{l^*})}{\tilde{\Delta}^2_{l^*}}\leq \frac{128\log(T{\Delta}^2/4)}{\Delta}\leq\frac{128\log(T{\Delta}^2)}{\Delta}\,.
    \end{align} 

    For the second term in \eqref{eq:two_price_ub03}, we note that during phase $b^*\,,$ if $p_2$ is selected first in phase $b^*\,,$ it would be selected for at most $(t_{b^*}-t_{b^*-1})/2\leq t_{b^*}/2$ times; Otherwise, it would be selected for at most $n_{l^*}-N_2(t_{b^*-1}+1)\leq n_{l^*}$ times, \ie,
    \begin{align}\label{eq:b_star_decomp}
        \nonumber&\E\left[N_2(t_{b^*-1}+1:t_{b^*}+1)\mid p_2\text{ eliminated before }t^*\right]\Pr(p_2\text{ eliminated before }t^*)\\
        \nonumber\leq&\frac{t_{b^*}}{2}\Pr(p_2\text{ selected first in phase }b^*\mid p_2\text{ eliminated before }t^*)\Pr(p_2\text{ eliminated before }t^*)\\
        \nonumber&+n_{l^*}\Pr(p_2\text{ not selected first in phase }b^*\mid p_2\text{ eliminated before }t^*)\Pr(p_2\text{ eliminated before }t^*)\\
        \leq&\frac{t_{b^*}}{2}\Pr(p_2\text{ selected first in phase }b^*\,, \,p_2\text{ eliminated before }t^*)+n_{l^*}\,.
    \end{align}
    For the first term, recall that \gdp~would first select the price with higher empirical mean reward computed with data up to phase $b^*\,,$ \ie, $\argmax_{k\in\{1\,,\,2\}}~\hat{\lambda}_k(t_{b^*-1}+1)\,,$ we have
    \begin{align}\label{eq:b_star_decomp2}
        \nonumber&\Pr(p_2\text{ selected first in phase }b^*\,,\, p_2\text{ eliminated before }t^*)\\
        \nonumber=&\Pr(\hat{\lambda}_2(t_{b^*-1})\geq\hat{\lambda}_2(t_{b^*-1})\,,\, p_2\text{ eliminated during }[t_{b^*-1}+1\,,\,t^*])\\
        \nonumber\leq& \Pr(\hat{\lambda}_2(t_{b^*-1})\geq\hat{\lambda}_2(t_{b^*-1}))\\
        \leq&\Pr\left(\exists k\in\{1\,,\, 2\}~\left|\hat{\lambda}_k(t_{b^*-1}+1) - \lambda_k\right|\geq\frac{\Delta}{2}\right)\,,
 \end{align}
where the second inequality uses the fact that $\hat{\lambda}_2(t_{b^*-1}+1)\geq \hat{\lambda}_1(t_{b^*-1}+1)$ implies that $|\hat{\lambda}_k(t_{b^*-1}+1)-\lambda_k|\geq\Delta/2$ has to hold for at least one $k\in\{1\,,\, 2\}\,.$

Since each $p_kD_t(p_k)\in[0\,,\,1]$ and is thus $1/4$-subGaussian, and note that each price has accrued at $t_{b^*-1}/2$ pieces of historical data by the end of phase $b^*-1\,,$ we have that $\hat{\lambda}_k(t_{b^*-1}+1) - \lambda_k$ is $1/(2t_{b^*-1})$-subGaussian (see, \eg, corollary 1.7 of \cite{RH18}). From a standard deviation bound (see, \eg, Lemma 1.4 of \cite{RH18}) that for every $k\in\{1\,,\, 2\}$
\begin{align*}
    \E\left[\left|\hat{\lambda}_k(t_{b^*-1}+1) - \lambda_k\right|\right]\leq\sqrt{\frac{\pi}{t_{b^*-1}}}\,.
\end{align*}
Hence, by the Markov's inequality and a union bound, we have
\begin{align*}
\Pr\left(\exists k\in\{1\,,\, 2\}~\left|\hat{\lambda}_k(t_{b^*-1}+1) - \lambda_k\right|\geq\frac{\Delta}{2}\right)\leq& \frac{2\E\left[\left|\hat{\lambda}_k(t_{b^*-1}+1) - \lambda_k\right|\right]}{\Delta/2}\leq\frac{4\sqrt{\pi}}{\Delta\sqrt{t_{b^*-1}}}
\end{align*}
Putting this back together with \eqref{eq:b_star_decomp} and \eqref{eq:b_star_decomp2}, we have
\begin{align*}
    \nonumber&\E\left[N_2(t_{b^*-1}+1:t_{b^*}+1)\mid p_2\text{ eliminated before }t^*\right]\Pr(p_2\text{ eliminated before }t^*)\\
    \leq&\frac{2t_{b^*}\sqrt{\pi}}{\Delta\sqrt{t_{b^*-1}}}+n_{l^*}\leq\frac{a\sqrt{\pi}}{\Delta}+n_{l^*}=\frac{e\sqrt{\pi T}}{\Delta}+n_{l^*}\,,
\end{align*}
where we have use the definition of $t_b\,,$ \ie, \eqref{eq:def_u_t}), that $t_b^*=\lceil u_{b^*}\rceil =\lceil a\sqrt{u_{b^*-1}}\rceil \leq 2a\sqrt{t_{b^*}-1}$ and $a=e\sqrt{T}\,.$ Combining this with \eqref{eq:two_price_ub04}, we have that \eqref{eq:two_price_ub03} can be upper bounded as
\begin{align}\label{eq:two_price_ub05}
    \nonumber&\regret(\gdp\mid p_2\text{ eliminated before }t^*)\Pr(p_2\text{ eliminated before }t^*)\\
    \leq& e\sqrt{\pi T}+ \frac{192\log(T\Delta^2)}{\Delta}\,.
\end{align}
\end{itemize}

Summarizing \eqref{eq:two_price_ub01}, \eqref{eq:two_price_ub02},  \eqref{eq:two_price_ub05}, and the case when $\Delta\leq \gamma\,,$ we have that, the expected regret (excluding refund from price protection) of \gdp~is of order
\begin{align}
    O\left(T\Delta\bm{1}[\Delta\leq\gamma]+\left(\sqrt{T}+\frac{1+\log(T\Delta^2)}{\Delta}\right)\bm{1}[\Delta\geq\gamma]\right)\,.
\end{align}
By setting $\gamma=\Theta(1/\sqrt{T})\,,$ the expected regret (excluding refund from price protection) is at most $O(\sqrt{T})\,.$ The statement follows by further incorporating the expected regret from the refund.

\section{Proof of \cref{prop:two_price_ub2}}\label{sec:prop:two_price_ub2}

For the first $2N$ time periods, both prices are selected $N$ times. If $p_1$ is selected starting from time step $2N+1\,,$ the expected regret coming from the refund mechanism is at most $N_2(2N+1) = N=O(T^{2/3})\,.$

Now, suppose w.l.o.g. that price $p_1$ is the optimal price, and define
\begin{align}
    \Delta = \lambda_1 - \lambda_2\,.
\end{align}
The regret is thus $\Delta\E[N_2(T+1)]\,.$ 
Note that
\begin{align}\label{eq:n2_expect}
    \nonumber\E[N_2(T+1)] =& N+(T-2N)\Pr\left(\hat{\lambda}_2(2N+1)\geq \hat{\lambda}_1(2N+1)\right)\\
    \leq& N+T\Pr\left(\hat{\lambda}_2(2N+1)\geq \hat{\lambda}_1(2N+1)\right)\\
    \nonumber\leq& N+T\Pr\left(\exists k\in\{1\,,\, 2\}~\left|\hat{\lambda}_k(2N+1) - \lambda_k\right|\geq\frac{\Delta}{2}\right)\,,
\end{align}
where the second inequality uses the fact that $\hat{\lambda}_2(2N+1)\geq \hat{\lambda}_1(2N+1)$ implies that $|\hat{\lambda}_k(2N+1)-\lambda_k|\geq\Delta/2$ has to hold for at least one $k\in\{1\,,\, 2\}\,,$ we only need to upper bound the last term in \eqref{eq:n2_expect}.

Since each $p_kD_t(p_k)\in[0\,,\, 1]$ and is thus $1/4$-subGaussian, we have that $\hat{\lambda}_k(2N+1) - \lambda_k$ is $1/(4N)$-subGaussian (see, \eg, corollary 1.7 of \cite{RH18}). From a standard deviation bound (see, \eg, Lemma 1.4 of \cite{RH18}) that for every $k\in\{1\,,\ 2\}$
\begin{align*}
    \E\left[\left|\hat{\lambda}_k(2N+1) - \lambda_k\right|\right]\leq\frac{\sqrt{2\pi}}{2\sqrt{N}}\,.
\end{align*}
Hence, by the Markov's inequality and a union bound, we have
\begin{align*}
    \Pr\left(\exists k\in\{1\,,\, 2\}~\left|\hat{\lambda}_k(2N+1) - \lambda_k\right|\geq\frac{\Delta}{2}\right)\leq \frac{2\E\left[\left|\hat{\lambda}_k(2N+1) - \lambda_k\right|\right]}{\Delta/2}\leq\frac{2\sqrt{2\pi}}{\Delta\sqrt{N}}
\end{align*}
Putting this back to \eqref{eq:n2_expect}, we have
\begin{align*}
    \nonumber\E[N_2(T+1)]\leq N+\frac{2\sqrt{2\pi}T}{\Delta\sqrt{N}}\,.
\end{align*}
Consequently, the expected regret (excluding refund) is upper bounded as $$\Delta\E[N_2(T+1)]\leq N\Delta+\frac{2\sqrt{2\pi}T}{\sqrt{N}}=O(T^{2/3})\,.$$
The statement holds by combining this with the expected regret from refund.

\section{Proof of \cref{thm:k_price_lb}}\label{sec:thm:k_price_lb}
To facilitate our discussion, we further define $N_{-k}(t)$ be the number of times that price $p_k$ is not selected before time period $t$ (we notice that this is also equal to $t-1-N_k(t)$), \ie,
\begin{equation*}
	N_{-k}(t)=\sum_{s=1}^{t-1}1[i_s\neq k]~(=t-1-N_k(t))\,.
\end{equation*}

Similar to the proof of \cref{thm:two_price_lb}, we first establish a result regarding the refund mechanism, which shall be used extensively throughout the rest of the proof.
\begin{proposition}\label{prop:k_refund}
	Suppose $D_t(p_k)\geq \mu$ holds almost surely. Given a price protection period $M$, for any time step $T_0<T$, if the prices other than $p_1$ are selected at least $M_0$ times, \ie, $N_{-1}(T_0)\geq M_0$, and $p_1$ is selected during the time interval $[T_0+1,T]$, then the refund incurred is at least $(p_2-p_1)\min_{i\in[K]}\mu_i\min\{M,M_0\}$.
\end{proposition}
The proof of this result is provided in \cref{sec:prop:k_refund}. Depending on the value of $M$, we distinguish three different cases.
		
\noindent\textbf{Case 1. $M\leq \sqrt{KT}\,:$} For this case, one can use the standard lower bound argument for the multi-armed bandit (see, \eg, theorem 5.1 of \cite{ABFS02}) to show that the lower bound is $\Omega(\sqrt{kT})$.
		
\noindent\textbf{Case 2. $\sqrt{KT}< M< K^{1/3}T^{2/3}\,:$} For this case, we proceed according to the following steps.
\begin{itemize}
	\item \textbf{Step 1. Indistinguishable Instances:} We first construct the following baseline instance, \ie, instance 1. The demand parameters in instance 1 satisfy:
	\begin{align*}
		\textbf{Instance 1: } p_1\mu_1-\epsilon K^{1/2}M^{-1/2}=p_2\mu_2^{(1)}=p_3\mu_3^{(1)}=\dots=p_K\mu_K^{(1)}=\dfrac{1}{2}\,.
	\end{align*}
	Then we construct other instances based on the instance 1. For each $i\in [K]\setminus\{1\}$, we keep $\mu_1$ unchanged and set the demand parameters of instance $i$ as follows:
	\begin{align*}				\textbf{Instance $i$: }\mu_j^{(i)}=\mu_j^{(1)}\ \forall j\neq i\,,\ p_i\mu_i^{(i)}=p_1\mu_1^{(i)}+\epsilon K^{1/2}M^{-1/2}=p_i\mu_i^{(1)}+2\epsilon K^{1/2}M^{-1/2}\,.
	\end{align*}
	For instance $i$, setting the price to be $p_1$ will deterministically bring a payment of $p_1\mu_1$, while selecting $p_j$ for $j\neq 1$ will lead to a random instantaneous payment following a Bernoulli distribution with support on $1/3$ and $2/3$ and mean $p_j\mu_j^{(i)}$ (we require $p_2\geq 2/3$ so that reward $2/3$ is attainable with some demand in $[0,1]$).
			
	In all of the instances, we choose $\epsilon=1/10$. We also enforce that there exists two constants $c_1,c_2\geq 1/10$ satisfying
	\begin{equation*}
		p_2-p_1\geq c_1\,,\qquad D_t(p_k)\geq c_2,\ \text{a.s.}\, \forall k\in[K].
	\end{equation*}
	\item \textbf{Step 2. Key Event and Sample Space:} For any policy $\pi,$ to determine if $\pi$ has selected each price $p_k\neq p_1$ for sufficiently many times, we define $T_1=T+1-K^{-1/2}M^{3/2}/2$ and the event $E_1$ on every price $p_k\neq p_1$ is selected for at least $M/6$ times in total before time step $T_1$, \ie,
	\begin{equation*}
		E_1=\left\{N_{-1}(T_1)\geq\dfrac{M}{6}\right\}.
	\end{equation*} 
	Note that the premise of this case imposes that $M<K^{1/3}T^{2/3}$, we thus have that $K^{-1/2}M^{3/2}<T$, therefore $T_1>T/2>M/6>0$. Our definitions of $T_1$ and $E_1$ are valid.
			
	Associated with these is a sequence of sample space/history (for $\pi$):
	\begin{equation*}	   	S=\prod_{u=1}^{T}([K]\cup\{\infty\},[0,1]\cup\{\infty\})\,,
	\end{equation*}
	which records the price selected $i_u$ and the realized payment $R_u$ (does not include refund) in each time period. We further require that if for some time step $s\in[T]$, $N_{-1}(s)\geq \lfloor M/6\rfloor$, then $i_u=R_u=\infty$ for all $u>s$. Therefore, the sample space only contains information before the $p_2\,,\ldots\,,p_K$ are selected less than $M/6$ times in total.
			
	We remark that since $\pi$ is non-anticipatory and there is no randomness when $p_1$ is selected, whether $E_1$ holds or not is completely determined by the first $M/6$ times when the prices other than $p_1$ are selected, i.e., $E_1$ can be interpreted as a subset of $S$. Throughout the rest of the proof, any probability or expectation is defined on $S$ and we use $\text{Pr}^{(j)}_S(\cdot)$ and $\mathbb{E}_S^{(j)}(\cdot)$ to denote the probability and expectation w.r.t. instance $j$.
	
	\item \textbf{Step 3. Complete the Proof:} 
	We first make two observations
	\begin{enumerate}
	    \item Suppose $E_1$ holds, then for instance 1, if $p_1$ is never chosen in the last $T-T_1=K^{-1/2}M^{3/2}/2$ time steps, then the regret incurred by falsely selecting $p_k\neq p_1$ would be
	    $$\frac{K^{-1/2}M^{3/2}}{2}\epsilon K^{1/2}M^{-1/2}=\Omega(M)\,;$$ Otherwise, \cref{prop:k_refund} indicates that the refund incurred will be at least $c_1c_2M/6=\Omega(M)\,;$
	    \item Suppose $E_1$ does not hold, then for any instance $k(~\neq 1)\,,$ we have that the regret incurred by falsely selecting $p_1$ is at least
	    $$\left(T_1-\frac{M}{6}\right)\epsilon K^{1/2}M^{-1/2}\geq\left(\frac{T}{2}-\frac{M}{6}\right)\epsilon K^{1/2}M^{-1/2}\geq\frac{TK^{1/2}}{30M^{1/2}}=\Omega(M)\,,$$
	    where we use multiple times of the premise that $K^{1/3}T^{2/3}\geq M$ and hence $T\geq K^{-1/2}M^{3/2}\,.$
	\end{enumerate}
	Combining the above, we know that for any $k\neq 1,$ the regret of $\pi$ on instance 1 and instance $k$ is at least
	\begin{align}
	    \Omega(M(\Pr^{(1)}_S(E_1)+\Pr^{(k)}_S(\neg E_1)))\,.\label{eq:lb04}
	\end{align}
	
	We then try to show that $\Pr^{(1)}_S(E_1)+\Pr^{(k)}_S(\neg E_1)=\Omega(1)$ for at least one $k\in[K]\setminus\{1\}.$ By the Bretagnolle-Huber inequality (see, \eg, theorem 14.2 in \cite{LS18}), we have that for any $k\in[K]\setminus\{1\}\,,$
	\begin{equation}			\text{Pr}_S^{(1)}(E_1)+\text{Pr}_S^{(k)}(\neg E_1)\geq \exp\left(-\kl(\text{Pr}_S^{(1)}||\text{Pr}_S^{(k)})\right)\,.\label{eq:lb05}
	\end{equation}
	Defining $Q_k(r)$ as the number of times that $p_k$ is selected among the first $r$ times that prices $p_2,\ldots,p_K$ are selected, \ie,
	\begin{align}
		Q_k(r)=\sum_{s=1}^T\bm{1}[i_s=k\,,\ N_{-1}(s+1)\leq r]\,,\label{eq:def_k}
	\end{align}
	we can also show the following proposition by choosing $r=M/6\,:$
	\begin{proposition}\label{prop:chain_rule}
	    For any $k\in[K]\setminus \{1\}$, we have
		\begin{align*}
		\kl(\text{Pr}^{(1)}_S||\text{Pr}^{(k)}_S)=\E^{(1)}_S\left[Q_k\left(\dfrac{M}{6}\right)\right]\kl\left(\text{Bernoulli}(p_k\mu_k^{(1)})||\text{Bernoulli}(p_k\mu_k^{(k)})\right)\,.
		\end{align*}
	\end{proposition}
	The proof of this proposition is provided in \cref{sec:prop:chain_rule}. To apply this result, we note that for any possible history, we have
	\begin{align*}
		\sum_{k=2}^K Q_k\left(\dfrac{M}{6}\right)=\sum_{s=1}^T \bm{1}\left[i_s\neq 1,\ N_{-1}(s+1)\leq \dfrac{M}{6}\right]\leq \dfrac{M}{6}\,.
	\end{align*}
	Taking the expectation on instance 1, we further have
	\begin{equation*}
		\sum_{k=2}^K \E_S^{(1)}\left[Q_k\left(\dfrac{M}{6}\right)\right]\leq \dfrac{M}{6}\,.
	\end{equation*}
	To this end, we define $i^*$ as the least selected price among $p_2\,,\ldots\,,p_K$ in the first $M/6$ time steps when $p_1$ is not selected, \ie, $i^*=\argmin_{k\in[K]\setminus\{1\}}\E_S^{(1)}[Q_k(M/6)]\,,$ then, 
	\begin{equation*}			\E_S^{(1)}\left[Q_{i^*}\left(\dfrac{M}{6}\right)\right]\leq \dfrac{1}{K-1}\sum_{k=2}^K \E_S^{(1)}\left[Q_k\left(\dfrac{M}{6}\right)\right]\leq \frac{M}{6(K-1)}\leq \frac{M}{3K}.
	\end{equation*}
	Therefore, following \eqref{eq:lb05} and \cref{prop:chain_rule}, we have
	\begin{align}
	    \nonumber\Pr_S^{(1)}(E_1)+\Pr_S^{(i^*)}\geq&\exp\left(-\E^{(1)}_S\left[Q_k\left(\dfrac{M}{6}\right)\right]\kl\left(\text{Bernoulli}(p_{i^*}\mu_{i^*}^{(1)})||\text{Bernoulli}(p_{i^*}\mu_{i^*}^{(i^*)})\right)\right)\\
	    \label{eq:lb06}\geq&\exp\left(-\frac{M}{6K}\log_e\left(1-\frac{16\epsilon^2K}{M}\right)\right)\geq\exp(-1)\,,
	\end{align}
	where the last step uses the fact that for $x\in[0,1/4],$ $\log_e(1-x)\geq-4\log_e(4/3)x\geq-4x$ and $M\geq \sqrt{KT}\geq K\,.$
	
	Combining \eqref{eq:lb04} and \eqref{eq:lb06}, the expected regret of $\pi$ on the instance 1 and instance $i^*$ sum up to $\Omega(M),$ which indicates that $\pi$ has to incur $\Omega(M)$ expected regret in at least one of them.
\end{itemize}
	
\noindent\textbf{Case 3. $M>K^{1/3}T^{2/3}\,:$} For this case, we proceed according to the following steps.
\begin{itemize}
	\item \textbf{Step 1. Indistinguishable Instances:} We first construct the following baseline instance (instance 1). The demand parameters in instance 1 are as follows:
	\begin{equation*}
		\textbf{Instance 1: }p_1\mu^{(1)}_1-\epsilon K^{1/3}T^{-1/3}=p_2\mu_2^{(1)}=p_3\mu_3^{(1)}=\dots=p_K\mu_K^{(1)}=\dfrac{1}{2}.
	\end{equation*}
		
	Then we construct other instances based on the baseline instance. For each $i\in[K]\setminus\{1\}\,,$ we set the demand parameters of instance $i$ as follows:
	\begin{equation*}
		\textbf{Instance $i$: }\mu_j^{(i)}=\mu_j^{(1)}\ \forall j\neq i\,,\ p_i\mu_i^{(i)}=p_1\mu_1^{(i)}+\epsilon K^{1/3}M^{-1/3}=p_i\mu_i^{(1)}+2\epsilon K^{1/3}T^{-1/3}\,.
	\end{equation*}
	For instance $i$, setting the price to be $p_1$ will deterministically bring a revenue of $p_1\mu^{(i)}_1$, while selecting $p_j$ for $j\neq 1$ will lead to a reward following a Bernoulli distribution with mean $p_j\mu_j^{(i)}$.
			
	In all of the instances, we choose $\epsilon=1/10$. We also enforce that there exists two constants $c_1,c_2\geq 1/10$ satisfying
	\begin{equation*}
		p_2-p_1\geq c_1;\ \mu_1,\mu_i^{(j)}\geq c_2,\ \forall i,j\in[k].
	\end{equation*}
	\item \textbf{Step 2. Key Event and Sample Space:} For any policy $\pi,$ to determine if $\pi$ has selected each price $p_k\neq p_1$ for sufficiently many times, we define $T_1=T/2+1$ and the event $E_1$ on every price $p_k\neq p_1$ is selected for at least $K^{1/3}T^{2/3}/6$ times in total before time period $T_1$, \ie,
	\begin{equation*}
		E_1=\left\{N_{-1}(T_1)\geq\dfrac{K^{1/3}T^{2/3}}{6}\right\}.
	\end{equation*} 
	Note that the premise of this case imposes that $K<T$, we thus have that $K^{1/3}T^{2/3}<T$, therefore $T_1>T/2>K^{1/3}T^{2/3}/6>0\,.$ Our definitions of $T_1$ and $E_1$ are valid.
			
	Associated with these is a sequence of sample space/history (for $\pi$):
	\begin{equation*}	   	S=\prod_{u=1}^{T}([K]\cup\{\infty\},[0,1]\cup\{\infty\})\,,
	\end{equation*}
	which records the price selected $i_u$ and the realized payment $R_u$ (does not include refund) in each time period. We further require that if for some time step $s\in[T]$, $N_{-1}(s)\geq \lfloor K^{1/3}T^{2/3}/6\rfloor$, then $i_u=R_u=\infty$ for all $u>s$. Therefore, the sample space only contains information before the $p_2\,,\ldots\,,p_K$ are selected less than $K^{1/3}T^{2/3}/6$ times in total.
			
	We remark that since $\pi$ is non-anticipatory and there is no randomness when $p_1$ is selected, whether $E_1$ holds or not is completely determined by the first $K^{1/3}T^{2/3}/6$ times when the prices other than $p_1$ are selected, \ie, $E_1$ can be interpreted as a subset of $S$. Throughout the rest of the proof, any probability or expectation is defined on $S$ and we use $\text{Pr}^{(j)}_S(\cdot)$ and $\mathbb{E}_S^{(j)}(\cdot)$ to denote the probability and expectation w.r.t. instance $j$.
			
	\item \textbf{Step 3. Complete the Proof:}
	We first make two observations
	\begin{enumerate}
	    \item Suppose $E_1$ holds, then for instance 1, if $p_1$ is never chosen in the last $T-T_1=T/2$ time steps, then the regret incurred by falsely selecting $p_k\neq p_1$ would be
	    $$\frac{T}{2}\epsilon K^{1/3}T^{-1/3}=\Omega(K^{1/3}T^{2/3})\,;$$ Otherwise, \cref{prop:k_refund} indicates that the refund incurred will be at least $c_1c_2K^{1/3}T^{2/3}/6=\Omega(K^{1/3}T^{2/3})\,;$
	    \item Suppose $E_1$ does not hold, then for any instance $k(~\neq 1)\,,$ we have that the regret incurred by falsely selecting $p_1$ is at least
	    $$\left(T_1-\frac{K^{1/3}T^{2/3}}{6}\right)\epsilon K^{1/3}T^{-1/3}\geq\left(\frac{T}{2}-\frac{T}{6}\right)\epsilon K^{1/3}T^{-1/3}\geq\frac{K^{1/3}T}{30T^{1/3}}=\Omega(K^{1/3}T^{2/3})\,,$$
	    where we use multiple times of the premise that $K^{1/3}T^{2/3}\leq T\,.$
	\end{enumerate}
	Combining the above, we know that for any $k\neq 1,$ the regret of $\pi$ on instance 1 and instance $k$ is at least
	\begin{align}
	    \Omega(M(\Pr^{(1)}_S(E_1)+\Pr^{(k)}_S(\neg E_1)))\,.\label{eq:lb07}
	\end{align}
	
	We then try to show that $\Pr^{(1)}_S(E_1)+\Pr^{(k)}_S(\neg E_1)=\Omega(1)$ for at least one $k\in[K]\setminus\{1\}.$ By the Bretagnolle-Huber inequality (see, \eg, theorem 14.2 in \cite{LS18}), we have that for any $k\in[K]\setminus\{1\}\,,$
	\begin{equation}			\Pr_S^{(1)}(E_1)+\Pr_S^{(k)}(\neg E_1)\geq \exp\left(-\kl(\Pr_S^{(1)}||\Pr_S^{(k)})\right)\,.\label{eq:lb08}
	\end{equation}
	We still use the notation $Q_k(r)$ defined in \eqref{eq:def_k}. Similar to Proposition \ref{prop:chain_rule}, we can also show the following proposition by choosing $r=K^{1/3}T^{2/3}/6\,:$
	\begin{proposition}
	    For any $k\in[K]\setminus \{1\}$, we have
		\begin{align*}
		\kl(\Pr^{(1)}_S||\Pr^{(k)}_S)=\E^{(1)}_S\left[Q_k\left(\dfrac{K^{1/3}T^{2/3}}{6}\right)\right]\kl\left(\text{Bernoulli}(p_k\mu_k^{(1)})||\text{Bernoulli}(p_k\mu_k^{(k)})\right)\,.
		\end{align*}
	\end{proposition}
	The proof of this proposition is very similar to that of Proposition \ref{prop:chain_rule} and is thus omitted. To apply this result, we note that for any possible history, we have
	\begin{align*}
		\sum_{k=2}^K Q_k\left(\dfrac{K^{1/3}T^{2/3}}{6}\right)=\sum_{s=1}^T \bm{1}\left[i_s\neq 1,\ N_{-1}(s+1)\leq \dfrac{K^{1/3}T^{2/3}}{6}\right]\leq \dfrac{K^{1/3}T^{2/3}}{6}\,.
	\end{align*}
	Taking the expectation w.r.t. instance 1, we further have
	\begin{equation*}
		\sum_{k=2}^K \E_S^{(1)}\left[Q_k\left(\dfrac{K^{1/3}T^{2/3}}{6}\right)\right]\leq \dfrac{K^{1/3}T^{2/3}}{6}\,.
	\end{equation*}
	To this end, we define $i^*$ as the least selected price among $p_2\,,\ldots\,,p_K$ in the first $K^{1/3}T^{2/3}/6$ time steps when $p_1$ is not selected, \ie, $i^*=\argmin_{k\in[K]\setminus\{1\}}\E_S^{(1)}[Q_k(K^{1/3}T^{2/3}/6)]\,,$ then, 
	\begin{equation*}			\E_S^{(1)}\left[Q_{i^*}\left(\dfrac{K^{1/3}T^{2/3}}{6}\right)\right]\leq \dfrac{1}{K-1}\sum_{k=2}^K \E_S^{(1)}\left[Q_k\left(\dfrac{K^{1/3}T^{2/3}}{6}\right)\right]\leq \frac{K^{1/3}T^{2/3}}{6(K-1)}\leq \frac{T^{2/3}}{3K^{2/3}}.
	\end{equation*}
	Therefore, following \eqref{eq:lb08} and \cref{prop:chain_rule}, we have
	\begin{align}
	    \nonumber\Pr_S^{(1)}(E_1)+\Pr_S^{(i^*)}\geq&\exp\left(-\E^{(1)}_S\left[Q_k\left(\dfrac{M}{6}\right)\right]\kl\left(\text{Bernoulli}(p_{i^*}\mu_{i^*}^{(1)})||\text{Bernoulli}(p_{i^*}\mu_{i^*}^{(i^*)})\right)\right)\\
	    \label{eq:lb09}\geq&\exp\left(-\frac{T^{2/3}}{6K^{2/3}}\log_e\left(1-\frac{16\epsilon^2K^{2/3}}{T^{2/3}}\right)\right)\geq\exp(-1)\,,
	\end{align}
	where the last step uses the fact that for $x\in[0,1/4],$ $\log_e(1-x)\geq-4\log_e(4/3)x\geq-4x$ and $T\geq K.$
	
	Combining \eqref{eq:lb07} and \eqref{eq:lb09}, the expected regret of $\pi$ on the instance 1 and instance $i^*$ sum up to $\Omega(K^{1/3}T^{2/3}),$ which indicates that $\pi$ has to incur $\Omega(K^{1/3}T^{2/3})$ expected regret in at least one of them.
	\end{itemize}

\subsection{Proof of \cref{prop:k_refund}}\label{sec:prop:k_refund}

We denote all time intervals when $p_1$ is not selected as $I_1,\dots,I_L$. That is, if $t<T_0$ and $i_t\neq 1$, then $t\in I_l$ for some $l\in[L]$; otherwise, $t\notin I_l$ for all $l\in[L]$. Among others, we use the partition with the smallest $L$, that is $\max_{t\in I_l}t+1<\min_{t\in I_{l+1}}t$.
		
If there exists $l\in [L]$ such that $|I_l|\geq M$, then the last $M$ time steps of $I_l$ will incur a refund of at least $(p_2-p_1)M\min_{i\in[k]}\mu_i$; otherwise for all time step when $p_1$ is not selected, $p_1$ is selected within $M$ time steps, which will incur a refund of at least $(p_2-p_1)M_0\min_{i\in[k]}\mu_i$. Therefore the refund incurred is at least $(p_2-p_1)\min_{i\in[k]}\mu_i\min\{M,M_0\}$.

\subsection{Proof of \cref{prop:chain_rule}}\label{sec:prop:chain_rule}
Similar to $S$, we define a sequence of sample space for the first $t-1$ time steps:
	\begin{equation*}
		S_t=\prod_{u=1}^{t-1}([K]\cup\{\infty\}\,,\ [0,1]\cup\{\infty\}),
	\end{equation*}
which records the price selected $i_u$ and the realized instantaneous payment $R_u$ (excluding refund) in each time step of $1\,,\ \ldots\,,\ t-1$. Similar to the definition of $S$, we also require that if for some time step $s\in[t-1]$, the number of times that the $p_2\,,\ \ldots\,,\ p_K$ are selected in total already exceeds $M/6$, then $i_u=R_u=\infty$ for all $u>s$.
		
We denote $\{S_t^{(u)}\}_{u=1}^{m_t}$ to be all possible trajectories, and for each trajectory $S_t^{(u)}$, we denote the number of times that $p_2\,,\ \ldots\,, \ p_K$ are recorded as $V(S_t^{(u)})\,$ (the time periods recorded as $(\infty,\infty)$ are ignored), \ie, 
\begin{align}
	V(S_t^{(u)}) = \sum_{s=1}^{t-1}\bm{1}[i_s\neq 1\,,\ i_s\neq \infty]\,.
\end{align}
		
By the chain rule of the KL divergence (see, \eg, theorem 2.5.3 of \cite{CoverT06}), we have
\begin{align}
	\kl(\Pr_S^{(1)}||\Pr_S^{(k)})
	=&\sum_{t=1}^T\sum_{u=1}^{m_{t}}\Pr^{(1)}(S_{t}^{(u)})\cdot\kl(\Pr_S^{(1)}(\{i_t,R_t|S_t^{(u)}\})||\Pr_S^{(k)}(\{i_t,R_t|S_t^{(u)}\}))\,.\label{kl-decomp-1}
\end{align}
If $V(S_t^{(u)})\geq \lfloor M/6\rfloor$, then regardless of instance, $(i_t,R_t)=(\infty,\infty)$, and
\begin{equation*}
	\kl(\Pr_S^{(1)}(\{i_t,R_t|S_t^{(u)}\})||\Pr_S^{(k)}(\{i_t,R_t|S_t^{(u)}\}))=0.
\end{equation*}
Therefore
\begin{align}
    \nonumber&\sum_{u=1}^{m_{t-1}}\Pr^{(1)}_S(S_t^{(u)})\cdot\kl(\Pr_S^{(1)}(\{i_t,R_t|S_t^{(u)}\})||\Pr_S^{(k)}(\{i_t,R_t|S_t^{(u)}\}))\\
	\nonumber=&\sum_{V(S_t^{(u)})< \lfloor M/6\rfloor}\Pr^{(1)}(S_t^{(u)})\cdot\kl(\Pr_S^{(1)}(\{i_t,R_t|S_t^{(u)}\})||\Pr_S^{(k)}(\{i_t,R_t|S_t^{(u)}\}))\\
	\nonumber=&\sum_{V(S_t^{(u)})< \lfloor M/6\rfloor}\Pr^{(1)}_S(S_t^{(u)})\left(\sum_{j=2}^K\Pr^{(1)}_S(i_t=j|S_t^{(u)})\cdot\kl(\text{Bernoulli}(p_j\mu_j^{(1)}||\text{Bernoulli}(p_j\mu_j^{(k)})))\right)\\
	=&\sum_{V(S_t^{(u)})< \lfloor M/6\rfloor}\sum_{j=2}^K\Pr^{(1)}_S(S_t^{(u)})\Pr^{(1)}_S(i_t=j|S_t^{(u)})\cdot\kl(\text{Bernoulli}(p_j\mu_j^{(1)}||\text{Bernoulli}(p_j\mu_j^{(k)})))\,.\label{kl-decomp-2}
\end{align}
Since $p_j\mu_j^{(1)}=p_j\mu_j^{(k)}$ for all $j\neq k$, we have
\begin{align}
	\nonumber&\sum_{j=2}^K\Pr^{(1)}_S(S_t^{(u)})\cdot\Pr^{(1)}_S(i_t=j|S_t^{(u)})\cdot\kl(\text{Bernoulli}(p_j\mu_j^{(1)}||\text{Bernoulli}(p_j\mu_j^{(k)})))\\
	\nonumber=&\Pr^{(1)}_S(S_t^{(u)})\Pr_S^{(1)}(i_t=k\mid S_t^{(u)})\cdot\kl(\text{Bernoulli}(p_k\mu_k^{(1)}||\text{Bernoulli}(p_k\mu_k^{(k)})))\\
	=&\Pr_S^{(1)}(S_t^{(u)},\ i_t=k)\cdot\kl(\text{Bernoulli}(p_k\mu_k^{(1)}||\text{Bernoulli}(p_k\mu_k^{(k)})))\,.\label{kl-decomp-3}
\end{align}
Combining~\eqref{kl-decomp-1}\eqref{kl-decomp-2}\eqref{kl-decomp-3}~we have
\begin{align*}
	\kl(\Pr_S^{(1)}||\Pr_S^{(k)})
	=&\kl(\text{Bernoulli}(p_k\mu_i^{(1)}||\text{Bernoulli}(p_k\mu_k^{(k)})))\sum_{t=1}^T\sum_{V(S_t^{(u)})<\lfloor M/6\rfloor}\Pr_S^{(1)}(S_t^{(u)}\,,\ i_t=k)\\
	=&\kl(\text{Bernoulli}(p_k\mu_k^{(1)}||\text{Bernoulli}(p_k\mu_k^{(k)})))\sum_{t=1}^T\Pr_S^{(1)}\left(i_t=k\,,\ N_{-1}(t)\leq\dfrac{M}{6}-1\right)\\
	=&\kl(\text{Bernoulli}(p_k\mu_k^{(1)}||\text{Bernoulli}(p_k\mu_k^{(k)})))\sum_{t=1}^T\Pr_S^{(1)}\left(i_t=k\,,\ N_{-1}(t+1)\leq \dfrac{M}{6}\right)\\
	=&\kl(\text{Bernoulli}(p_k\mu_k^{(1)}||\text{Bernoulli}(p_k\mu_k^{(k)})))\mathbb{E}_S^{(1)}\sum_{t=1}^T 1\left[i_t=k\,,\ N_{-1}(t+1)\leq \dfrac{M}{6}\right]\\
	=&\kl(\text{Bernoulli}(p_k\mu_k^{(1)}||\text{Bernoulli}(p_k\mu_k^{(k)})))\mathbb{E}_S^{(1)} \left[Q_k\left(\dfrac{M}{6}\right)\right]\,.
\end{align*}
Therefore
\begin{equation*}
	\kl(\Pr_S^{(1)}||\Pr_S^{(k)})=\mathbb{E}_S^{(1)}\left[ Q_k\left(\dfrac{M}{6}\right)\right]\kl(\text{Bernoulli}(p_k\mu_k^{(1)}||\text{Bernoulli}(p_k\mu_k^{(k)}))).\qedhere
\end{equation*}

\section{Proof of \cref{thm:k_price_ub2}}\label{sec:thm:k_price_ub2}
We distinguish three different cases:

\textbf{Case 1.} $M\leq \sqrt{KT}\,$: By Theorem 3.1 of \cite{AO10} we have $\texttt{Regret}'(\cA)=O(\sqrt{KT}\log(K))$, and $B_\cA\leq \log_2(T/e)/2+1=O(\log T)$. By Proposition \ref{prop:conversion} we have that in this case $\texttt{Regret}(\cA)=O(\sqrt{KT}\log(K)+M\log(T))$.

\textbf{Case 2.} $\sqrt{KT}<M\leq K^{1/3}T^{2/3}\,$: By Theorem 4 of \cite{GaoHRZ19} we have $\texttt{Regret}'(\cA)=O(\sqrt{KT}\log(K)\sqrt{\log(T)})$, and $B_\cA\leq \log_2(\log(T))=O(\log(\log(T)))$. By Proposition \ref{prop:conversion} we have that in this case $\texttt{Regret}(\cA)=O(\sqrt{KT}\log(K)\sqrt{\log(T)}+M\log(\log(T)))$.

\textbf{Case 3.} $M>K^{1/3}T^{2/3}\,$: We first discuss the upper bound for $\texttt{Regret}'(\cA)$. We assume w.l.o.g. that $\lambda_1=\max_{k\in[K]}\lambda_k$, and we denote $i^*=\argmax_{k\in[K]}\hat{\lambda}_k(n)\,,$ where $n=\lceil T^{2/3}K^{-2/3}\rceil\,.$ Then the regret without refund will be
\begin{equation*}
    n\sum_{k\neq 1}(\lambda_1-\lambda_k)+(T-Kn)\mathbb{E}[\lambda_1-\lambda_{i^*}]\leq (K-1)n\lambda_1+T\mathbb{E}[\lambda_1-\lambda_{i^*}].
\end{equation*}
Since $\lambda_1=\mathbb{E}[\hat{\lambda_1}(n)]$, we have
\begin{equation*}
    \mathbb{E}[\lambda_1-\lambda_{i^*}]=\mathbb{E}[\hat{\lambda}_1(n)-\lambda_{i^*}].
\end{equation*}
By the definition of $i^*$ we have that $\hat{\lambda}_1(n)\leq\max_{k\in[K]}\hat{\lambda}_k(n)=\hat{\lambda}_{i^*}(n)$, therefore $\hat{\lambda}_1(n)\leq \hat{\lambda}_{i^*}(n)$ for any realization of demand, and we have
\begin{equation*}
    \mathbb{E}[\hat{\lambda}_1(n)-\lambda_{i^*}]=\mathbb{E}[\hat{\lambda}_1(n)-\hat{\lambda}_{i^*}(n)+\hat{\lambda}_{i^*}(n)+\lambda_{i^*}]\leq \mathbb{E}[\hat{\lambda_{i^*}}(n)-\lambda_{i^*}]\leq\mathbb{E}\left[\max_{k\in[K]}(\hat{\lambda}_k(n)-\lambda_k)\right].
\end{equation*}
Since $p_kD_k\in[0,1]$, $\{p_kD_k-\lambda_k\}$ are $1/4$-subgaussian random variables, and $\{\hat{\lambda}_k(n)-\lambda_k\}$ are $1/(4n)$-subgaussian random variables. By the maximal inequality (see, \eg, theorem 1.14 of \cite{RH18}), we have
\begin{equation*}
    \mathbb{E}\left[\max_{k\in[K]}(\hat{\lambda}_k(n)-\lambda_k)\right]\leq \dfrac{1}{2\sqrt{n}}\sqrt{2\log(K)}.
\end{equation*}
Therefore the regret without refund will be upper bounded by
\begin{equation*}
    Kn\lambda_1+\dfrac{T}{\sqrt{2n}}\sqrt{\log(K)}.
\end{equation*}
and if we choose $n=\lceil T^{2/3}K^{-2/3}\rceil$ then the regret will be $O(K^{1/3}T^{2/3}\sqrt{\log(K)})$.

Then we consider the refund incurred. Since the prices remain the same after the first $Kn$ time steps, the refund is upper bounded by $Kn\lambda_1(p_K-p_1)=O(K^{1/3}T^{2/3})$. Combining the two parts of discussions we have that the regret in this case is $O(K^{1/3}T^{2/3}\sqrt{\log(K)})$.
\end{appendices}

\end{document}